\documentclass[11pt]{article}
\usepackage[T1]{fontenc}
\usepackage[utf8]{inputenc}
\usepackage[margin=2.6cm]{geometry}
\usepackage{amsmath,amssymb,amsthm}
\usepackage{booktabs}
\usepackage{array}
\usepackage{longtable}
\usepackage{placeins}
\usepackage{xcolor}
\usepackage{tikz}
\usetikzlibrary{arrows.meta,backgrounds,calc,fit,positioning}
\usepackage{pgfplots}
\pgfplotsset{compat=1.18}
\usepackage{url}
\usepackage[colorlinks=true,linkcolor=blue,citecolor=blue,urlcolor=blue]{hyperref}
\hypersetup{
  pdftitle={Structure-Preserving Uncertainty Propagation in First-Order Proof Search},
  pdfauthor={Tanel Tammet},
  bookmarksdepth=2
}

\newtheorem{proposition}{Proposition}
\newtheorem{corollary}[proposition]{Corollary}

\newcommand{\GK}{GK}
\newcommand{\ProbLog}{\textsc{ProbLog}}
\newcommand{\sfour}{\ensuremath{(s^+,s^-,c,g)}}
\newcolumntype{L}[1]{>{\raggedright\arraybackslash}p{#1}}
\newcolumntype{C}[1]{>{\centering\arraybackslash}p{#1}}

\definecolor{gkblue}{HTML}{2F6FA3}
\definecolor{gkred}{HTML}{B24A4A}
\definecolor{gkgold}{HTML}{B47A16}
\definecolor{gkteal}{HTML}{287D78}
\definecolor{gkpurple}{HTML}{725A9A}
\definecolor{gkgray}{HTML}{66717A}
\tikzset{
  gkbox/.style={draw=gkgray, rounded corners=2pt, fill=white,
    align=center, inner sep=4pt, font=\small},
  gkprocess/.style={gkbox, fill=gkblue!7, draw=gkblue},
  gkpositive/.style={gkbox, fill=gkblue!10, draw=gkblue},
  gknegative/.style={gkbox, fill=gkred!9, draw=gkred},
  gkblocker/.style={gkbox, fill=gkgold!12, draw=gkgold},
  gkresult/.style={gkbox, fill=gkteal!10, draw=gkteal},
  gkfaded/.style={gkbox, fill=gkgray!8, draw=gkgray},
  gkarrow/.style={-{Stealth[length=2.2mm]}, semithick, draw=gkgray},
  gkblockarrow/.style={-{Stealth[length=2.2mm]}, semithick, dashed, draw=gkgold}
}

\title{Structure-Preserving Uncertainty Propagation\\
in First-Order Proof Search}

\author{
Tanel Tammet\\
Tallinn University of Technology\\
\texttt{tanel.tammet@taltech.ee}
}

\date{August 2026}

\begin{document}
\maketitle

\begin{abstract}
\GK{} is a query-directed first-order prover that extends ordinary resolution-based proof search with explicit positive and negative claims,
numerical confidence values, and prioritized default rules with exceptions. It works directly with non-ground clauses,
including equality and function terms. Candidate proofs are found by bounded first-order proof search; exception conditions
of defaults are checked by further bounded searches, recursively when exceptions themselves depend on defaults.
This avoids requiring a finite global grounding, while allowing incomplete searches to be reported as such.

This paper adds structure-preserving quantitative reporting to that framework. Retained proof histories are used in two calculations.
The first reconstructs the uncertain ground premises used by each proof and computes the probability that at least
one retained proof is available, without counting shared premises independently. The second resolves positive
and negative support at intermediate atoms before that support is propagated through later rules;
the same calculation evaluates uncertain exception conditions for individual rule applications.
Reports separate positive support, negative support, conflict, and ignorance and identify detected
incomplete calculations or fallbacks. The implementation performs bounded reconstruction and dependency
traversal after proof search and still requires no global grounding. Analytic examples and independent simulators
reproduce the reference calculations on their stated fragments. Comparisons with probabilistic logic,
probabilistic ASP, default logic, and goal-directed ASP identify cases of agreement, semantic difference, unsupported translation, and incomplete computation.
\end{abstract}

\setcounter{tocdepth}{1}

\clearpage
\tableofcontents
\clearpage

\section{Introduction}

Knowledge bases assembled from heterogeneous sources may contain automatically extracted or generated claims
alongside curated facts and rules. Such claims may have different reliabilities and may conflict; 
generated text may also contain both supported and unsupported factual claims \cite{min2023factscore}. 
General rules may have exceptions, evidence for an exception can itself be uncertain or contradicted, 
and two proofs of the same conclusion can share part of their support.

\GK{} is a query-directed first-order reasoning system built on
\textsc{gkc}, a resolution-based prover for first-order logic with equality,
function terms, and arbitrary clauses~\cite{tammet2019}.
Earlier \GK{} work added confidence-carrying proof
search and combination of supporting and opposing evidence~\cite{tammet2021},
followed by recursively checked prioritized default rules~\cite{tammet2022}.
\GK{} therefore combines first-order proof search with explicit positive and
negative claims, input confidence values, and prioritized exceptions, without
requiring a finite global grounding.

This paper adds two quantitative report calculations to that framework.  One
combines proofs that may share uncertain premises.  The other resolves opposing
support at an intermediate atom before that atom is used by later rules; the
same calculation evaluates uncertain exception conditions for individual rule
applications.

Scalar proof values omit two distinctions needed by the combination rules:
whether alternative proofs reuse the same uncertain clause instances, and at
which atom opposing evidence enters a derivation.  We call an evaluation
\emph{dependency-aware} when it retains these identities and locations until
support is combined.  In \GK{} this evaluation is performed after proof
search.  It traverses the dependency graph between ground atoms
(Sections~\ref{sec:sharedthreshold} and~\ref{sec:premisesearch}) instead of
using a single value computed for each finished proof.

Probabilistic logic programming provides
probabilities and conditioning on a logic-program fragment
\cite{deraedt2007,fierens2015}; probabilistic ASP combines quantitative choices
with stable-model reasoning on groundable programs
\cite{baral2009,cozman2020,hahn2025}; and s(CASP) provides goal-directed
non-ground default negation without a probabilistic annotation semantics
\cite{arias2018}.  Markov logic accepts weighted first-order formulas but
defines a finite-domain ground graphical model rather than prioritized
defaults whose exception conditions are checked on candidate
proofs~\cite{richardson2006}.  \GK{} combines query-directed first-order resolution,
positive and negative input confidences, recursive prioritized exception
checks, and proof-producing answers without global grounding.

Consider $0.5::bird(a)$, $0.5::\neg bird(a)$, and
$0.9::flies(a)\leftarrow bird(a)$.  A calculation that combines opposition
only at the final query propagates positive support $0.45$ through $bird(a)$;
the equal positive and negative support for that premise has not been
resolved.  Similar errors arise when repeated uses of one clause are conflated
with its ground instances, when premises shared by several proofs are counted
more than once, or when an exception to one rule is treated as evidence for
the opposite conclusion.

\GK{} therefore retains proof provenance until same-polarity support has been
combined and evaluates opposition at the atom where it occurs.  A separate
directed clause-level search evaluates premise conflict and exception
conditions locally.

We compare \GK{} with implemented probabilistic and nonmonotonic reasoners
on three restricted fragments: one-sided probabilistic proof combination,
finite default rules with uncertain inputs, and query-directed first-order
problems.  The comparison distinguishes exact translations, approximate
analogues, unsupported translations, and incomplete runs.  Reports retain
each system's native output type; the comparison does not convert them to one
score.

The main contributions are:

\begin{enumerate}
\item a retained-proof calculation that reconstructs the uncertain ground
premises used by each proof and avoids double-counting premises shared by
several proofs;
\item a dependency-aware calculation that resolves positive and negative
support before propagation and evaluates uncertain exceptions for individual
rule applications;
\item a bounded implementation that reports which calculation was used and
marks detected fallbacks and unsupported cases; and
\item analytic, simulation, and system-comparison results that identify the
fragments on which the calculations agree with their reference semantics and
with other implemented systems.
\end{enumerate}

Reports identify the selected calculation and detected incomplete or
unsupported cases (Section~\ref{sec:reportstatus}); signed confidence is
positive support minus negative support.

Input confidences may represent source assessments or externally supplied scores.
Such scores need not be calibrated probabilities of correctness \cite{guo2017calibration}. 
This paper takes them as input and does not estimate or calibrate them.
The reported
values measure available proof support under the stated models and are not
claimed to be posterior probabilities of truth.  The distinction between proof
support and truth probability also affects comparisons: systems based on
different possible-world constructions can return different numbers for the
same surface syntax without an implementation error in either system.

Executable binaries, documentation, examples, and the public samplers are
available at \url{https://github.com/tammet/gkreasoner}.  

Section~2 gives motivating examples.  Section~3 defines the input language
and proof search, and Section~4 reviews related systems.  Sections~5 and~6
define the reference semantics and the bounded report calculations.
Sections~7--9 describe the evaluation and the system comparisons.
Sections~10--12 discuss related semantics, limitations, and conclusions.  The
appendices provide concrete syntax, further default cases, complete inputs,
and comparison provenance.

\section{Motivating examples}
\label{sec:example}

Four small examples introduce the problem before the formal definitions.  The
first two show cases in which the calculations agree.  The last two introduce
uncertain exceptions and opposition at an intermediate premise.  
Throughout the semantic exposition, we write
$p::H\leftarrow B_1\land\cdots\land B_k$ for a rule with input confidence $p$,
omit a confidence of one, and use $\neg$ for explicit
classical negation.  A default is written
$p::H\leftarrow B\ [\mathrm{unless}\ E]$; priorities are introduced later,
where they are needed.
Appendix~\ref{sec:gksyntax} gives the corresponding concrete \GK{} syntax and
report fields.

An \emph{atom} is a positive predicate application $A$; a \emph{literal} is
$A$ or its explicit negation $\neg A$.  For a query literal $L$, write
\[
  R(L)=\bigl(s^+(L),s^-(L),c(L),g(L)\bigr),
\]
where the four components are positive support, negative support, conflict,
and ignorance.  Positive support is unopposed support for $L$, while
negative support is unopposed support for the explicit opposite of $L$;
support in the conflict region is recorded separately.  We call
the number written before \texttt{::} the \emph{input confidence}.  It is an
activation probability in ground-instance activation semantics and determines
the support contribution of an applicable fact or rule in shared-threshold
semantics.  The final scalar $C(L)=s^+(L)-s^-(L)$ is called
\emph{signed confidence}; it is not a probability of truth.

\paragraph{A classical penguin default.}
The birds-and-penguins example is
\[
\begin{gathered}
bird(tweety),\qquad penguin(pingu),\\
bird(x)\leftarrow penguin(x),\qquad
\neg flies(x)\leftarrow penguin(x),\\
flies(x)\leftarrow bird(x)
  \quad[\mathrm{unless}\ \neg flies(x)],
\qquad Q=flies(x).
\end{gathered}
\]
\GK{} accepts $tweety$.  It also finds a candidate default proof for $pingu$,
but the proof of $\neg flies(pingu)$ blocks that default, so $pingu$ is
rejected and both proofs can be reported.  In the usual finite ASP encoding,
the default becomes \texttt{flies(X) :- bird(X), not -flies(X)}.  clingo,
DLV2, and s(CASP) give the same two conclusions on this program
\cite{alviano2017,gebser2019,arias2018,tammet2022}.  This example motivates
defaults but does not distinguish the systems.

\paragraph{Positive support propagated through a rule.}
Now consider two positive bird facts and one ordinary rule:
\[
\begin{gathered}
0.8::bird(tweety),\qquad 0.6::bird(robin),\\
0.9::flies(x)\leftarrow bird(x),
\qquad Q=flies(x).
\end{gathered}
\]
Each answer has one proof containing one ground fact instance and one ground
rule instance.  \GK{} therefore reports $0.8\cdot0.9=0.72$ for
$flies(tweety)$ and $0.6\cdot0.9=0.54$ for $flies(robin)$.  \ProbLog{} gives
the same query probabilities when the corresponding ground choices are
independent~\cite{deraedt2007,fierens2015}.  The calculations agree here
because each answer has one proof, support occurs in only one polarity, and
the answers share no premises.  Three distinctions developed below do not yet arise here: repeated uses of
one ground clause instance, overlap between alternative proofs, and opposing
evidence at a premise.

\paragraph{An uncertain exception.}
Now make the exception evidence uncertain:
\[
\begin{gathered}
bird(tweety),\qquad bird(robin),\qquad
0.9::\neg flies(tweety),\\
flies(x)\leftarrow bird(x)
  \quad[\mathrm{unless}\ \neg flies(x)],
\qquad Q=flies(x).
\end{gathered}
\]
In ground-instance activation semantics, the contrary clause is active in
$0.9$ of the worlds and the recursive exception policy blocks the default there.  In
the remaining $0.1$, the default is available.  In shared-threshold semantics,
ordinary support for the opposite literal is usable in a region of probability
$0.9$, and the contrary-gated default is usable only in the remaining
$0.1$.  The two
constructions therefore agree in this simple case and give
$R(flies(tweety))=(0.1,0.9,0,0)$.  Their agreement does not extend to an exception
condition that itself has opposing support, as
Section~\ref{sec:defaults} shows.  \GK{}
accepts $robin$ with signed confidence one but rejects the $tweety$ candidate with
signed confidence $0.1-0.9=-0.8$.  A stratified
\ProbLog{} encoding with an independent probabilistic contrary fact and a
negation-as-failure guard reproduces these two marginals.  Standard clingo and
DLV2 do not themselves assign a probability to the $0.9$ fact; probabilistic
logic programs and probabilistic ASP systems can do so, but their treatment of
multiple stable models or cycles requires an additional semantic choice
\cite{baral2009,cozman2020,totis2023}.

\paragraph{Opposition at a premise.}
Finally, consider the premise-conflict case:
\[
\begin{gathered}
0.5::bird(a),\qquad 0.5::\neg bird(a),\\
0.9::flies(x)\leftarrow bird(x),
\qquad Q=flies(a).
\end{gathered}
\]
A calculation that checks opposition only at the final query returns $0.45$:
it finds the positive proof of $flies(a)$ but does not inspect opposition to
its premise~\cite{tammet2021}.  Plain \ProbLog{} likewise treats $bird(a)$ and an
explicitly named negative atom as separate choices and returns $0.45$ for the
positive proof.  Adding an independent guard against the negative atom instead
returns $0.5\cdot0.5\cdot0.9=0.225$.  The policy used here instead resolves
the equal opposing support at $bird(a)$ before applying the rule: equally
strong opposing evidence at the premise leaves no usable support for either
polarity, so the rule does not fire.  \GK{} resolves
that opposition before propagating the premise.  Its report is
$R(flies(a))=(0,0,0,1)$: signed confidence is zero and $bird(a)$ is
identified as the
contested source.  Ambiguity-propagating
defeasible logics can impose a related blocking policy, whereas standard ASP
encodings require an explicit treatment of the inconsistent premise
\cite{governatori2004}.

The examples above have finite, function-free groundings.  In full
first-order default logic, deciding whether a default applies requires
checking that its exception condition has no admissible proof; this is
undecidable in general.  Most ASP systems and many implemented
default-logic systems therefore first construct a finite ground program
and then apply propositional or stable-model reasoning.  This approach
does not apply directly when function symbols generate an infinite
grounding.  \GK{} instead searches the original non-ground clauses,
delays exception checks until candidate proofs have been found, and
bounds both the main search and the recursive exception searches.  A
bounded failure to find an exception proof can therefore leave the
result incomplete.  s(CASP) also avoids global grounding, but uses a
goal-directed logic-programming execution model
\cite{arias2018,tammet2022}.

\subsection{Models, calculations, and reported quantities}

A \emph{retained proof} is an answer proof kept by the bounded search for
later report construction.  Its \emph{proof history} is the directed acyclic
graph (DAG) of recorded inference steps and parent links.  The retained proofs
for one answer substitution and one query polarity form its
\emph{retained proof set}.  The retained-proof calculation estimates the
probability that at least one proof in this set is available.

The paper defines two separate reference models.  Ground-instance activation
models the availability of uncertain clause instances.  Shared-threshold
semantics models the usable support at ground atoms.  \GK{} first constructs
a retained-proof result and attempts the dependency-aware calculation when it
detects a contested ground atom.  A completed dependency-aware result
replaces the retained-proof result; otherwise the retained-proof result is
returned directly or as a flagged fallback.
The selected calculation and its status are therefore part of the
interpretation of each answer row.  Different rows in one run may be produced
by different calculations.

For semantic interpretation, the author or analyst must specify whether an
input confidence is intended as a ground-instance activation probability or
as the parameter of a support contribution in shared-threshold semantics.
The standard \GK{} report does not enforce that choice.  Model-specific
analysis must use the corresponding reference
calculation or sampler.  Applying both models to one knowledge base is a
sensitivity comparison.  Values from the two models are never multiplied or
pooled.

\begin{center}
\scriptsize
\begin{tabular}{L{3.0cm}L{3.2cm}L{7.6cm}}
\toprule
Quantity & Where it belongs & Meaning \\
\midrule
Input confidence
& both models & activation probability in Model~1; parameter determining a
support contribution in Model~2 \\
Search-time proof score
& proof search & product carried by one search history; used for search,
retention, and exception checks, not as the final report semantics \\
Proof-availability probability
& Model~1 & probability that all activation events used by one retained proof
are present \\
Retained-proof value
& Model~1 calculation & union probability of the retained same-polarity proof
family, subject to replay and search coverage \\
Rule confidence
& Model~2 & parameter determining the support contribution of an applicable
directed rule \\
Pooled support values $a,b$
& Model~2 & positive and negative support after same-polarity aggregation and
before opposition resolution \\
$s^+,s^-,c,g$
& selected report & shared-threshold region probabilities for a completed
dependency-aware result; otherwise a proof-pool decomposition whose
interpretation depends on the calculation status; see
Section~\ref{sec:reportstatus} \\
Signed confidence $C=s^+-s^-$
& selected report & direction and magnitude derived from the selected
four-component report \\
\bottomrule
\end{tabular}
\end{center}

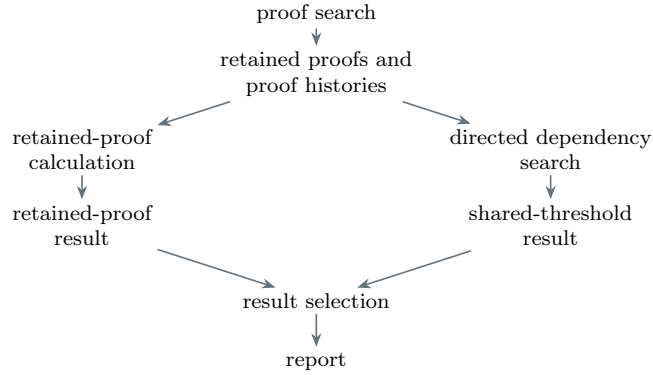
\begin{figure}[htbp]
\centering
\begin{tikzpicture}[
  mapnode/.style={font=\scriptsize, align=center, inner sep=2pt},
  maparrow/.style={-{Stealth[length=1.8mm]}, semithick, draw=gkgray}
]
  \node[mapnode] (search) at (0,0) {proof search};
  \node[mapnode] (hist) at (0,-.8) {retained proofs and\\proof histories};
  \node[mapnode] (events) at (-3.1,-1.8) {retained-proof\\calculation};
  \node[mapnode] (directed) at (3.1,-1.8) {directed dependency\\search};
  \node[mapnode] (proofresult) at (-3.1,-2.8) {retained-proof\\result};
  \node[mapnode] (threshold) at (3.1,-2.8) {shared-threshold\\result};
  \node[mapnode] (select) at (0,-3.8) {result selection};
  \node[mapnode] (report) at (0,-4.6) {report};
  \draw[maparrow] (search) -- (hist);
  \draw[maparrow] (hist) -- (events);
  \draw[maparrow] (hist) -- (directed);
  \draw[maparrow] (events) -- (proofresult);
  \draw[maparrow] (directed) -- (threshold);
  \draw[maparrow] (proofresult) -- (select);
  \draw[maparrow] (threshold) -- (select);
  \draw[maparrow] (select) -- (report);
\end{tikzpicture}
\caption{Calculation sequence used to assemble a \GK{} report.}
\label{fig:modelcalculationmap}
\end{figure}

\section{Language and proof search}
\label{sec:setting}

\subsection{Input language and clause confidences}

The rule notation of Section~\ref{sec:example} is expository; the underlying
reasoner accepts arbitrary first-order formulas and clausifies them before
search.  The quantitative result depends on the clauses produced by parsing
and clausification: separate occurrences of the same formula are treated as
separate inputs, an authored rule head can determine the direction of a rule,
and two classically equivalent source formulas can therefore receive
different results.  Confidence
and exception metadata attach to the individual resulting input clauses.  A confidence-bearing non-Horn
clause is written
\[
  p::L_1\lor\cdots\lor L_n.
\]
Any input clause may carry an input confidence $p\in(0,1]$.
If the input syntax permits a confidence on a compound formula, that notation
is syntactic shorthand.  The input reader first clausifies the formula into
$n$ clauses and assigns $p^{1/n}$ to each.  All later search,
provenance, and event identities refer to those clauses.  This split preserves
the product of all $n$ clause confidences, but a derivation may use only some
of the clauses.  This is an input convention; it makes no claim about
formula-level probabilities.

Rule-form input records its authored conclusion and may attach an
exception condition and a priority.  Section~\ref{sec:defaultrules}
defines the default notation, its Skolemization convention, and its
clausified representation.  Exception conditions are literals; a
compound source condition must first be represented by clauses defining
an auxiliary literal.  The orientation policy and the resulting
directed applications are defined in
Section~\ref{sec:premisesearch}.

Clause confidences are interpreted per ground instance.
Section~\ref{sec:twomodels} defines the resulting activation semantics, while
Section~\ref{sec:proofevents} defines how \GK{} identifies the ground
instances used by a retained proof.  Neither interpretation asserts that a
clause is true with posterior probability equal to its confidence.
Two identical uncertain statements entered as separate occurrences define
separate evidence sources and separate activation events.  Authors who intend
one source must enter one statement: entering a statement twice changes its
quantitative effect although it adds nothing logically.

Absence of support is not support for the negation.  \GK{} makes no
closed-world assumption and has no negation as failure: support for $\neg L$
has to come from a clause whose conclusion is $\neg L$.  A literal that no
derivation supports in either polarity is therefore distinct from one whose
explicit negation has support.  The specialized class-frequency interpretation used by the optional
paired reference-class construction is stated separately in
Appendix~\ref{sec:pairedexceptions}.

A query is either \emph{ground}, containing no variables, or \emph{open},
containing variables whose values are requested.

As in standard resolution question answering, an open query is negated and
extended with a special answer literal that records the substitution applied
to those variables; each recorded substitution $\sigma$ is one answer.
(Throughout, $\sigma$ names an answer substitution and $\theta$ a grounding
substitution inside a proof.)  An answer substitution may be ground or may
itself contain residual variables: with input $P(f(x))$ and query $P(y)$, the
recorded answer is $y=f(x)$.

The quantitative semantics below is defined for ground query instances.
\GK{} invokes the dependency-aware evaluator only for ground answer
instances; a non-ground answer instance currently retains the retained-proof
result, without a dedicated flag (Section~\ref{sec:limitations}).

The prover supports equality,
function terms, arbitrary clauses, and a shared-memory knowledge
base~\cite{tammet2019}.

\subsection{Default rules and priorities}
\label{sec:defaultrules}

Reiter default logic extends first-order logic with rules of the form
\[
  \frac{\alpha:\beta_1,\ldots,\beta_n}{\gamma}.
\]
Such a rule may derive $\gamma$ when $\alpha$ is derivable and no
negated justification $\neg\beta_i$ is derivable in the candidate
extension.  The notation used in this paper,
\[
  H \leftarrow B\ [\text{unless }E],
\]
is the one-exception case: $B$ is the prerequisite, $H$ is the
conclusion, and derivability of $E$ blocks the rule application.  In
Reiter's notation the corresponding justification is $\neg E$.  A
normal default has $E=\neg H$.

Quantified defaults require a choice of quantifier interpretation.
Following Reiter's Skolemized reading and the convention used in the
earlier \GK{} implementation, quantified defaults are Skolemized before
clausification.  For example,
\[
  \frac{:\exists x\,P(x)}{\exists x\,P(x)}
\]
is represented as
\[
  \frac{:P(c)}{P(c)},
\]
where $c$ is a fresh Skolem constant
\cite{reiter1980,tammet2022}.

A normal default ``birds normally fly'' is written in the primary notation as
\[
  flies(x)\leftarrow bird(x)
  \quad[\mathrm{unless}\ \neg flies(x);\ \mathrm{priority}\ \pi].
\]
The literal $E$ following \emph{unless} is the \emph{exception condition};
here it is $\neg flies(x)$.  The optional value $\pi$ is the priority of
the default.  Positive integer priorities are ordered numerically, with a
larger integer denoting higher priority.  During an exception check at
priority $\pi$, support derived through a default of strictly lower
priority is excluded.  An omitted priority is parsed as zero, meaning
unranked; priority zero is incomparable with the positive ranks.
Priorities may also be classes ordered by a taxonomy, in which case a
more specific class defeats a more general one
\cite{tammet2022}.  Section~\ref{sec:localrules} defines how priorities
also affect opposing defaults, and
Appendix~\ref{sec:rankclasses} gives the detailed restriction used
inside exception checks.

\GK{} records the exception condition and priority by augmenting the
ordinary clause with a \emph{blocker literal}:
\[
  \neg bird(x)\lor flies(x)\lor
  block(\pi,\neg flies(x)).
\]
The blocker literal is carried through candidate proofs.  Its
substitutions and the nested searches that it initiates are described
in Section~\ref{sec:proofsearch}.

\subsection{Proof search, exception checking, and proof histories}
\label{sec:proofsearch}
\label{sec:exceptionchecking}% at subsection level: \paragraph is unnumbered

\GK{} is query-directed and does not globally ground the knowledge base.
The main given-clause search operates on non-ground clauses.  Answer
literals record substitutions for open queries, while blocker literals
carry the instantiated exception conditions and priorities of default
applications through the proof.  Search continues after the first
answer so that several answer substitutions and several proofs of each
polarity can be retained.

The proof search can be summarized in three stages.
\begin{enumerate}
\item Parse and clausify the knowledge base and the negated query.  Add answer
literals to open queries and store confidence and exception metadata
on the input clauses.
\item Run a bounded given-clause search using resolution, paramodulation,
factoring, simplification, and arithmetic instantiation.  Retain
histories for answer clauses rather than stopping at the first proof.
\item For each answer substitution, collect proofs of the query literal
and its opposite-polarity literal.  Check blocker literals in those
proofs by nested searches for their instantiated exception conditions.
\end{enumerate}

Figure~\ref{fig:proofsearch} shows the main answer search and the nested
exception searches.  For each answer substitution, the two solid
branches collect proofs of the query and its explicit opposite.  Dashed
branches check the instantiated exception conditions carried by those
proofs.

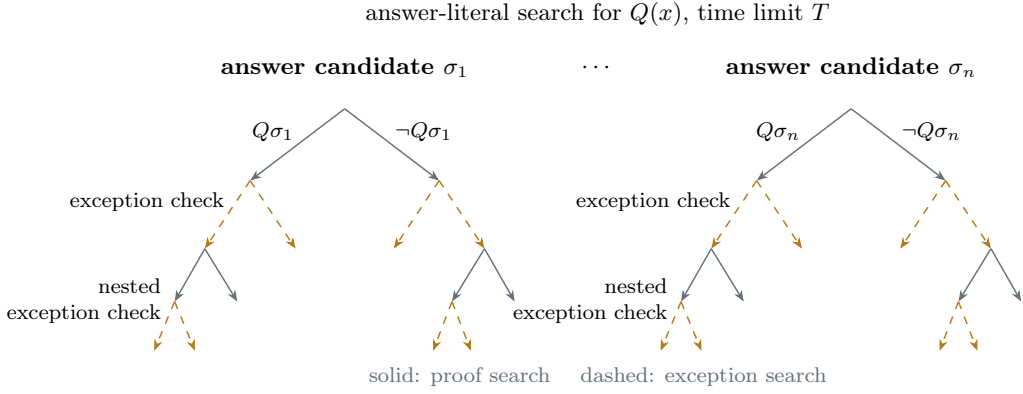
\begin{figure}[htbp]
\centering
\begin{tikzpicture}[
  search branch/.style={-{Stealth[length=1.5mm,width=1.15mm]},
    line width=.55pt, draw=gkgray},
  exception branch/.style={search branch, dashed, draw=gkgold},
  edge label/.style={font=\scriptsize, inner sep=1pt, text=black},
  candidate label/.style={font=\footnotesize\bfseries, text=black}
]
  % One compact tree is repeated for each answer substitution.  Solid edges
  % are resolution searches; dashed edges are bounded blocker searches.
  \newcommand{\candidateproofsearch}[2]{%
    \begin{scope}[shift={(#1,0)}]
      \node[candidate label] at (0,.55) {answer candidate $\sigma_{#2}$};

      \draw[search branch] (0,0) --
        node[edge label, above left] {$Q\sigma_{#2}$} (-1.25,-.95);
      \draw[search branch] (0,0) --
        node[edge label, above right] {$\neg Q\sigma_{#2}$} (1.25,-.95);

      \draw[exception branch] (-1.25,-.95) --
        node[edge label, above left] {exception check} (-1.85,-1.85);
      \draw[exception branch] (-1.25,-.95) -- (-.65,-1.85);
      \draw[exception branch] (1.25,-.95) -- (.65,-1.85);
      \draw[exception branch] (1.25,-.95) -- (1.85,-1.85);

      \draw[search branch] (-1.85,-1.85) -- (-2.25,-2.55);
      \draw[search branch] (-1.85,-1.85) -- (-1.42,-2.55);
      \draw[search branch] (1.85,-1.85) -- (1.42,-2.55);
      \draw[search branch] (1.85,-1.85) -- (2.25,-2.55);

      \draw[exception branch] (-2.25,-2.55) --
        node[edge label, above left, align=right]
          {nested\\exception check} (-2.52,-3.20);
      \draw[exception branch] (-2.25,-2.55) -- (-1.98,-3.20);
      \draw[exception branch] (1.42,-2.55) -- (1.17,-3.20);
      \draw[exception branch] (1.42,-2.55) -- (1.67,-3.20);
    \end{scope}%
  }

  \node[font=\footnotesize] at (0,1.25)
    {answer-literal search for $Q(x)$, time limit $T$};
  \candidateproofsearch{-3.35}{1}
  \node[font=\footnotesize] at (0,.55) {$\cdots$};
  \candidateproofsearch{3.35}{n}

  \node[font=\scriptsize, text=gkgray] at (0,-3.55)
    {solid: proof search \quad dashed: exception search};
\end{tikzpicture}
\caption{Proof search structure.  Solid branches collect proofs of
$Q\sigma_i$ and $\neg Q\sigma_i$.  The main answer search has time limit $T$.
A dashed branch checks the exception
condition named by a blocker in that
proof, with a smaller limit $T_1$; a nested exception check runs with
$T_2<T_1$.  Candidate and proof searches continue after the first
substitution and first proof.}
\label{fig:proofsearch}
\end{figure}
\FloatBarrier

\paragraph{Nested exception checks.}
Exception checking is delayed until candidate answer proofs have been
found.  A blocker search is a bounded search for the exception condition
stored in a blocker literal.  Because the blocker receives the
substitutions applied to the default clause, the search concerns that
particular ground or partially ground rule application.  An exception
proof may contain further blocker literals, so checking is recursive.
Each recursive level receives a smaller time budget.  The priority
restriction of Section~\ref{sec:defaultrules} remains in force
throughout the nested search.

Each returned proof $d$ of an exception condition is considered
separately.  Let $w(d)$ be its search-time proof score.  The proof enters
recursive checking only if
\[
  w(d)\ge\lambda,
\]
where $\lambda$ is the configurable exception limit; its default value
is $0.5$.  Proofs are not paired with proofs of the opposite literal for
this operational test.  If several exception proofs exceed the limit,
each is checked recursively.  One exception proof that survives its own
nested checks disables the candidate rule application.
Because this is a threshold test, a small change in the proof scores near
the limit can change whether the exception is used and may cause an
abrupt change in the reported confidence.  The positive limit prevents
very weak exception proofs from initiating further searches and disabling
candidates in compatibility mode.  Setting $\lambda=0$ permits every
returned exception proof to enter recursive checking.  The limit is an
operational search filter and is not part of either reference semantics.

This recursive blocker procedure is used during proof search and in
compatibility-mode answer selection.  Dependency-aware reporting uses a
different calculation: positive and negative support at the exception
atom are collected by the report-time traversal and resolved before the
exception condition is used.  Shared predecessors of those derivations
are retained rather than treated as independent evidence.

Stages 2 and 3 use the confidence-carrying resolution implementation
described in \cite{tammet2021}.  The logical inference rules produce the
same resolvents as ordinary resolution, but every history also carries a
search-time proof score and the input-clause dependencies of the derivation.
Binary resolution and paramodulation multiply the two parent proof scores;
factoring preserves its parent's score.  The dependency sets of the parent
derivations are combined by set union.
Proof-score- and dependency-aware
subsumption prevents a logically more general clause from deleting a stronger
or independently supported derivation merely because ordinary subsumption
would permit it.  A relevance filter discards contradiction proofs that do
not depend on the query clause, limiting explosion from unrelated
inconsistencies.

These modifications affect search and proof retention.  The search-time
proof score is the product of the input confidences recorded along one
derivation; it ranks and retains derivations and tests exception proofs.  It
is not the final quantitative semantics of this paper: it neither establishes
that all proofs were found nor gives a query probability.  The retained-proof calculation later reconstructs
ground activation events and takes a union over retained proofs; the
dependency-aware evaluator instead calculates shared-threshold regions.
A multiplied history value counts uncertain uses
along the recorded derivation.  \emph{Proof replay} is the deterministic
re-execution of the recorded inference steps of a proof history to reconstruct
the substitutions and ground clause instances used in the proof.  It
identifies repeated uses of the
same ground instance, and set-level proof combination handles overlap among
several retained proofs.  The dependency-aware traversal then addresses
opposition at intermediate atoms.  Because proof search is bounded, its
strategy, memory limit, and time limit may cause it to omit proofs or
terminate an exception check.

Proof search produces candidate answers and their proof histories.
Report construction then has two paths.  \GK{} first replays the retained
histories, reconstructs the uncertain ground clause instances used by
each proof, and combines same-polarity proofs while preserving shared
instances.  If the query or a premise on a retained proof path may have
support in both polarities, \GK{} starts a separate report-time traversal
of the clausified knowledge base.  This traversal collects the relevant
directed positive and negative derivations and computes the
dependency-aware report.  A completed dependency-aware result replaces
the retained-proof result.  Otherwise \GK{} returns the retained-proof
result directly or as a flagged fallback.  Section~\ref{sec:implementation}
defines the calculations and their limits.

\FloatBarrier

\section{Related work}
\label{sec:related}

\subsection{Implemented systems}

The implemented systems compared here differ in which objects carry numerical values, how
nonmonotonicity is represented, what form of answer is returned, and whether a
finite grounding is required.  Throughout, \emph{native} means the unmodified
input or output supported directly by a system.  Tables~\ref{tab:systemobjects}
and~\ref{tab:systemstrategies} state these distinctions before particular
values are compared.  Numerical agreement is expected only when the translation is exact and the
input confidences have the same interpretation in both systems.  Every external
result is tied to a concrete system version, input, and query.  The system comparisons
include PASTA~\cite{azzolini2023}, TweetyProject~\cite{thimm2017}, and
I-DLV~\cite{calimeri2017}.  The compared TweetyProject modules implement
Reiter default logic (RDL) and defeasible logic programming (DeLP).  DLV2
2.1.1 is abbreviated as DLV in the tables.

\begin{table}[ht]
\centering
\scriptsize
\caption{Numerical quantities and native outputs returned by the comparison systems.}
\label{tab:systemobjects}
\begin{tabular}{L{2.5cm}L{6.3cm}L{5.2cm}}
\toprule
System & Numerical quantity, if any & Native output \\
\midrule
\GK{} & retained-proof value or shared-threshold region probabilities,
depending on the selected calculation
& supporting and defeating proofs, four-component report, and status flags \\
\ProbLog{} 2.2.10 & probability of probabilistic choices and query success
& query probability \\
PASTA 1.0.1 & lower and upper query probabilities under credal stable-model semantics
& lower/upper query probability \\
TweetyProject 1.31 RDL/DeLP & none in the compared modules
& extensions or warranted/rejected/undecided status \\
clingo 5.6.2; DLV 2.1.1 & none in the base systems
& stable models or cautious consequences \\
I-DLV 1.1.6 & none in the base system
& query answers or ground program for a solver \\
s(CASP) 1.1.4 & none in the base system
& partial stable model and justification \\
plingo 1.1.0 & weights or probabilities over stable models
& model or query probabilities \\
smProbLog 2.1.0.42 & probabilistic choices with stable-model completion
& marginal probabilities, including inconsistent mass \\
\bottomrule
\end{tabular}
\end{table}

\begin{table}[ht]
\centering
\scriptsize
\caption{Nonmonotonic and operational distinctions.  The entries summarize
the modes used or discussed here, not every capability of each system.}
\label{tab:systemstrategies}
\begin{tabular}{L{2.8cm}L{5.7cm}L{5.4cm}}
\toprule
System & Nonmonotonic mechanism & Evaluation strategy \\
\midrule
\GK{} & recursively checked exception conditions with explicit priorities
& query-directed first-order resolution plus bounded report traversal \\
\ProbLog{} & Prolog negation on its supported fragment; does not implement
\GK{}'s recursive exception checking & relevant grounding and formula compilation \\
PASTA & stable models and credal lower/upper semantics
& ASP grounding and model analysis \\
TweetyProject RDL/DeLP & default extensions or dialectical defeat
& finite extension construction or dialectical analysis, depending on the selected module \\
clingo/DLV & stable models and negation as failure
& full grounding before model search on the tested inputs \\
I-DLV & stratified query rewriting or grounding for a solver
& Magic Sets in query mode; otherwise ASP grounding \\
s(CASP) & goal-directed stable-model reasoning
& top-down, non-ground evaluation \\
plingo & weighted or probabilistic stable-model semantics
& clingo-backed model enumeration and optimization \\
smProbLog & total choices with stable-model completion
& ASP grounding followed by ProbLog-style knowledge compilation \\
\bottomrule
\end{tabular}
\end{table}

\subsection{Related approaches}

Reiter's default logic provides the classical default-rule setting, and
prioritized variants add explicit orders or
specificity~\cite{reiter1980,brewka1994}.  Earlier \GK{} versions combined
proofs from confidence-annotated clauses with recursively checked default
exceptions and priorities, combined opposition only at the answer, and
estimated proof overlap from scalar proof
information~\cite{tammet2021,tammet2022}.  The calculations defined here
retain ground-instance overlap and resolve opposition at intermediate atoms.
Defeasible logic distinguishes ambiguity-blocking from ambiguity-propagating
proof theories~\cite{governatori2004}.  Possibilistic approaches attach
necessity degrees to arguments or uncertain defaults
\cite{alsinet2008,dupinsaintcyr2008}, while probabilistic answer-set and rule
systems construct distributions over choices and models
\cite{baral2009,cozman2020,sneyers2013}.  \GK{} applies quantitative
operations to retained first-order proofs, including their shared premises
and intermediate opposition.

Structured argumentation distinguishes evidence for a contrary conclusion,
which constitutes a rebutting attack, from an undercutting attack on a
particular inference~\cite{pollock1987,prakken2018}.  \GK{} uses that distinction to keep evidence
for an exception condition separate from ordinary negative support for a
conclusion.  Quantitative bipolar argumentation provides a related
support-and-attack perspective~\cite{rago2016,potyka2020}; Section~10 returns
to the numerical quantities defined by these systems and to argument strength.

Under distribution semantics, systems such as \ProbLog{} sample or compile
independent ground choices and compute query success
\cite{sato1995,deraedt2007,fierens2015}; the retained-proof calculation uses
the same ground-instance interpretation on a retained, one-sided proof set.
Stable-model probabilistic systems combine quantitative choices with one or
more stable models~\cite{baral2009,cozman2020,totis2023}; the
shared-threshold calculation instead combines positive and negative support
at an atom before downstream use.  Section~10 gives the detailed
interpretations.

\GK{} reconstructs uncertainty events from histories produced by a
query-directed resolution prover and does not require a finite global
grounding.  Lifted probabilistic theorem proving also uses first-order
structure, but calculates weighted-model probabilities and does not
recursively check prioritized default justifications~\cite{gogate2011}.
s(CASP) is non-ground and goal-directed but does not provide the quantitative
semantics defined here~\cite{arias2018}.  Section~\ref{sec:furtherrelated}
returns to detailed model-level comparisons after the semantics and
evaluation.

\section{Reference semantics}
\label{sec:twomodels}

The paper uses two reference constructions.  Model~1 asks which uncertain
ground clauses are present and whether the resulting clause set proves the
query.  Model~2 keeps all clauses and asks, at each ground atom, which
positive or negative support is usable.  The first model combines complete
proofs; the second resolves opposition before support is propagated.
Formally, ground-instance activation (Model~1) defines uncertainty over
ground clause-instance availability, and shared thresholds (Model~2) define
support regions over directed ground applications.  Both models are
independent of proof search; Section~6 describes the bounded calculations
that \GK{} applies after search.

\subsection{Ground-instance activation}
\label{sec:groundactivation}

For an uncertain input clause $C$ and a ground substitution $\theta$ for its
variables, define the activation event
\[
  e_{C,\theta}=\{\text{the ground instance }C\theta\text{ is active}\}.
\]
Its probability is the input confidence on $C$.  Distinct input-clause occurrences
and distinct resulting ground instances define independent events; instances
with confidence one are always active.  If two substitutions produce the same
syntactic ground instance of the same identified clause occurrence, they name
one activation event.  A derivation is available exactly when every
activation event of the ground clause instances it uses holds.  Reading a
clause confidence as the activation probability of each ground instance
matches the ground-choice convention of distribution semantics
\cite{sato1995,deraedt2007}.

A ground-instance activation world is constructed and evaluated in three steps.
\begin{enumerate}
\item \emph{Ground.}  Form the Herbrand universe of the clausified knowledge
base and query, and replace clause variables by its ground terms.  Keep each
resulting ground instance of an identified clause occurrence once.  In a
function-free input with constant set $\mathcal C$, a clause with $v$
variables has at most $|\mathcal C|^v$ such instances.  A function symbol can
make the Herbrand universe and the grounding infinite.  Grounding is part of
the model definition; Section~6 describes the bounded calculation used in
reports.
\item \emph{Delete.}  Keep each ground instance of a clause with input confidence $p$
with probability $p$ and delete it otherwise, independently for every instance.
Instances of clauses without an input confidence are always kept.  What
remains is an ordinary clause set without confidence metadata.
\item \emph{Decide.}  Write $K_\omega$ for the active clause set of the
world $\omega$.  Let $K\vdash_q Q$ mean that binary resolution, factoring,
and equality paramodulation derive the empty clause from $K$ together with
the clausified negation of the ground query $Q$, with a query clause in the
ancestry.  The subscript $q$ marks this query-ancestry requirement; it
discards contradictions derived wholly from $K$, so an inconsistent active
clause set does not thereby prove every query.  Let $K\vdash^\Delta_q Q$
denote the default-aware extension of this relation, in which
an active default application is admitted only when its ordinary body is
derivable and no admissible proof of its exception condition exists under
the stated priority and cycle policies of Section~\ref{sec:defaults}; the
superscript $\Delta$ marks the use of defaults.  The
operational exception limit of Section~\ref{sec:exceptionchecking} does not restrict this reference
relation.  Without defaults,
$\vdash^\Delta_q$ and $\vdash_q$ coincide.  Determine whether
$K_\omega\vdash^\Delta_q Q$ and, separately in the same world, whether
$K_\omega\vdash^\Delta_q\neg Q$.
On an acyclic default-dependency graph, admissibility is defined inductively:
strict derivations are admissible; a default derivation is admissible when its
body derivations are admissible and no admissible proof of its exception
condition permitted by its priority restriction exists.  Exception-proof
admissibility is evaluated inductively at earlier nodes in a topological
ordering of the acyclic dependency graph.  Cyclic cases are assigned
by the query-relative policies of Appendix~\ref{sec:defaultextensions} and lie
outside this acyclic inductive definition.  Thus
$\vdash^\Delta_q$ is the idealized default-aware reference relation.
\end{enumerate}
Each world is thus classified into one of four cells: the query only, its
negation only, both, or neither.  The model's primary quantities are the ground-instance activation
probabilities
$P_{\mathrm{act}}(Q)$, $P_{\mathrm{act}}(\neg Q)$, and their difference
$P_{\mathrm{act}}(Q)-P_{\mathrm{act}}(\neg Q)$.
The probabilities of the ``both'' and ``neither'' cells are also well defined in
the model, but they depend on the joint dependence between positive and
negative derivability, and the retained-proof calculation of
Section~\ref{sec:proofevents} does not reconstruct that dependence: when that
calculation supplies a four-component report, its conflict and ignorance
fields are a decomposition of the two proof-pool marginals rather
than the activation-world probabilities of ``both'' and ``neither''.  Explicitly,
if $P_+$ and $P_-$ are the positive and negative retained-proof union
probabilities, the proof-derived tuple is
\[
 \bigl(\max(P_+-P_-,0),\ \max(P_--P_+,0),\ \min(P_+,P_-),\
       1-\max(P_+,P_-)\bigr).
\]
This decomposition preserves both proof-pool
marginals and their signed difference, sums to one, and places the smaller
marginal wholly inside the larger.  It is a compact maximal-overlap
decomposition of the two proof-pool marginals when the joint dependence of the two proof
families is unavailable.  Its last two components are not estimates of the
activation-world ``both'' and ``neither'' cells.

\subsection{Shared thresholds}
\label{sec:sharedthreshold}

No clause is deleted here, and no proof search is performed within a
shared-threshold world: once the program is grounded and oriented, each world
is evaluated deterministically.  A \emph{support contribution} is the value
supplied to a conclusion literal by one directed application when its body has
the required usable states and its exception condition does not disable the
application.  For a ground atom $A$, the \emph{positive support pool before
opposition} aggregates contributions concluding $A$, and the \emph{negative
support pool before opposition} aggregates contributions concluding $\neg A$.
To \emph{pool} same-polarity
contributions $r_i$ means to combine them as $1-\prod_i(1-r_i)$.
The pooled values $a$ and $b$ are support before opposition resolution.
The reported components $s^+$ and $s^-$ contain only support surviving
opposition; conflict is recorded in $c$.

Once the two pooled values are $a$ and $b$, $A$ is \emph{positively usable}
when its positive support is active and its negative support is not; it is
\emph{negatively usable} when the converse holds.  In the conflict and
ignorance regions, neither polarity is usable as a premise of a downstream
application.  A shared-threshold world is defined in three steps.
\begin{enumerate}
\item \emph{Ground and orient.}  Ground as above, and apply the directed
reading to each ground clause
(Section~\ref{sec:premisesearch}): each eligible conclusion literal, with the
remaining literals as its conditions, is one directed application, and
rule-form input normally supplies one authored conclusion.
Section~\ref{sec:premisesearch} states how eligible conclusions are
determined and why the direction is treated as part of the input.
\item \emph{Draw.}  Give every ground atom $A$ a shared uniform threshold
$U_A\sim\mathrm{Uniform}(0,1)$, independently across atoms.  Ordinary
opposition and strict-priority override compare both support pools with
$U_A$.  Each pool is tested only after same-polarity pooling, not once per
default application.  Sharing $U_A$ gives maximal overlap between ordinary
opposing pools.  Contrary-gated interactions add the independent auxiliary
threshold defined in Section~\ref{sec:localrules}.  These dependence choices
are part of the semantics; they are not implied by the input confidences.
\item \emph{Evaluate.}  For every directed application
$B_1\land\cdots\land B_k\Rightarrow H$, add an edge $B_i\to H$ for each
condition literal and an edge to $H$ from its literal exception condition.
Clause bodies are conjunctions of literals.  Disjunctive source formulas are
first clausified into directed applications; compound exception conditions
must be represented by an auxiliary literal as stated in Section~3.1.
In an acyclic graph, any topological order of the resulting dependency graph
places all dependencies before the atoms that use them; unrelated atoms may
occur in either order.  The core semantics and correspondence claim are
acyclic.  The explicit cycle extensions and unsupported cycle classes are
listed in Appendix~\ref{sec:defaultextensions}.  Take the atoms in such an
order.  At an atom, each directed application whose condition atoms are
usable in the required polarity in this world produces its specified support
contribution
unless it is a blocked default.  A default is blocked exactly when its
exception condition is positively usable after local opposition and priority
restrictions have been evaluated.  Negative-usable, conflicted, and ignorant
exception states do not block it.  The contributions
concluding $A$ are pooled into a single value $a$ and those concluding
$\neg A$ into $b$.  In ordinary opposition the shared threshold $U_A$ then
gives one of four regions: $A$ is positively usable if $b<U_A\le a$ and
negatively usable if $a<U_A\le b$; it falls in the conflict region if
$U_A\le\min(a,b)$ and in the ignorance region if $U_A>\max(a,b)$.
Contrary-gated and priority interactions use the local rules of
Section~\ref{sec:localrules}.
\end{enumerate}
The four reported components are the probabilities of these four regions at the
query atom.  A shared-threshold world is evaluated through directed
applications rather than by running a theorem prover.  An atom with strong
evidence on both sides falls in the \emph{conflict} region.  Input
confidence has a different role in this model.  Whenever a rule's conditions
are usable, the rule contributes its confidence to the conclusion's pool, and
the conclusion's threshold decides the outcome; the rule is not independently
activated with that probability.
All directed ground applications are evaluated in every world, in dependency
order, so the evaluation is a function of the drawn thresholds.
Section~\ref{sec:sharedthreshold} develops the arithmetic.

\subsubsection{Same-polarity pooling}

This section develops Model~2 of Section~\ref{sec:twomodels}.  Its third step
pooled the support contributions at an atom and, in ordinary opposition,
compared the two pools against one shared threshold; both operations are
defined here, and the exception conditions
that decide which contributions are present follow in
Section~\ref{sec:defaults}.

Only support from positively or negatively usable regions is propagated.  Conflict and
ignorance are reported but are not propagated as premise states.  They are
derived when the positive and negative support pools are compared; they
are retained in the answer report to distinguish cancellation from missing
evidence.  Conflict shows that zero signed confidence is caused by opposing
support rather than by an absence of support, and it supplies the optional
conflict-sensitivity interval
(Appendix~\ref{sec:defaultextensions}).  They are not themselves propagated,
so the four components are not four equally active propagation states.

Which contributions are present at an atom is itself uncertain, because a
contribution comes from a rule whose body atoms may or may not be usable.  The
\emph{dependency set} of a ground atom $A$ contains the predecessor ground atoms
whose outcomes determine which support contributions reach $A$.  A
\emph{configuration} $\kappa$ assigns one local state to each atom of that
dependency set;
Section~\ref{sec:condfirst} shows why configurations must be evaluated
separately rather than averaged first.  For a finite acyclic relevant graph,
let $\Omega=[0,1]^m$ contain its distinct threshold variables, with product
Lebesgue measure $\mu$.  Dependency order defines each predecessor state as a
function $X_B:\Omega\to\{+,-,c,g\}$.  For a configuration $\kappa$ on dependency
set $\operatorname{Pred}(A)$, define
\[
  \Omega_\kappa=\{\omega\in\Omega:X_B(\omega)=\kappa(B)
                  \text{ for every }B\in\operatorname{Pred}(A)\},
  \qquad P(\kappa)=\mu(\Omega_\kappa).
\]
The sets $\Omega_\kappa$ are disjoint and exhaustive.  This definition reuses one
threshold variable whenever paths share a predecessor.  An exact calculation
partitions $\Omega$ at the finite threshold boundaries introduced in
dependency order and sums the cells in each $\Omega_\kappa$.  Marginal state
probabilities may be multiplied only when their defining threshold sets are
disjoint.  For a ground atom $A$ and a
configuration $\kappa$, let $r_i^+$ be the value of each positive support
contribution present in $\kappa$ and pool them as
\begin{equation}
 a_\kappa=1-\prod_i(1-r_i^+).
 \label{eq:propool}
\end{equation}
The negative support contributions are pooled in the same way,
$b_\kappa=1-\prod_j(1-r_j^-)$.  The configuration is \emph{contested} when
$a_\kappa>0$ and $b_\kappa>0$.  Noisy-or here is the defined same-polarity aggregation
operator.  Noisy-or has the usual interpretation in which distinct causes or
sources fail independently.  For uncalibrated scores or dependent sources,
noisy-or remains a chosen aggregation rule rather than a derived probability
law.  Other pooling rules express different assumptions about dependence
between the underlying sources or causes;
using one formula without specifying those assumptions generally changes the
meaning of the result~\cite{genest1986,klement2000}.
A \emph{shared predecessor} of two dependency paths is a ground atom
occurring on both paths; its sampled state is reused rather than sampled
independently.  A \emph{region} is a subset of the threshold sample space,
its \emph{probability} is the measure assigned to that subset, a \emph{report
component} is one of $(s^+,s^-,c,g)$, and a sampler's \emph{sample frequency}
is a Monte Carlo estimate of that component.
This same-polarity formula is also the accumulation rule for MYCIN's original
measures of belief and disbelief~\cite{shortliffe1975}, and it is the support
and attack aggregator in DF-QuAD~\cite{rago2016}.  The systems differ in what
they do after aggregation and in what the resulting number denotes.

\subsubsection{Shared-threshold opposition}
\label{sec:sharedlevel}

For every ground atom $A$, draw the shared threshold
$U_A\sim\mathrm{Uniform}(0,1)$, independently across atoms.  Ordinary
opposition and strict-priority override use this shared threshold for both
support pools.  A pooled support value $x$ is active exactly
when $U_A\le x$.  Using the same $U_A$ for both polarities makes the two pass
events nested and gives them the largest overlap compatible with their
marginal support.  Within configuration $\kappa$:
\[
\begin{array}{rcl}
A\text{ is positively usable} &\Longleftrightarrow& b_\kappa<U_A\le a_\kappa,\\
A\text{ is negatively usable} &\Longleftrightarrow& a_\kappa<U_A\le b_\kappa,\\
\kappa\text{ is in the conflict region} &\Longleftrightarrow& U_A\le\min(a_\kappa,b_\kappa),\\
\kappa\text{ is in the ignorance region} &\Longleftrightarrow& U_A>\max(a_\kappa,b_\kappa).
\end{array}
\]
Thus the four components for one ordinary configuration are
\begin{equation}
\begin{split}
s^+&=\max(0,a-b),\qquad s^-=\max(0,b-a),\\
c&=\min(a,b),\qquad g=1-\max(a,b).
\end{split}
\label{eq:masses}
\end{equation}

Figure~\ref{fig:sidepooling} connects the two operations: noisy-or first pools
contributions within each polarity, and only then are the two pooled values
compared against the same threshold.  It focuses on the defined operator;
alternative pooling families make different assumptions about dependence
between sources or causes, as noted above.

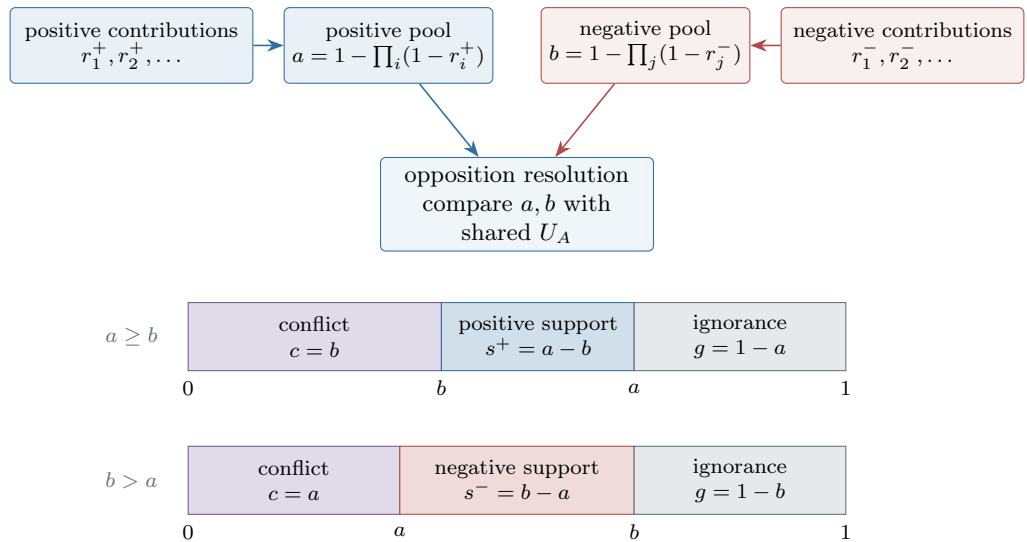
\begin{figure}[htbp]
\centering
\begin{tikzpicture}[
  side box/.style={gkbox, font=\scriptsize, text width=2.55cm,
    inner sep=3pt, minimum height=1.0cm},
  positive box/.style={side box, fill=gkblue!10, draw=gkblue},
  negative box/.style={side box, fill=gkred!9, draw=gkred}
]
  \node[positive box, text width=3.0cm] (pin) at (-5.1,2.35)
    {positive contributions\\$r_1^+,r_2^+,\ldots$};
  \node[positive box] (apool) at (-1.7,2.35)
    {positive pool\\$a=1-\prod_i(1-r_i^+)$};
  \node[negative box] (bpool) at (1.7,2.35)
    {negative pool\\$b=1-\prod_j(1-r_j^-)$};
  \node[negative box, text width=3.0cm] (nin) at (5.1,2.35)
    {negative contributions\\$r_1^-,r_2^-,\ldots$};
  \node[gkprocess, font=\footnotesize, text width=3.4cm,
    inner sep=3pt] (threshold) at (0,0.25)
    {opposition resolution\\compare $a,b$ with\\shared $U_A$};

  \draw[gkarrow, draw=gkblue] (pin) -- (apool);
  \draw[gkarrow, draw=gkred] (nin) -- (bpool);
  \draw[gkarrow, draw=gkblue] (apool) -- (threshold);
  \draw[gkarrow, draw=gkred] (bpool) -- (threshold);

  \node[font=\scriptsize\bfseries, text=gkgray] at (-5.1,-1.5) {$a\ge b$};
  \filldraw[fill=gkpurple!22, draw=gkpurple] (-4.35,-1.95) rectangle (-1.0,-1.05);
  \filldraw[fill=gkblue!22, draw=gkblue] (-1.0,-1.95) rectangle (1.55,-1.05);
  \filldraw[fill=gkgray!16, draw=gkgray] (1.55,-1.95) rectangle (4.35,-1.05);
  \node[font=\scriptsize, align=center] at (-2.68,-1.5) {conflict\\$c=b$};
  \node[font=\scriptsize, align=center] at (0.28,-1.5) {positive support\\$s^+=a-b$};
  \node[font=\scriptsize, align=center] at (2.95,-1.5) {ignorance\\$g=1-a$};
  \node[font=\scriptsize] at (-4.35,-2.18) {$0$};
  \node[font=\scriptsize] at (-1.0,-2.18) {$b$};
  \node[font=\scriptsize] at (1.55,-2.18) {$a$};
  \node[font=\scriptsize] at (4.35,-2.18) {$1$};

  \node[font=\scriptsize\bfseries, text=gkgray] at (-5.1,-3.4) {$b>a$};
  \filldraw[fill=gkpurple!22, draw=gkpurple] (-4.35,-3.85) rectangle (-1.55,-2.95);
  \filldraw[fill=gkred!20, draw=gkred] (-1.55,-3.85) rectangle (1.55,-2.95);
  \filldraw[fill=gkgray!16, draw=gkgray] (1.55,-3.85) rectangle (4.35,-2.95);
  \node[font=\scriptsize, align=center] at (-2.95,-3.4) {conflict\\$c=a$};
  \node[font=\scriptsize, align=center] at (0,-3.4) {negative support\\$s^-=b-a$};
  \node[font=\scriptsize, align=center] at (2.95,-3.4) {ignorance\\$g=1-b$};
  \node[font=\scriptsize] at (-4.35,-4.08) {$0$};
  \node[font=\scriptsize] at (-1.55,-4.08) {$a$};
  \node[font=\scriptsize] at (1.55,-4.08) {$b$};
  \node[font=\scriptsize] at (4.35,-4.08) {$1$};
\end{tikzpicture}
\caption{Same-polarity pooling followed by shared-threshold opposition
resolution.  Here $U_A\sim\mathrm{Uniform}(0,1)$ is the single threshold shared
by the positive and negative support pools for ground atom $A$.  The
interval diagrams show the two possible orderings of the pooled values.  They
make visible why the four-component report distinguishes conflict from
ignorance while signed confidence keeps only direction.}
\label{fig:sidepooling}
\end{figure}

For any query literal $L$, define its \emph{signed confidence} by
\begin{equation}
 C(L)=s^+(L)-s^-(L).
 \label{eq:confidence}
\end{equation}
The sign gives the direction of the quantitative assessment when nonzero;
zero has no quantitative direction.  In dependency-aware reporting this
direction can determine accepted-versus-rejected list placement.
Compatibility mode instead uses the recursive exception check and can
therefore place a candidate differently from the report's direction
(Section~\ref{sec:reporting}).  When direction is already conveyed by list
placement, the displayed magnitude is $|C(L)|$.  The four-component report retains the positive and
negative support from which this confidence is computed, together with
conflict and ignorance.

For example, $a=0.7$ and $b=0.4$ give
$(s^+,s^-,c,g)=(0.3,0,0.4,0.3)$ and signed confidence $C=0.3$.  Reporting
signed confidence alone would omit the distinction between the $0.4$ conflict region
and the $0.3$ ignorance region.

\begin{proposition}[Partition]
The four quantities in Equation~\eqref{eq:masses} are non-negative, sum to
one, and at most one of $s^+$ and $s^-$ is non-zero.
\end{proposition}

\begin{proof}
If $a\ge b$, the tuple is $(a-b,0,b,1-a)$; if $b>a$, it is
$(0,b-a,a,1-b)$.  The claim follows in both cases.
\end{proof}

\begin{proposition}[Robust dependence bounds]
Let $E_+$ and $E_-$ be events with marginal probabilities $P(E_+)=a$ and
$P(E_-)=b$, representing activation of the pooled positive and negative
support.  Then
\begin{align}
P(E_+\setminus E_-)-P(E_-\setminus E_+)&=a-b, \label{eq:signed}\\
\inf P(E_+\setminus E_-)&=\max(0,a-b), \label{eq:frechet}\\
\sup P(E_+\setminus E_-)&=\min(a,1-b), \label{eq:frechetsup}
\end{align}
where the extrema range over all couplings of $E_+$ and $E_-$ with these
marginals.  The shared-threshold coupling attains the lower bound.
\end{proposition}

\begin{proof}
Write $q=P(E_+\cap E_-)$.  The left side of Equation~\eqref{eq:signed} is
$(a-q)-(b-q)=a-b$.  Also $P(E_+\setminus E_-)=a-q$.  The feasible interval for
$q$ is $[\max(0,a+b-1),\min(a,b)]$, which gives the two extrema.  With one
uniform $U$, take $E_+=\{U\le a\}$ and $E_-=\{U\le b\}$.  Their intersection has
probability $\min(a,b)$, attaining the lower bound.
\end{proof}

The unopposed positive support is consequently the Fr\'echet lower bound on
support that survives explicit opposition when only the two marginals are
known.  If maximal overlap is not justified, the interval
\[
  [\max(0,a-b),\min(a,1-b)]
\]
provides a sensitivity interval over all couplings; independence gives the interior
value $a(1-b)$.  The shared-threshold semantics uses the lower bound.
Equation~\eqref{eq:signed} also explains why the signed confidence $a-b$ used
in \GK{}~\cite{tammet2021} is independent of the unknown
coupling, although signed confidence alone omits conflict and ignorance.

Section~\ref{sec:poolingrules} compares this choice with alternative
opposition and pooling rules.

\subsubsection{Configuration-wise evaluation}
\label{sec:condfirst}

The final unconditional result is obtained by marginalizing over predecessor
configurations, but opposition must first be evaluated within each
configuration; opposing pools cannot in general be combined only after all
derivations have been flattened to marginals.  Let $\kappa$ range over the
mutually exclusive assignments to the dependency set of a ground atom $A$
that determine which support contributions for $A$ and $\neg A$
are present.  Conditional on $\kappa$, their pooled support values are
$a_\kappa$ and $b_\kappa$.
The unconditional report is
\begin{equation}
 (s^+,s^-,c,g)=
 \sum_\kappa P(\kappa)\bigl(
   \max(a_\kappa-b_\kappa,0),\max(b_\kappa-a_\kappa,0),
   \min(a_\kappa,b_\kappa),1-\max(a_\kappa,b_\kappa)
 \bigr).
 \label{eq:conditionfirst}
\end{equation}
Both total support directions can be positive because different
configurations can favor different sides, while the directions remain
disjoint in each world.

The case that motivated premise-level opposition resolution has a simpler closed form.

\begin{proposition}[One contested premise]
Let the only derivation of $H$ be a positive rule
$r::H\leftarrow B$.  Suppose the ordinary positive and negative pools at $B$ are $a$ and
$b$, there is no direct evidence about $H$, and the rule application and the
uncertainty determining $B$ have no shared uncertain predecessors.  Then
\[
  R(H)=\bigl(r\max(a-b,0),\;0,\;0,\;1-r\max(a-b,0)\bigr).
\]
In particular, equal opposing premise support yields pure ignorance at the
conclusion.
\end{proposition}

\begin{proof}
The premise is positively usable on the threshold interval
$b<U_B\le a$, whose length is $\max(a-b,0)$.  Only on that interval is the rule's
support contribution present at $H$.  In those worlds, the positive support
pool at $H$ has value $r$.  Since $U_H$ is independent of $U_B$, the condition
$U_H\le r$ holds in a fraction $r$ of them.
There is no negative support at $H$, so the remaining worlds are ignorant
rather than negative-only or conflicted.
\end{proof}

Thus $0.5::B$, $0.5::\neg B$, and $0.9::H\leftarrow B$ give pure ignorance
for $H$, rather than the query-level value $0.45$.  With $0.5$ for $B$ and
$0.2$ for $\neg B$,
the result is $0.9(0.5-0.2)=0.27$.  Evidence for $\neg B$ weakens the use of
$B$; it does not by contraposition become evidence for $\neg H$.

Consider a positive fact for $bird(a)$ with input confidence $0.5$.  A
negative rule has confidence $0.9$ and is present only when a
predecessor atom is usable, an event
of probability $0.2$.  Subtracting the marginal negative support $0.18$ from $0.5$
gives $0.32$.
Equation~\eqref{eq:conditionfirst} instead gives positive support
$0.8\cdot0.5=0.4$ and negative support $0.2\cdot(0.9-0.5)=0.08$.  The latter
calculation preserves the fact that the negative derivation is absent in 80\% of
the configurations.

A rule uses a premise only in the region where it is positively usable after opposition has
been resolved.  When premise and
rule paths share predecessors, the configurations must enumerate the shared
elements and their four possible outcomes.  Chains without shared predecessors
reduce to products.  A simple shared-predecessor example illustrates the
reason: $0.7$ for the predecessor,
$0.8$ for a premise rule, and $0.9$ for a conclusion rule give
$0.7\cdot0.8\cdot0.9=0.504$.  Multiplying two already marginalized links
would count the $0.7$ predecessor twice and return $0.3528$.

The shared-threshold model is evaluated on the complete relevant directed
ground dependency graph; Section~\ref{sec:premisesearch} describes how \GK{}
reconstructs a bounded query-relevant part of that graph.

\subsection{Comparison on a contested premise}

Take the contested-premise program of Section~\ref{sec:example}:
\[
0.5::bird(a),\qquad 0.5::\neg bird(a),\qquad
0.9::flies(x)\leftarrow bird(x),\qquad Q=flies(a).
\]
The only constant is $a$, so grounding replaces $x$ by $a$ and yields three
ground clauses, all uncertain: the two facts and the single rule instance
$0.9::flies(a)\leftarrow bird(a)$.  Model~1 therefore has $2^3=8$ worlds,
listed in Table~\ref{tab:eightworlds}.

\begin{table}[htbp]
\centering
\small
\caption{The eight ground-instance activation worlds of the contested-premise
program.  The first three columns say whether the instance was kept or deleted.
The query is derivable in the two worlds that keep both $bird(a)$ and the rule
instance; its negation is derivable in none.}
\label{tab:eightworlds}
\begin{tabular}{@{}cccrcc@{}}
\toprule
$bird(a)$ & $\neg bird(a)$ & $flies(a)\leftarrow bird(a)$ &
  probability & $flies(a)$? & $\neg flies(a)$? \\
\midrule
kept & kept & kept & $.225$ & yes & no \\
kept & kept & deleted & $.025$ & no & no \\
kept & deleted & kept & $.225$ & yes & no \\
kept & deleted & deleted & $.025$ & no & no \\
deleted & kept & kept & $.225$ & no & no \\
deleted & kept & deleted & $.025$ & no & no \\
deleted & deleted & kept & $.225$ & no & no \\
deleted & deleted & deleted & $.025$ & no & no \\
\bottomrule
\end{tabular}
\end{table}

The query column of Table~\ref{tab:eightworlds} contains ``yes'' in exactly
two rows --- the worlds keeping both $bird(a)$ and the rule instance --- each
of probability $0.5\cdot0.5\cdot0.9=0.225$; their sum is
$P_{\mathrm{act}}(flies(a))=0.225+0.225=0.45$.  The negation is derivable in no row: no
clause of the program concludes $\neg flies$, and in the two worlds keeping
both $bird(a)$ and $\neg bird(a)$ the resulting contradiction yields only
refutations that do not use the query, which $\vdash_q$ discards.  Model~1
therefore gives $P_{\mathrm{act}}(Q)=0.45$, $P_{\mathrm{act}}(\neg Q)=0$, and
difference $0.45$.

Shared-threshold semantics (Model~2) uses the shared threshold at each atom instead of deleting
clauses.  At
$bird(a)$, the positive and negative pools are both $0.5$.  For
$U_{bird(a)}\le0.5$ the atom falls in the conflict region; for
$U_{bird(a)}>0.5$ it falls in the ignorance region.  It is therefore
positively usable in no world.  The rule instance
$flies(a)\leftarrow bird(a)$ contributes its rule confidence $0.9$ only in worlds
where $bird(a)$ is positively usable, so it contributes in no world.  No
other application supports either $flies(a)$ or $\neg flies(a)$, so both
pools at the query atom are zero, and $flies(a)$ falls in the ignorance region in every
shared-threshold world.  Model~2 gives $R(flies(a))=(0,0,0,1)$, with the
conflict recorded
at $bird(a)$, where it arises, rather than at the query.  Model~2 still
applies the rule: the condition is evaluated first, and its state determines
whether the application contributes.

The two models calculate different quantities.  Model~1 counts worlds in
which the active clauses prove the query, and the clause $\neg bird(a)$ never
obstructs a proof that uses $bird(a)$: a world may contain both clauses, and
derivability from the query-relevant clauses is unaffected by their
inconsistency.  Model~2 resolves the opposing support for $bird(a)$ before
that premise can support $flies(a)$, so equally strong opposition leaves
nothing to propagate.  This paper uses Model~2 for premise-level opposition.
Table~\ref{tab:twostages} summarizes the two models.
Section~6 states how \GK{} selects between its bounded calculations, and
Section~\ref{sec:differences} classifies the resulting model and coverage
differences.

\begin{table}[!htb]
\centering
\small
\caption{Semantic comparison of the two reference constructions.}
\label{tab:twostages}
\begin{tabular}{@{}L{2.7cm}L{5.05cm}L{5.05cm}@{}}
\toprule
 & Model 1: ground-instance activation & Model 2: shared thresholds \\
\midrule
What is random
& Whether each ground instance of an uncertain clause is present
& One shared threshold per ground atom, with an additional independent
  threshold for contrary-gated interactions \\
What happens in one world
& The active clauses are tested for default-aware
  $\vdash^\Delta_q$-derivability of the query and of its negation
& All directed ground applications are evaluated in dependency order, with
  a least fixpoint only for finite monotone same-polarity components and named
  policies for the other supported cycles \\
Default application
& In each activation world, apply the recursive exception, priority, and
  cycle policy
& Evaluate the body and exception first; the default contributes only where
  the body is usable and the exception is not positively usable \\
Uncertain exception-condition opposition
& An active proof of the exception condition can block even when a proof of
  its explicit negation is also active
& Positive and negative support at the exception condition are resolved before it is
  used \\
Contrary-gated default
& Clause activation followed by an all-or-nothing exception check
& A separate local rule; ordinary support for the opposite literal blocks the default
  in its active region \\
What is counted
& The fraction of worlds deriving the query, minus the fraction deriving its
  negation
& The probabilities of the regions where the query atom is positively usable,
  negatively usable, conflicted, or ignorant \\
\bottomrule
\end{tabular}
\end{table}

\subsection{Defaults and exceptions}
\label{sec:defaults}

\subsubsection{One uncertain exception condition}

A default application has three components: the usable state of its
ordinary body, the usable state of its exception condition, and the rule
confidence applied when the exception does not block it.  In a shared-threshold
world, the body and exception are evaluated first.  If the body is not usable,
the application contributes nothing.  If the body is usable and the exception
is positively usable, the default is blocked.  Otherwise the rule contributes
its stated confidence to the head's support pool.  Evidence for the explicit
negation of the exception condition does not support the head directly; it
only reduces the region in which the exception blocks the rule.

Here \emph{ordinary support} means support contributed by a fact, a
non-default rule, or a default whose exception condition is not the negation
of its conclusion.  After any exception condition has been evaluated, this
support enters the relevant polarity pool directly.  The default above instead
describes an undercutting exception.
When the exception condition is
also the negation of the head, the same evidence additionally rebuts the
conclusion; that contrary-gated case is treated separately below.  The main
local default cases are summarized below.

\begin{center}
\small
\begin{tabular}{L{3.3cm}L{5.5cm}L{5.0cm}}
\toprule
Case & Effect on the rule application & Combination rule \\
\midrule
Arbitrary exception $B$ to head $H$
& usable support for $B$ undercuts that application only
& evaluate $B$ first; omit the application where $B$ is positively usable \\
Exception $\neg H$ to head $H$
& the same contribution undercuts the default and rebuts $H$
& contrary-gated rule; see Section~\ref{sec:localrules} \\
Opposed ranked defaults
& each conclusion is the other's exception
& mutual blocking at equal rank; strict-priority override at unequal rank \\
Paired reference-class extension
& an explicit exception rule partitions the blocked class
& optional class-frequency rule of Appendix~\ref{sec:pairedexceptions} \\
\bottomrule
\end{tabular}
\end{center}

\begin{quote}
\small
\textsc{EvaluateDefault}($D=r::H\leftarrow A\ [\mathrm{unless}\ B]$): evaluate
$A$ and $B$ first; if $A$ is not usable, return no contribution; if $B$ is
positively usable, return zero in the core semantics; the optional paired
reference-class construction defines the only nonzero blocked-branch case
considered here; otherwise contribute rule confidence $r$ to the support pool of $H$.
\end{quote}
If $B=\neg H$, the contrary-gated local rule of
Section~\ref{sec:localrules} is used because exception-condition evaluation and head
opposition are no longer acyclically separable.

Consider the ground program
\[
  .7::B,\qquad .3::\neg B,\qquad
  .9::H\ [\mathrm{unless}\ B].
\]
At $B$, the positive and negative pools are $.7$ and $.3$, so
\[
 R(B)=(.4,0,.3,.3).
\]
Thus $B$ is positively usable, and blocks the default, in $.4$ of the
threshold worlds.  Conflict at $B$ does not block the rule.  The remaining
$.6$ is unblocked, and the default's rule confidence $.9$ gives
\[
  s^+(H)=.6\cdot.9=.54,\qquad
  R(H)=(.54,0,0,.46).
\]
The threshold regions make the calculation explicit:
\begin{center}
\small
\begin{tabular}{L{5.0cm}rcc}
\toprule
Region at $B$ & Region probability & Blocked? & Contribution to $H$ \\
\midrule
$U_B\le .3$: conflict & .3 & no & $.3\cdot.9$ \\
$.3<U_B\le .7$: positively usable & .4 & yes & 0 \\
$U_B>.7$: ignorance & .3 & no & $.3\cdot.9$ \\
\bottomrule
\end{tabular}
\end{center}
The $.3$ negative pool at $B$ does not support $H$; it reduces the
region where the exception condition is positively usable from $.7$ to $.4$.

Ground-instance activation gives $.27$ instead.  In an activation world,
the default is blocked whenever the positive $B$ clause is active; an
independently active $\neg B$ clause does not remove the derivation of
$B$.  Including the default's own activation probability, the positive query
probability is therefore
\[
  .9(1-.7)=.27.
\]
Shared-threshold semantics first resolves the $.7/.3$ opposition at the
exception condition and therefore gives $.54$.  This is a difference between the two
reference semantics, not a sampling discrepancy.

\subsubsection{Blocked and unblocked cases}

The general definition is configuration-wise.  For each predecessor
configuration $\kappa$, the support contributed by a default application
$\delta=r::H\leftarrow A\ [\mathrm{unless}\ B]$ is
\[
\operatorname{contrib}_\delta(\kappa)=
\begin{cases}
0, & A\text{ is not usable in }\kappa,\\
r_{\mathrm{blocked}}, & A\text{ is usable and }B\text{ is positively usable in }\kappa,\\
r_{\mathrm{open}}, & A\text{ is usable and }B\text{ is not positively usable in }\kappa.
\end{cases}
\]
The core complete-blocking semantics has $r_{\mathrm{open}}=r$ and
$r_{\mathrm{blocked}}=0$.  The optional paired
reference-class construction of Appendix~\ref{sec:pairedexceptions} defines
nonzero blocked-branch components under its stated syntactic restrictions.

All contributions are pooled within each configuration before marginalizing
over configurations.  If
$\mathcal D^+$ is the set of positive applications at the head and
$\operatorname{contrib}_\delta(\kappa)$ includes the corresponding
ordinary-rule contribution when $\delta$ is
not a default, then
\begin{equation}
 a_\kappa=1-\prod_{\delta\in\mathcal D^+}
   \bigl(1-\operatorname{contrib}_\delta(\kappa)\bigr),
\label{eq:defaultconfigpool}
\end{equation}
with an analogous $b_\kappa$ for negative applications.  The local opposition or
combination rule is applied to $a_\kappa,b_\kappa$, and only then are its four
regions weighted by $\Pr(\kappa)$ and summed.  This order preserves shared predecessors between the
default, its body and exception condition, and other head applications.

For one isolated default, the configuration-wise expression reduces to an
expected support contribution conditional on the body being usable.  The same reduction applies in a
setting without shared predecessors in which later
scalar pooling is justified.  Conditional on the ordinary body $A$ being
usable, write
\[
 \beta_{B\mid A}
 =\Pr(B\text{ is positively usable}\mid A\text{ is usable}).
\]
This is a usable-region probability, not a truth probability for $B$.  The
support contributions in the unblocked and blocked classes are
$r_{\mathrm{open}}$ and $r_{\mathrm{blocked}}$, respectively.  Their
conditional expectation is
\begin{equation}
 r_{\mathrm{open}}(1-\beta_{B\mid A})+r_{\mathrm{blocked}}\beta_{B\mid A}.
\label{eq:blockmix}
\end{equation}
The unconditional expected support contribution is
\[
 \Pr(A\text{ is usable})\,
 \bigl[r_{\mathrm{open}}(1-\beta_{B\mid A})+
 r_{\mathrm{blocked}}\beta_{B\mid A}\bigr].
\]
Only under independence may $\beta_{B\mid A}$ be replaced by the marginal
probability that $B$ is positively usable.  For several complete-blocking
exception conditions whose positively usable regions and relevant predecessor support are independent, the
independent shortcut conditional on the body is
\[
 r_{\mathrm{open}}\prod_{i=1}^k(1-\beta_{B_i\mid A}).
\]
Shared body, exception-condition, or head support instead requires the configuration-wise
calculation above.

\begin{proposition}[Conditional-support preservation]
Suppose the positively usable state of $B$ partitions the usable-body cases
into an unblocked class of conditional probability $1-\beta_{B\mid A}$ and a blocked
class of conditional probability $\beta_{B\mid A}$.  If one isolated default
application contributes $r_{\mathrm{open}}$ and $r_{\mathrm{blocked}}$ in the
unblocked and blocked classes, respectively, then
Equation~\eqref{eq:blockmix} is its conditional expected support contribution.
In particular, complete blocking contributes $r_{\mathrm{open}}$ throughout
the unblocked class.
\end{proposition}

\begin{proof}
The two classes are disjoint and exhaustive within the usable-body cases.
Their weighted contributions are
$r_{\mathrm{open}}(1-\beta_{B\mid A})$ and
$r_{\mathrm{blocked}}\beta_{B\mid A}$; adding them gives
Equation~\eqref{eq:blockmix}.  Dividing the first term by the
unblocked probability returns $r_{\mathrm{open}}$.
\end{proof}

This proposition explains why exception-condition support is not subtracted directly from
the rule confidence.  Direct subtraction would return
$\max(r_{\mathrm{open}}-\beta_{B\mid A},0)$ for a
complete blocking.  It imposes a maximal-overlap coupling between the
application event
and the event in which the exception condition is positively usable and
generally fails to preserve $r_{\mathrm{open}}$
conditional on the exception condition not being positively usable.  A
supported exception condition is
therefore an undercutting attack on one inference; evidence for the negated
head is a rebutting attack and enters the atom's negative pool.  This provides
a quantitative implementation of Pollock's distinction between undercutting
and rebutting attacks~\cite{pollock1987}.
Prakken instead derives argument strength from a probability distribution over
the premises and conclusions of structured arguments~\cite{prakken2018}.
Possibilistic uncertain-default formalisms keep default selection and
certainty propagation as two stages~\cite{dupinsaintcyr2008}.

The limiting cases $\beta_{B\mid A}=0$ and $\beta_{B\mid A}=1$ select the
unblocked and blocked branches exactly.  With confidence-one inputs, priority assigns the
whole overlapping threshold region to the higher-ranked side.  The application
is then either retained or removed.

\subsubsection{Contrary-gated defaults}
\label{sec:localrules}

A default whose exception condition is the negation of its own conclusion ---
$flies(x)$ unless $\neg flies(x)$ --- is here called a
\emph{contrary-gated default}:
the evidence that makes its exception positively usable is also evidence for
the opposite conclusion.  This differs from an arbitrary exception such as
$flies(a)$ unless $injured(a)$: evidence for $injured(a)$ undercuts that
flying rule but does not support $\neg flies(a)$.  In the contrary-gated case
alone, the exception-condition support also enters the opposite head pool.

For an atom involved in a contrary-gated interaction, draw an auxiliary gate
threshold $V_A$, independently of $U_A$.  The default-versus-ordinary and
equally ranked contrary-gated cases use the independent thresholds $U_A$ and
$V_A$.  Strict-priority override is defined separately below and uses the
shared threshold $U_A$.  All present same-polarity contributions are
combined first, and each resulting pool is compared with one threshold,
rather than every default application receiving a separate threshold.  Let
$a$ and $b$ be the pooled positive and negative support values.  If neither
side is a contrary-gated default,
both instead use the single shared threshold $U_A$ of
Section~\ref{sec:sharedlevel}:
\[
(s^+,s^-,c,g)=(\max(a-b,0),\max(b-a,0),\min(a,b),1-\max(a,b)).
\]

The independent pair is a semantic postulate for this local default case.  It
makes a contrary-gated positive pool with value $a$ retain that value where
its exception condition is not positively usable, while a simultaneous
positive test for the exception disables the application.  A shared threshold would instead treat
the two pools as ordinary opposing evidence and put their overlap in conflict.
The change is not inferred from the provenance of the contributions; a
different coupling defines a different default semantics.

Now assume that the positive side is a contrary-gated default pool and the
negative side is ordinary support.  Ordinary opposing support with pooled
value $b$ is active when $U_A\le b$.  The contrary-gated default contribution
with value $a$ passes its auxiliary test when $V_A\le a$, and it is usable
only where the ordinary opposing support does not pass the shared-threshold
test, $U_A>b$.  The regions are exclusive:
\[
  (s^+,s^-,c,g)=\bigl(a(1-b),\; b,\; 0,\; (1-a)(1-b)\bigr).
\]
A default with rule confidence $0.6$ opposed by an ordinary fact with
input confidence $0.9$ reports $(0.06,0.90,0,0.04)$.  The mirror case, where
the negative side is the contrary-gated
default, is symmetric; it covers queries for the explicit negation, such
as ``which object does not fly'' against a flying default.

\subsubsection{Opposing defaults}

When both sides are contrary-gated defaults with the same priority, the
two thresholds $U_A$ and $V_A$ are independent and each side contributes
only if the other threshold test fails.  By convention the positive pool is
tested by $U_A$ and the negative pool by $V_A$:
\[
  (s^+,s^-,c,g)=\bigl(a(1-b),\; b(1-a),\; 0,\; (1-a)(1-b)+ab\bigr).
\]
The four threshold regions are:
\begin{center}
\small
\begin{tabular}{ccc}
\toprule
$U_A\le a$ & $V_A\le b$ & Outcome \\
\midrule
yes & no & positive-only \\
no & yes & negative-only \\
yes & yes & mutual blocking; ignorance \\
no & no & ignorance \\
\bottomrule
\end{tabular}
\end{center}
At $a=b=1$ this is the Nixon diamond: the whole probability is in the
ignorance region.
Mutual blocking yields ignorance because neither rule application remains
enabled after its exception condition is evaluated.  When both threshold tests pass, each default finds its
exception active and both rule applications are disabled.  No conclusion is
usable in that world, so the region is ignorance rather than conflict.  Two
certain contradictory facts remain active on both sides and instead yield
conflict.
With rule confidences $0.6$ and $0.9$ the
report is $(0.06,0.36,0,0.58)$.

\subsubsection{Priorities}

If both sides are contrary-gated defaults and one has strictly higher priority,
strict-priority override uses the shared threshold $U_A$, rather than the two
independent tests used for equally ranked contrary-gated defaults.  The overlap region is assigned to the
higher-priority side, and the lower-priority side keeps only its excess:
\[
  (s^+,s^-,c,g)=\bigl(a,\; \max(b-a,0),\; 0,\; 1-\max(a,b)\bigr)
  \qquad(\text{positive side higher}).
\]
A $0.6$ default at priority 3 against a $0.9$ default at priority 2 reports
$(0.6,0.3,0,0.1)$: the higher-priority side retains its full support, while the
lower-priority side retains only support outside the overlap.  At rule confidence
one this is deterministic strict-priority override.

\begin{proposition}[Local partitions and deterministic limits]
For $a,b\in[0,1]$, each tuple in this family of local combination rules is
non-negative and sums to one.  At confidence-one inputs, certain ordinary
contrary support defeats a contrary-gated default, two equally ranked
contrary-gated
defaults yield pure ignorance, and the
strictly higher-priority default receives full support.
\end{proposition}

\begin{proof}
For a default opposed by ordinary support the sum is
$a(1-b)+b+(1-a)(1-b)=1$.  For equal priorities it is
$a(1-b)+b(1-a)+(1-a)(1-b)+ab=1$.  For strict-priority override,
if $a\ge b$ the tuple is $(a,0,0,1-a)$, and if $b>a$ it is
$(a,b-a,0,1-b)$.  Non-negativity is immediate.  Substitution of $a=b=1$
gives the three stated deterministic limits.
\end{proof}

Strict-priority override applies only to ranked defaults.  A \emph{ranked
default} has an explicit positive priority $\pi\ge1$.  An \emph{unranked
default} has priority zero; an omitted priority is parsed as zero.  Priority
zero means unranked; it does not place the default below every positive
priority.  Ordinary
non-default support has no default
priority.  Strict-priority override applies exactly when every present
contribution on both sides comes from a ranked default, with the maximum
priorities compared strictly.  A present unranked default or ordinary support
prevents strict-priority override, and the two pools are combined by the
applicable ordinary rule.  A priority-zero exception condition may disable
its own rule application, but the corresponding support is not entered into
the priority comparison at the conclusion atom.  Equal priorities never
defeat each other; they produce the symmetric mutual blocking
above.

When several ranked contributions occur on one side, all present contributions
on that side are pooled.  The maximum present priority determines which pooled
side receives the overlap; lower-ranked contributions on that side remain in
its pool.  This differs from the priority restriction inside an exception
check, where support derived through a lower-priority default is excluded.
Strict-priority override is defined only when every contribution in both pools
comes from a positively ranked default.  If ordinary or unranked support is present, there is no rank for that
contribution, so the semantics uses the applicable ordinary combination rule
for the full pools.  This switch between combination rules is part of the
definition, even when
the ordinary contribution is small.

\begin{center}
\small
\begin{tabular}{L{4.5cm}L{4.2cm}L{5.0cm}}
\toprule
Positive and negative sides & Priority condition & Local reference rule \\
\midrule
ordinary support / ordinary support
& none & shared-threshold opposition \\
contrary-gated default / ordinary support
& none & default-versus-ordinary rule \\
two contrary-gated defaults
& same explicit rank $r\ge 1$ & mutual blocking \\
two contrary-gated defaults
& unequal explicit ranks $r,s\ge 1$ & strict-priority override \\
one pool contains an acyclic priority-zero default
& priority-zero contribution is treated as unranked
& select the local rule using the mixed-case procedure below \\
opposite priority-zero defaults
& reciprocal exception cycle & ignore the internal cycle edges; combine the
active applications as ordinary support without default priority \\
\bottomrule
\end{tabular}
\end{center}

Mixed cases are resolved configuration by configuration.  Pool each polarity,
then classify a side as \emph{ranked-default-only} exactly when every present
contribution on it is a contrary-gated default with positive rank.  The local rule is selected
exhaustively:
\begin{enumerate}
\item if both sides are ranked-default-only, compare the maximum ranks and use mutual blocking or strict-priority
override;
\item if exactly one side is ranked-default-only, use the default-versus-ordinary
rule, treating the other pool as ordinary even when that pool also contains
ranked defaults;
\item if neither side is ranked-default-only, use ordinary shared-threshold
opposition for the two full pools.
\end{enumerate}
Before this classification, the reciprocal priority-zero rule removes its two
internal exception edges.  Other priority-zero applications are unranked and
therefore prevent a side from being ranked-default-only.  External
evidence remains present.  Thus a pool containing both ranked-default and
ordinary or unranked contributions falls under item 2 or 3, depending on the
composition of the opposite pool.

For each directed application, self-reference is determined relative to that
application's own conclusion literal, independently of the polarity of the
current query.  In a reciprocal priority-zero pair, the exception reference
from either application to the opposite application is an internal cycle
reference.  These two internal references do not disable the pair.  Exception
support entering the pair from any other derivation remains external and may
disable the application it reaches.  After the internal references have been
removed, the active positive and negative applications are combined as
ordinary unranked support.  A default whose exception condition is its own
conclusion remains self-blocking and contributes nothing.

\subsection{Exactness conditions}

For the correspondence claim, an evaluated query is \emph{within the
correspondence fragment} when the following semantic conditions hold:
\begin{itemize}
\item the relevant directed ground program and dependency graph are finite and
acyclic;
\item every exception uses a local interaction form defined in
Section~\ref{sec:defaults};
\item no retained answer proof depends on a classical factoring step: the
directed shared-threshold evaluation defines no counterpart of classical
factoring, so factoring-derived answers are outside the fragment
(Section~\ref{sec:results}, Example~12).
\end{itemize}
The current correspondence claim additionally requires the following
implementation conditions:
\begin{itemize}
\item the report-time traversal enumerates every relevant directed ground
application and every generated premise is ground;
\item every support configuration and shared predecessor needed by the query is
enumerated without truncation or deadline;
\item the coarse dependency index does not merge distinct ground predecessors
that share a predicate-symbol and arity key and must be evaluated separately;
\item the operational exception filter admits every directed exception
derivation required by the reference calculation.
\end{itemize}
This is a condition on an evaluated query, not a restriction on the input
language accepted by the prover.  The exception-filter item is an operational
precondition for correspondence with \GK{}; it is not part of the
shared-threshold world definition.

\begin{proposition}[Dependency-evaluator correspondence]
\label{prop:dependencycorrespondence}
For an evaluated query within the correspondence fragment defined above, the
dependency-aware evaluator returns the query-region measures induced by the
shared-threshold world definition.
\end{proposition}

\begin{proof}[Proof sketch.]
Order the finite dependency graph from premises to the query.  At an input
atom, same-polarity pooling followed by the shared threshold gives exactly the
four regions of Equation~\eqref{eq:masses}.  Inductively assume that the
evaluator and the world definition agree on every body atom.  They then enable
the same ground rule contributions in every configuration.  For a default
application, the induction hypothesis makes the evaluator and threshold world
agree on whether the body is usable and whether the exception condition is
positively usable;
both therefore include the rule contribution in exactly the same
configurations.  In a contrary-gated interaction, both use the shared and
auxiliary threshold regions specified in Section~\ref{sec:localrules}.  The
local priority rules partition those configurations by their defining
threshold conditions, while shared-predecessor enumeration reuses the same
predecessor outcome on every dependent path.  Hence the evaluator sums exactly
the same threshold-region measures at the head atom, and induction gives the
claim at the query.  This is exact evaluation of measures induced by the
continuous shared-threshold world definition, not enumeration of individual
threshold worlds.
\end{proof}

\subsection{Polarity symmetry}
\begin{corollary}[Polarity invariance of completed reports]
\label{cor:polarity}
Let $A$ be a closed ground atom for which dependency-aware evaluation
completes within the correspondence fragment.  If
\[
  R(A)=(s^+(A),s^-(A),c(A),g(A)),
\]
then evaluation of the opposite-polarity query satisfies
\[
  R(\neg A)=(s^-(A),s^+(A),c(A),g(A)).
\]
Consequently,
\[
  C(\neg A)=-C(A).
\]
Both orientations have the shared-threshold calculation method, are within
the fragment, and are shared-threshold partitions.
\end{corollary}

\begin{proof}
The shared-threshold world definition evaluates one pair of positive and
negative pools at the canonical ground atom $A$.  Reorienting the query as
$\neg A$ exchanges the names of the positive-only and negative-only regions
but leaves the conflict and ignorance regions unchanged.  The relevant
directed dependency graph, its acyclicity, and the completion conditions are
properties of that same atom-level calculation and are unchanged by the
presentation orientation.  The correspondence proposition transfers these
region measures and statuses to every completed within-fragment evaluator result.
\end{proof}

\paragraph{Implementation.}
The dependency-aware evaluator follows this orientation directly.  It canonicalizes
the queried ground atom, evaluates the positive orientation once, and obtains
the report for the explicitly negated query by exchanging the positive and
negative components.  Classification of a self-referential exception is made
relative to the head literal of its own directed application.
Self-reference classification is therefore independent of query orientation.
The optional report detail states whether the component-swap relation is
guaranteed.

The correspondence proposition is stated for a finite acyclic dependency
graph.  The reciprocal priority-zero case evaluated in
Section~\ref{sec:defaultcalcs} is a separately defined and independently tested
cyclic extension.  It is not covered by the acyclic correspondence
proposition.

\section{Computing \GK{} reports}
\label{sec:implementation}

The reference models define the target quantities.  \GK{} constructs a
retained-proof result from answer-proof histories and, where required, attempts
a dependency-aware evaluation from the clausified knowledge base.
Figure~\ref{fig:confidencepipeline} shows the two paths and the rule for
selecting the report.

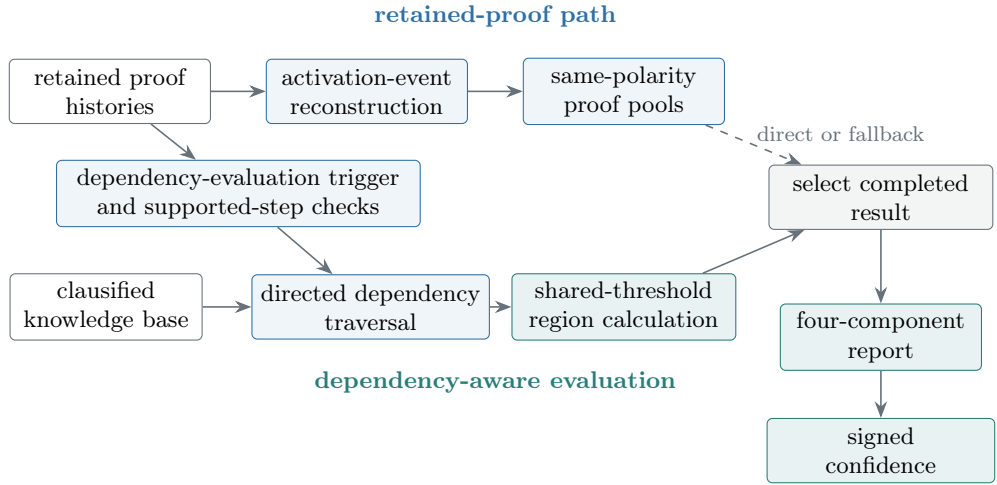
\begin{figure}[htbp]
\centering
\begin{tikzpicture}[
  pipeline box/.style={gkbox, font=\footnotesize, text width=2.45cm,
    inner sep=3pt, minimum height=.85cm},
  pipeline process/.style={pipeline box, fill=gkblue!7, draw=gkblue},
  pipeline result/.style={pipeline box, fill=gkteal!10, draw=gkteal},
  pipeline fallback/.style={pipeline box, fill=gkgray!8, draw=gkgray,
    text width=2.75cm},
  pipeline label/.style={font=\scriptsize, inner sep=1pt, text=gkgray}
]
  \node[pipeline box] (found) at (-5.1,0)
    {retained proof\\histories};
  \node[pipeline process] (replay) at (-1.7,0)
    {activation-event\\reconstruction};
  \node[pipeline process] (pool) at (1.7,0)
    {same-polarity\\proof pools};

  \node[pipeline process, text width=4.6cm] (trigger) at (-3.4,-1.35)
    {dependency-evaluation trigger\\and supported-step checks};
  \node[pipeline box, text width=2.4cm, inner xsep=2pt] (clauses) at (-5.15,-2.85)
    {clausified\\knowledge base};
  \node[pipeline process, text width=3.0cm, inner xsep=2pt] (reeval) at (-1.65,-2.85)
    {directed dependency\\traversal};
  \node[pipeline result, text width=2.75cm] (threshold) at (1.7,-2.85)
    {shared-threshold\\region calculation};

  \node[pipeline fallback] (select) at (5.1,-1.4)
    {select completed\\result};
  \node[pipeline result] (four) at (5.1,-3.25)
    {four-component\\report};
  \node[pipeline result, text width=2.8cm] (verdict) at (5.1,-4.75)
    {signed\\confidence};

  \draw[gkarrow] (found) -- (replay);
  \draw[gkarrow] (replay) -- (pool);
  \draw[gkarrow] (found) -- (trigger);
  \draw[gkarrow] (trigger) -- (reeval);
  \draw[gkarrow] (clauses) -- (reeval);
  \draw[gkarrow] (reeval) -- (threshold);
  \draw[gkarrow] (threshold) -- (select);
  \draw[gkarrow, dashed] (pool) -- node[pipeline label, above right]
    {direct or fallback} (select);
  \draw[gkarrow] (select) -- (four);
  \draw[gkarrow] (four) -- (verdict);

  \node[font=\footnotesize\bfseries, text=gkblue] at (0,1.0)
    {retained-proof path};
  \node[font=\footnotesize\bfseries, text=gkteal] at (0,-3.85)
    {dependency-aware evaluation};
\end{tikzpicture}
\caption{Post-search report construction.  Retained proof histories produce
the retained-proof result.  When a contested atom is detected, \GK{} attempts
a separate dependency-aware traversal of the clausified knowledge base.  A
completed traversal replaces the retained-proof result; otherwise the
retained-proof value is returned directly or as a flagged fallback.  Signed
confidence and answer placement are derived from the selected four
components.}
\label{fig:confidencepipeline}
\end{figure}

\subsection{Ground-instance event identity}
\label{sec:proofevents}

Section~\ref{sec:groundactivation} defines the activation model.  This subsection defines the event
identifiers used by \GK{}'s retained-proof calculation.  For every derived clause,
the prover records the inference step and pointers to its parents.  The
resulting DAG of an answer clause is its \emph{proof history}.  Replaying that
history recovers the ground instance of every uncertain input clause used by
the proof (Section~\ref{sec:replay}); the retained proofs of one polarity can
then be combined without treating a shared instance as independent.

When every reconstructed activation-event identifier is ground and proof replay and
inclusion--exclusion remain within their bounds, the calculation is the exact
union probability of the retained proof set.  It equals
the full activation-model value when that set covers all minimal
explanations and is otherwise a lower bound.  Non-ground identifiers or exceeded
calculation bounds instead produce the approximations described below.

\label{sec:eventidentity}

Recall from Section~\ref{sec:setting} that the \emph{activation event}
$e_{C,\theta}$ of an uncertain input clause $C$ and a ground substitution
$\theta$ for its variables holds, with probability $p_e$ equal to the
input confidence of $C$, exactly when the ground instance $C\theta$ is active,
independently of every other activation event.  Here $C$ denotes an identified
input-clause occurrence, not merely a clause formula.  Two activation events
are identical exactly when they have the same input-clause identifier and the
same resulting ground clause instance.  Grounding substitutions that produce
that same instance therefore name one event.  Syntactically identical statements entered
twice have distinct identifiers and therefore define distinct activation
events.  The pair consisting of the identified occurrence $C$ and its
resulting ground instance $C\theta$ is the \emph{activation-event identifier}.
The \emph{activation-event set} $S_D$ of a derivation $D$ is the set of
activation events of the clause instances $D$ uses.  Being a set, $S_D$ records
one event for an instance that $D$ uses twice and two events for one clause
used at two different substitutions.  The \emph{proof-availability event} of
$D$ is the joint event
\[
  A_D=\bigcap_{e\in S_D}e
\]
that every event in $S_D$ holds; the derivation is available exactly on
$A_D$, so
\begin{equation}
  P(A_D)=\prod_{e\in S_D} p_e.
  \label{eq:proofproduct}
\end{equation}

This convention matches the grounding of a probabilistic clause into separate
probabilistic facts in \ProbLog~\cite{deraedt2007,fierens2015}.  Uncertainty
shared across all ground instances of a rule can instead be represented by one
uncertain auxiliary premise.

Counting once per ground instance is a choice among three conventions, and they
differ observably.  Suppose a rule has input confidence $0.8$.  If a proof needs two
different ground instances of it, the per-instance product is $0.8^2=0.64$; if
it needs one instance whose conclusion is used in two branches that join again
later, that instance is a single event and the product is $0.8$.  Multiplying
once per use of the clause would give $0.64$ in both cases, and multiplying
once per input clause would give $0.8$ in both.
Proposition~\ref{prop:countingorder} states that the three conventions are
ordered.

\subsection{Combining retained proofs}
\label{sec:proofunion}

For retained proofs $D_1,\ldots,D_m$ of the same answer and polarity, the
retained-proof value is the probability that at least one proof-availability
event holds:
\begin{equation}
 P\left(\bigcup_{i=1}^{m}A_{D_i}\right)
 =\sum_{\emptyset\ne I\subseteq\{1,\ldots,m\}}
   (-1)^{|I|+1}
   \prod_{e\in\cup_{i\in I}S_{D_i}}p_e.
 \label{eq:ie}
\end{equation}
The union inside each term counts a shared event once, which is what makes the
formula sensitive to overlap.  Four consequences follow directly: two proofs
with the same event set contribute as one; a proof whose event set contains
another's adds nothing and can be dropped; proofs with disjoint event sets
combine by noisy-or; and partially overlapping proofs fall between the largest
single proof value and that noisy-or value.

\begin{corollary}[Proof-union invariance]
\label{cor:proofunioninvariance}
For a fixed retained proof set whose activation-event identifiers are ground and
within the inclusion--exclusion bound, Equation~\eqref{eq:ie} is invariant
under proof ordering, systematic variable renaming, and duplication of a proof
with the same activation-event set.
\end{corollary}

\begin{proof}
Proof ordering only permutes the indexed union.  Systematic variable renaming
does not change the ground clause instances named by the reconstructed events.
Finally, union is idempotent, so duplicating an identical activation-event set
does not change the event whose probability is computed.
\end{proof}

Figure~\ref{fig:proofpooling} shows two proofs with one shared activation
event.  The diagram carries information that a pair of already pooled scalar
values does not.

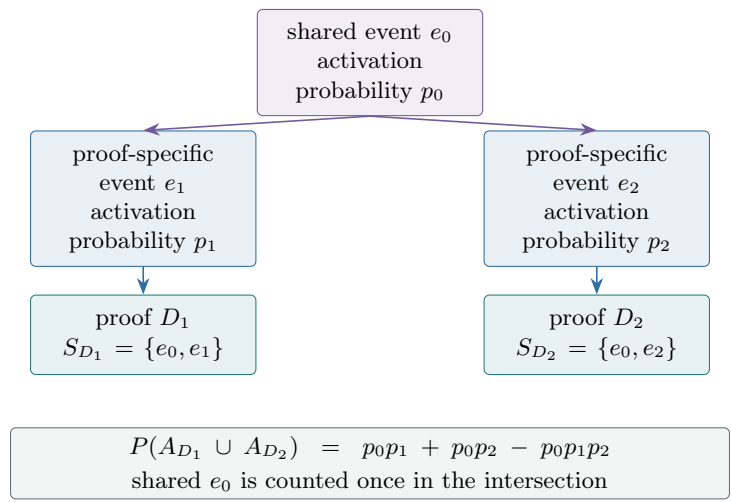
\begin{figure}[htbp]
\centering
\vspace{0.35\baselineskip}
\begin{tikzpicture}[
  event box/.style={gkbox, font=\footnotesize, text width=2.7cm,
    inner xsep=4pt, inner ysep=5pt, minimum height=1.05cm},
  private box/.style={event box, fill=gkblue!10, draw=gkblue},
  proof box/.style={event box, fill=gkteal!10, draw=gkteal}
]
  \node[event box, fill=gkpurple!10, draw=gkpurple] (shared)
    at (0,1.35) {shared event $e_0$\\activation probability $p_0$};
  \node[private box] (left) at (-3.0,-0.45)
    {proof-specific event $e_1$\\activation probability $p_1$};
  \node[private box] (right) at (3.0,-0.45)
    {proof-specific event $e_2$\\activation probability $p_2$};
  \node[proof box] (d1) at (-3.0,-2.25)
    {proof $D_1$\\$S_{D_1}=\{e_0,e_1\}$};
  \node[proof box] (d2) at (3.0,-2.25)
    {proof $D_2$\\$S_{D_2}=\{e_0,e_2\}$};

  \draw[gkarrow, draw=gkpurple] (shared.south) -- (left.north);
  \draw[gkarrow, draw=gkpurple] (shared.south) -- (right.north);
  \draw[gkarrow, draw=gkblue] (left) -- (d1);
  \draw[gkarrow, draw=gkblue] (right) -- (d2);

  \node[gkfaded, font=\footnotesize, text width=9.3cm,
    inner sep=3pt] (formula) at (0,-3.95)
    {$P(A_{D_1}\cup A_{D_2})=p_0p_1+p_0p_2-p_0p_1p_2$\\[2pt]
     shared $e_0$ is counted once in the intersection};
\end{tikzpicture}
\caption{Provenance-aware same-polarity pooling.  The two proofs are neither
independent nor duplicates: they share $e_0$ and have distinct remaining
events.  Retaining the event sets permits exact union calculation within the
inclusion--exclusion bound; maximum and noisy-or use only the two scalar proof
values.}
\label{fig:proofpooling}
\end{figure}
\FloatBarrier

The overlap rule in \cite{tammet2021} interpolates between maximum and
noisy-or using the fraction of shared input clauses and a global independence
parameter.  Equation~\eqref{eq:ie} removes that parameter when ground
activation-event identities can be reconstructed.  When all event identifiers are
ground and the inclusion--exclusion bound is respected, it is exact for the
retained proof set under the activation model.  The calculation does not
include derivations omitted by the bounded search.  It equals the full
activation-model query value when the retained proofs cover all relevant
minimal explanations.  This is the quantity a probabilistic logic programming
system obtains by compiling every explanation~\cite{fierens2015}.  Coverage
holds in the small
examples of this paper; Section~\ref{sec:differences} gives a recursive example
where the gap is instead the dominant effect.

For two proofs $D_1$ and $D_2$, let
$h=\prod_{e\in S_{D_1}\cap S_{D_2}}p_e$ be the probability that all their
shared events are active, with $h=1$ when they share no event.
Inclusion--exclusion reduces to
\begin{equation}
  P(A_{D_1}\cup A_{D_2})
  =P(A_{D_1})+P(A_{D_2})-\frac{P(A_{D_1})P(A_{D_2})}{h}.
  \label{eq:measuredoverlap}
\end{equation}
The value ranges from noisy-or at $h=1$ to the maximum when one
activation-event set contains the other.  Unlike the shared-item-count interpolation of
\cite{tammet2021},
$h$ is the joint activation probability of the common events.  Equation
\eqref{eq:measuredoverlap} must not be folded over three or more proofs because
the intermediate scalar does not retain event identity; Equation~\eqref{eq:ie} is the
set-level generalization.

\subsection{Reconstructing ground instances}
\label{sec:replay}

The activation events of a derivation are not recorded during search:
resolution works on non-ground clauses, and a derived clause records its
history, not the substitutions that produced it.  The activation events are
reconstructed afterwards.  The
implementation walks the history DAG of the answer, renames variables apart at
the input leaves, re-executes each recorded resolution or factoring step, and
checks the clause it obtains against the clause the search stored at that node.
The substitutions accumulated along the way name the ground instance used at
each uncertain leaf, and those instances are the activation-event identifiers.
Reconstructing them after proof search keeps ground-instance records off the
derived clauses used by the search.

Replay does not always yield ground substitutions.  A non-ground identifier
represents several possible ground instances.

Within one proof, uses of the same identified clause occurrence whose
identifiers remain mutually unifiable are merged into a single event.  This
treats every pair of uses that can denote the same ground instance as denoting
one instance.  Requiring one fixed member of the represented family is
sufficient but not necessary for the proof to be available, which gives a
lower bound for one proof.

Across several proofs, the current implementation treats unresolved
identifiers as proof-specific.  Their possible ground-instance families may
overlap in unknown ways, so combining them as independent proof-specific
events is neither exact nor guaranteed to be a lower bound on the proof
union.  Such a result is therefore a retained-proof approximation.

Exact union requires finite grounding of the relevant identifiers and
recovery of their actual overlap.  Without that grounding, the guaranteed
lower bound is the maximum of the individual proof lower bounds.

Replay currently covers resolution, factoring, and the supported unary inference-history
steps.  An unsupported equality or rewriting history, a verification failure,
or a cap returns the search-time product that counts every uncertain clause
use separately.  This fallback is
conservative in the following sense.

\begin{proposition}[Counting order]
\label{prop:countingorder}
For one proof with event activation probabilities in $[0,1]$, let $p_u$ multiply a factor for
every uncertain use, let $p_i$ multiply once per distinct ground instance, and
let $p_c$ multiply once per distinct input clause.  Then
\[
       p_u\le p_i\le p_c.
\]
\end{proposition}

\begin{proof}
Passing from uses to instances merges equal uses.  Passing from instances to
clauses merges all instances of one clause.  Every merge removes one or more
factors in $[0,1]$ without changing the remaining factors, so it cannot
decrease the product.
\end{proof}

\GK{} is written in C.  The proof search, parsing, indexing, shared-memory
database, resolution, and equality machinery are inherited from \textsc{gkc}.
The uncertainty code records proof histories and dependencies, reconstructs
activation events after search, combines same-polarity proofs, performs recursive
blocker searches, and runs the separate directed clause traversal used to
build the shared-threshold report.

\subsection{Directed search for opposing support}
\label{sec:premisesearch}

The shared-threshold calculation needs all relevant directed positive and
negative derivations at each contested ground atom, including their shared
dependency sets.  Retained answer proofs identify where this evaluation
is needed but do not in general contain the opposing derivations.  \GK{}
therefore reconstructs a bounded query-relevant dependency graph by a separate
report-time traversal of the clausified knowledge base.  Only a completed
traversal replaces the retained-proof result.  An open premise or resource
bound causes a flagged retained-proof fallback; partial dependency-aware
results are not combined with retained-proof results.

A ground atom is \emph{potentially two-sided} when clauses retrieved from the
conclusion index contain
eligible conclusions of both polarities that unify with it.  Premise opposition is not
inferred analytically from the retained answer
proof.  After the main prover has produced an answer, report construction first
uses that proof only to test whether its inference steps are in the supported
fragment and to collect the ground query atom and the ground body atoms used
by its retained inference steps.  If any collected atom is potentially
two-sided, the \emph{dependency-evaluation trigger} starts a second,
report-time search at the ground answer atom.  Otherwise the report uses the
answer-level proof pools directly.
This is a syntactic test.  If it does not fire, \GK{} uses the direct
retained-proof result; this does not prove that no opposing derivation lies
outside the paths reached by the test.  Section~\ref{sec:limitations} lists
known cases in which the test misses an opposing derivation.

For each ground target literal, the second search scans the clausified
knowledge-base clauses, excluding query goals.  It matches each eligible
conclusion to the target without instantiating the target, applies the
resulting substitution to the rule body, and recursively evaluates the required polarity of every body
literal.  Both positive and negative derivations are collected: the traversal
enumerates all matching directed derivations reached before an implementation
cap or deadline, rather than stopping after a first-proof
refutation search.  The subsequent numerical stage enumerates configurations
specifying which directed derivations are present, pools each polarity,
preserves dependencies shared by the positive and negative pools, and applies
Equation~\eqref{eq:masses}.  The retained
answer proof supplies neither these clause matches nor their substitutions.

\paragraph{Directionality.}
An \emph{eligible conclusion} is a clause literal that the orientation policy
permits as the conclusion of a \emph{directed application}.  Write
$\overline L$ for the complementary literal: $\overline A=\neg A$ and
$\overline{\neg A}=A$.  For a ground
clause $L_1\lor\cdots\lor L_n$ and selected conclusion $L_i$, that application
is
\[
 \overline L_1\land\cdots\land\overline L_{i-1}\land
 \overline L_{i+1}\land\cdots\land\overline L_n
 \Rightarrow L_i.
\]
A unit clause contributes directly to its sole literal.
The eligible conclusion is recovered from syntax.  A fact concludes itself.
For a clausified default, blocker metadata records the authored head, which is
used as the eligible conclusion.  For a non-default clause without
authored-head metadata, positive literals are eligible conclusions.  An
all-negative clause without blocker metadata uses the deterministic guard,
source-head, and stored-order fallback specified in
Appendix~\ref{sec:gksyntax}.
For rule-form input, authored-head metadata is used in addition to the
syntactic tests.  Thus the rule
\[
  \neg q(x)\leftarrow\neg p(x)
\]
is read in its authored direction, with $\neg q(x)$ as its conclusion, even
though clausification produces an all-negative clause.  The metadata is
additive: it preserves authored rule direction without forcing an arbitrary
orientation on genuine disjunctions.  Direction is less stable for
equivalences and other non-rule formulas that clausify into
several products: there may be no single authored head for every product, and
the stored-literal fallback is order-dependent.  Classically equivalent source formulas
can therefore yield different quantitative reports after different
clausifications or orientations.  Directed quantitative meaning is attached
to the authored or selected applications, not to classical equivalence alone.

Only the selected head direction is used.  In particular,
$A\Rightarrow B$ can support $B$ from $A$ but does not support $\neg A$ from
$\neg B$.  This is motivated by the quantitative reading of a conditional: a
rule confidence attached to the forward reading does not determine the rule
confidence of its contrapositive without further information such as base
rates.  \GK{} applies this directed reading to all rules, including
confidence-one rules.  For a confidence-one rule intended as a strict
implication, classical contraposition would be sound; \GK{} currently does
not distinguish that case from a defeasible directed rule
(Section~\ref{sec:limitations}).

\paragraph{Why premise opposition uses a separate traversal.}
The existing exception mechanism asks an author-designated question of the
form ``can this blocker be derived?''  It therefore launches nested resolution
searches, including searches that test whether a blocker proof is itself
defeated, and gives deeper checks progressively smaller portions of the time
budget.  Proof existence under the full prover is the quantity that a blocker
check requests, so this strategy matches its operational semantics even though
its result can depend on search strategy and available time.

Premise opposition requires all directed derivations contributing to
positive and negative support at each reachable atom, together with their
ground dependencies and shared predecessors.  A single general
first-order refutation would return existence and one proof, not the two pools needed by
the arithmetic.  Repeating it for every retained body atom would require
retaining and replaying every found proof before same-polarity pooling and
conditioning on dependencies shared by the positive and negative support
pools; otherwise shared evidence would again be
double-counted.  It would also introduce contrapositives and hence compute a
classical relation different from the directed relation above, unless both
sides of every atom were redesigned accordingly.

The current evaluator therefore performs one recursive clause traversal rooted
at the answer, reaching relevant premises as body targets rather than spawning
a full prover run for each one.  Its fixed depth, width, and joint caps make its
coverage less dependent on proof-search strategy.  A report deadline still
limits the traversal and is flagged when reached.  Replacing this traversal
by diminishing-budget prover calls would thus change both the computed
relation and the information available for pooling, not merely the algorithm
used to obtain the same value.

\subsection{Reference calculations and result selection}
\label{sec:reporting}
\GK{} does not sample either reference model.  For Model~1 it computes the
probability of the union of the retained proof-availability events by
inclusion--exclusion over the ground instances those proofs use
(Section~\ref{sec:proofevents}); it does not enumerate activation worlds or
all derivations available in the activated clause set.  When all
activation-event identifiers are ground and proof replay succeeds, its value
equals the full model value if the retained proofs cover all minimal
explanations and is otherwise a lower bound.  When activation-event
identifiers remain non-ground, the across-proof treatment is an
approximation; the maximum of the individual proof lower bounds remains a
guaranteed lower bound (Section~\ref{sec:proofunion}).

For Model~2, \GK{} first performs a bounded backward search for directed
derivations and then computes the resulting configuration and
threshold-region measures by finite summation
(Sections~\ref{sec:sharedthreshold} and~\ref{sec:premisesearch}).

Neither calculation grounds the knowledge base: the first sees only the
finitely many ground instances its proofs actually used, and the second
instantiates top-down from the ground query atom by matching, so it never
enumerates substitutions for a variable that the match left unbound.  This
keeps function symbols and large constant domains usable.  It is also why an
unbound body variable makes the second evaluation return no dependency-aware
result: the evaluator does not expand the unbound variable
(Section~\ref{sec:premisesearch}).

Each calculation is exact for the retained proofs or directed derivations
actually collected only when its replay, union, traversal, and enumeration
bounds are not exceeded.  Otherwise the first may use a deterministic
approximation and the second returns the retained-proof fallback.

The two reference semantics are kept apart because they draw different
random variables: Model~1 draws clause activations, Model~2 draws atom
thresholds, and no experiment draws both.  Each reference-model quantity is a probability in
its own construction.  A direct retained-proof result, fallback, or
approximation is a status-qualified calculation over the proof information
actually retained and need not equal the full reference-model value.  The paper
makes no joint statement over the two.  On a
one-sided derivation that uses each uncertain instance once, both semantics
return the same query value,
which is why the one-sided cases of Section~\ref{sec:example} do not separate
them.

Exception checking and report-time dependency evaluation are separate
computations.  Exception checking applies the limit of Section~\ref{sec:exceptionchecking} to
individual returned proofs and recursively checks each proof whose score
reaches the limit, using diminishing time budgets.  Report-time dependency
evaluation instead computes the region in which the exception atom is
positively usable, conditional on the rule body.  It
then calculates the effect on the affected application.  If this calculation
cannot be completed, \GK{} returns a flagged retained-proof fallback.  The
retained-proof value may itself be classified as an approximation when replay,
grounding, or proof-union conditions are not satisfied.
Compatibility mode uses the operational exception check when accepting or
rejecting candidate answers.

The user-visible decisions are distinct.  Proof search first produces answer
candidates.  In compatibility mode every retained proof is checked
recursively; one proof that is not disabled by a nested exception is
sufficient to accept the candidate.  A candidate whose proofs are all
disabled is rejected.  In dependency-aware reporting, $C(L)>0$ places $L$
in the positive-answer list, $C(L)<0$ places the opposite-polarity literal in
the negative-answer list, and $C(L)=0$ gives no directional answer.  A fallback or approximation retains this placement
convention, but the reported number carries a status qualifier and does not
represent a completed shared-threshold decision.

\subsection{Report status and limits}
\label{sec:reportstatus}

\paragraph{Fields emitted by \GK{}.}
The optional detail object contains three implementation fields:
\begin{itemize}\sloppy
\item \texttt{calculation} identifies the selected calculation:
\texttt{canonical\_atom}, \texttt{blocked\_flat}, \texttt{proof\_fallback},
or \texttt{flat};
\item \texttt{coverage\_status} records the implementation's coverage
assessment: \texttt{complete}, \texttt{incomplete}, or
\texttt{unsupported\_fragment};
\item \texttt{polarity\_status} records whether the result was calculated
once for a canonical positive atom, so that querying its explicit negation
exchanges the positive and negative components while leaving conflict and
ignorance unchanged: \texttt{guaranteed} or \texttt{not\_guaranteed}.
\end{itemize}

\begin{center}
\scriptsize
\begin{tabular}{L{2.8cm}L{5.4cm}L{5.8cm}}
\toprule
\texttt{calculation} value & Computation & Paper-level interpretation \\
\midrule
\texttt{canonical\_atom}
& atom-centred dependency-aware evaluation
& completed shared-threshold result only when the fragment conditions also hold \\
\texttt{blocked\_flat}
& blocked/unblocked calculation for one share-free blocked proof
& query-directed local support decomposition using the Model~2 default
operation; outside the atom-level correspondence fragment \\
\texttt{proof\_fallback}
& opposition resolution over retained proof pools after dependency-aware evaluation failed
& flagged retained-proof fallback \\
\texttt{flat}
& direct retained-proof calculation
& proof-pool decomposition rather than an atom-level partition \\
\bottomrule
\end{tabular}
\end{center}

\texttt{coverage\_status=complete} means that the selected calculation
completed without a detected operational failure.  It does not by
itself establish that the query satisfies the correspondence fragment.
\texttt{polarity\_status=guaranteed} states only that the result has this
canonical polarity orientation.

\paragraph{Classifications used in this paper.}
The paper additionally uses three paper-level classifications.  They are not
emitted fields.  The \emph{calculation method} records which calculation
produced the tuple: shared-threshold result, local blocked/unblocked result,
direct retained-proof result, retained-proof fallback, or retained-proof
approximation.  The shared-threshold method used the
dependency-aware evaluator.  The direct retained-proof method is used when the
syntactic test finds no contested ground atom.  The
fallback method was selected because dependency traversal, shared-predecessor
construction, configuration enumeration, or a deadline prevented dependency-aware
evaluation from completing.  The approximation method returned a retained-proof value after
a replay failure, non-ground activation-event identifier, or proof-union cap.
The local blocked/unblocked method is selected when the retained positive
answer has a stored exception strength, has one share-free blocked proof, and
the local calculation accepts its ground exception pattern.  Atom-centred
evaluation is used when the exception shares a predecessor with another
contribution or requires derived counterevidence.  If that evaluation is
incomplete, the local method is selected only when its calculation completes.

The separate \emph{fragment status} states whether the exactness conditions
for the selected calculation hold.  The meaning of \texttt{in} and
\texttt{out} depends on the method code.  A factoring-derived answer can have an exact direct
retained-proof tuple while remaining outside the shared-threshold comparison
fragment.  A nonmonotonic cycle can return a flagged retained-proof fallback
while lying outside the fragment for completed shared-threshold evaluation.

\paragraph{Meaning of the four reported numbers.}
The third classification is the interpretation of the four displayed
numbers.  A completed shared-threshold result is an atom-level partition: its
four fields are the probabilities of the positive-only, negative-only,
conflict, and ignorance regions of one ground atom.  A direct retained-proof
result or flagged retained-proof fallback is a four-component decomposition
of the two proof-pool marginals as stated in Section~\ref{sec:twomodels};
its four fields sum to one but do not form one polarity-invariant atom-level
partition.  An outside-fragment numerical result is reported without an
atom-level partition interpretation.  This classification is distinct from
\texttt{polarity\_status}: a tuple can be presented in a canonical polarity
orientation without thereby becoming a completed semantic partition.
A \texttt{BL/out} tuple accounts for the blocked and unblocked local cases,
but remains a query-directed support
decomposition rather than an atom-level partition.

For \texttt{RP} and \texttt{FB}, the retained-proof fragment means that
proof replay succeeds, every activation-event identifier is ground, and the
union calculation remains within its bounds.  The resulting value is exact
for the fixed retained proof set.  Equality with the full Model~1 value
additionally requires that this set cover every relevant minimal explanation.
\texttt{AP} marks failure of the retained-set exactness conditions.  The
\texttt{in}/\texttt{out} suffixes do not name one common fragment; their
meaning is given for each code below.  In particular, \texttt{RP/out} can be
exact for its retained proof set while a known exclusion, such as factoring,
prevents the comparison claimed for \texttt{DA/in}.

Cause-specific flags continue to report detected conditions such as open
premises, replay failure, traversal or enumeration limits, and deadlines.  A
consolidated \texttt{status\_causes} list is not yet emitted.  Result tables
use these paper-level method and status codes:

\begin{center}
\scriptsize
\begin{tabular}{L{1.5cm}L{6.2cm}L{5.2cm}}
\toprule
Code & Calculation & Fragment status \\
\midrule
\texttt{DA/in} & dependency-aware shared-threshold result & within the correspondence fragment \\
\texttt{RP/in} & direct retained-proof result & exact for the retained proof set \\
\texttt{FB/in} & flagged retained-proof fallback & exact for the retained proof set; detected dependency-aware failure \\
\texttt{BL/out} & local blocked/unblocked result & complete local calculation; outside atom-level correspondence \\
\texttt{AP/out} & retained-proof approximation & replay failure, non-ground identifier, or proof-union limit \\
\texttt{RP/out} & direct retained-proof result & outside the ST fragment: factoring or another known exclusion \\
\texttt{FB/out} & flagged retained-proof fallback & ST calculation unsupported or incomplete \\
\bottomrule
\end{tabular}
\end{center}

These paper-level classifications are not all emitted by the implementation
evaluated here;
Example~12, for example, is \texttt{RP/out} in the paper because its direct
retained-proof tuple depends on factoring, although the evaluator does not
detect that exclusion.

\paragraph{Implementation limits.}
The numeric caps below are implementation safeguards, not values derived from
the semantics, and the paper does not claim that they are optimal.  Replay
limits cause a retained-proof fallback with a replay cause.  Directed
traversal, derivation, body, and joint-enumeration limits cause a flagged
retained-proof fallback when detected.  Partial dependency-aware results are
not combined with retained-proof results.  Report deadlines are user-adjustable and
flagged.  The proof-union cap currently produces only a warning, and
dependency-set overflow can remain silent; both are listed as limitations
below.

The exact proof union uses bit masks for at most 64 distinct events and full
inclusion--exclusion for at most 20 reduced activation-event sets.  Above the proof
limit it uses a deterministic sequential approximation and emits a warning.  Proof replay
is bounded at 3,000 nodes, depth 200, 64 uncertain leaf uses, and 32 literals
per replay clause.  The directed-evaluation bounds are:

\begin{center}
\scriptsize
\begin{tabular}{L{3.0cm}L{5.5cm}L{1.8cm}L{3.4cm}}
\toprule
Stage & Bounded object & Limit & Result when exceeded \\
\midrule
derivation collection & matching directed derivations & 24 per polarity
& flagged retained-proof fallback \\
presence enumeration & collected derivations & 16 total
& flagged retained-proof fallback \\
body construction & literals in one directed rule body & 8
& flagged retained-proof fallback \\
dependency tracking & predicate-symbol/arity pairs, each associated with one representative ground atom & 32
& may select a retained-proof result without a cause flag \\
recursive traversal & nested target atoms & depth 16
& flagged retained-proof fallback \\
joint calculation & shared predecessor ground atoms & 6, hence at most $4^6$ assignments
& flagged retained-proof fallback \\
report timing & one answer / whole report & 400 / 1,600 ms
& flagged retained-proof fallback \\
\bottomrule
\end{tabular}
\end{center}

The implementation uses a coarse dependency index to detect shared
dependencies and select ground atoms for joint enumeration.  Each
predicate-symbol and arity key stores one representative ground atom, so
distinct ground dependencies with the same key can be merged.  Such a query is outside
the correspondence fragment and may not receive a cause flag.  These are
implementation bounds rather than semantic parameters.

Proof replay is linear in the retained proof-history DAG within its bounds, but the
premise-opposition traversal is not linear in that DAG and does not currently
memoize intermediate results.
It scans the input clause list for each recursive target, so with branching
$b$ it can grow on the order of $b^d$ up to depth $d=16$ before its caps and
deadline intervene.  Two further explicit exponential operations are bounded:
$2^m$ inclusion--exclusion terms for $m\le20$ reduced activation-event sets and at
most $4^k$ joint assignments for $k\le6$ shared predecessor atoms.  Activation
events are reconstructed after proof search, so clauses carried by the main resolution
search do not grow with uncertainty metadata.  The mechanism therefore avoids
global grounding and changes neither the prover's search clauses nor its
indexes, at the cost of a separate bounded report-time traversal.

If dependency-aware evaluation encounters an unsupported derivation form, an
open premise, a detected implementation limit, or a report deadline, \GK{}
returns the retained-proof components and marks \texttt{PROOF\_FALLBACK} with
cause flags.  A completed shared-threshold evaluation can legitimately return
zero; a zero caused by cancellation reports the contested atom responsible.
The distinction
allows a caller to decide whether a flagged estimate is acceptable.  It does
not imply that every limitation is signaled: non-ground answer
instances, a missed dependency-evaluation test, and dependency-set overflow
can still select the retained-proof calculation
silently, as discussed in Section~\ref{sec:limitations}.

The current treatment of open premises is conservative.  During dependency-aware evaluation,
a premise that remains non-ground is assigned no usable state and normally
causes a retained-proof fallback.
The diagnostic \texttt{-dwopen} mode tries substitutions drawn from the known
constants, subject to a fixed bound.  Cycles behave as defined in Section~\ref{sec:cycles}:
cycles through exception conditions at the queried atom use the query-relative
credulous policy, self-blocking
defaults contribute nothing, and recursion through a contested atom falls back with
flags; the independent reference implementation of Section~\ref{sec:method}
reports all nonmonotonic cycles as outside its fragment.

\section{Evaluation design}
\label{sec:method}

\subsection{Evaluation questions}

The evaluation separates analytic checks, stochastic checks, system
comparisons, and proof-search capability.  It addresses four questions:
\begin{enumerate}
\item Does the retained-proof calculation reproduce analytic proof-union
probabilities when the retained proofs cover all relevant minimal
explanations?
\item Does the dependency-aware calculation reproduce the shared-threshold
semantics for opposition and uncertain exceptions?
\item Do the implementation's reports and status fields identify completed
calculations, fallbacks, and unsupported cases?
\item On exact common fragments, do \GK{} and the compared systems return
corresponding outputs?
\end{enumerate}
Two secondary expectations accompany these questions: the per-use,
per-clause, and shared-item-count approximations should show larger errors on
the discriminating examples, and undercutting one rule application should
outperform discounting the pooled conclusion under the stated synthetic
decision model.  Agreement is required only for completed E-labeled
quantities; N-labeled nonmonotonic cases may produce different native outputs
because the systems' possible-world or stable-model constructions differ, and
A-, P-, and U-labeled cases identify translation differences, output
differences, or unsupported encodings.

\subsection{Reference programs and samplers}
\label{sec:samplers}

There are two sampler families: the clause-activation sampler (CA) and the
shared-threshold sampler (ST).  The public repository contains both samplers,
the knowledge bases used below, and the data package for the
implemented-system comparison of
Section~\ref{sec:systemcomparison}.\footnote{Relative to the repository
root: both public sampler modes are invoked through
\path{montecarlo/gkmc.py}; shared-threshold evaluation is implemented in
\path{montecarlo/threshold_worlds.py}; the focused priority-zero check is
\path{montecarlo/test_threshold_rank0.py}; knowledge bases are under
\path{Examples/}; the comparison package is \path{comparisons/}.}  Other
numerical experiments are specified in the paper but are not claimed to be
part of that repository.

\subsubsection{Clause-activation sampler}
The clause-activation sampler implements the
grounding and activation steps of Model~1 directly.  It grounds the clausified
input over the constants occurring in it and in the query, rejecting inputs
that contain function symbols and refusing to truncate a ground set that
exceeds its cap rather than sampling a biased part of it.  Per world it keeps
each uncertain ground instance with probability equal to its clause's input
confidence.  To determine the outcome of one sampled world, the sampler runs bounded
unweighted \GK{} twice on its active ground clauses: once for the ground
query $Q$ and once for its explicit negation $\neg Q$.  These runs use the
default-aware, query-relevant derivability relation defined in
Section~\ref{sec:groundactivation}.  For completed worlds the sampler counts
query-only, negation-only, both, and neither outcomes and reports the two
derivability marginals and their signed difference.  These four cells are
activation-world derivability cells, not the conflict and ignorance regions
of shared-threshold semantics.

The reported CA runs used \texttt{-nonegative -plain} for each unweighted
world check, the default exception limit $\lambda=.5$, a 10-second limit per
query, and no additional proof-search or candidate-retention options.  Every
returned exception proof in an unweighted world has score one and therefore
passes this exception limit.  The bounded search can still omit a proof or
reach the time limit.

These checks use the ordinary bounded \GK{} prover.  The sampler therefore
does not independently decide derivability.  Each prover subprocess has a
configurable time limit, 10 seconds by default; a timeout or an omitted
proof leaves the sampled world unresolved.
Let $W$ be the requested number of sampled worlds, let $d_+$ and $d_-$ be
the numbers of worlds in which the completed checks derive $Q$ and $\neg Q$,
and let $t_+$ and $t_-$ be the corresponding timeout counts.  A semantically
interpretable run must retain these counts per row, together with the number
of worlds in which both checks completed and the denominator of every
reported value.  The conservative marginal bound is
\[
  \frac{d_+}{W}\ \leq\ P_{\mathrm{act}}(Q)\ \leq\ \frac{d_++t_+}{W}.
\]
The difference between the two derivability probabilities has the
conservative interval
\[
 P_{\mathrm{act}}(Q)-P_{\mathrm{act}}(\neg Q)\in
 \left[
 \frac{d_+-d_--t_-}{W},\
 \frac{d_++t_+-d_-}{W}
 \right].
\]
Thus a CA point estimate is evidence about Model~1 only when $t_+=t_-=0$
(or when a separate complete finite checker decides the unresolved worlds);
otherwise the interval, not a completed-world ratio, is the semantic result.
The stored aggregate underlying the 3,000-world values does
not preserve these counts.  Where shown, those CA values are therefore reported only
as bounded-prover operational outputs, not as estimates of mathematical
$\vdash^\Delta_q$-derivability.  Within its bounds the sampler can nevertheless reach
derivations that \GK{}'s report-time traversal is not designed to reach;
Section~\ref{sec:differences} uses exactly that operational gap.
The CA sampler reuses \GK{}'s prover for derivability checks.  It
independently tests activation sampling, but not proof search, cycle
handling, or recursive exception checks.

\subsubsection{Shared-threshold sampler}
The shared-threshold samplers perform the three steps of Model~2 and invoke no
prover.  Model~2 determines, per world, which region the
query atom's threshold falls in, not whether anything is provable, and once the
program is grounded and each clause read in one direction, that question is a
function of the drawn thresholds: the choice of clauses, directions, and
instantiations is fixed before sampling begins, so the outcome depends on the
drawn thresholds alone and no search is required.  The program compiles the input into ground directed rule
instances and draws the shared threshold $U_A$ for every atom.  It also draws an independent auxiliary threshold
$V_A$ for atoms participating in contrary-gated interactions.  Ordinary and
strict-priority cases use only $U_A$.  It evaluates
the atoms in dependency
order.  A finite cycle containing only positive same-polarity dependencies is
handled by a least fixpoint.  For one fixed threshold world, the sampler
repeatedly adds every fact whose support passes its threshold and every rule
conclusion whose body is already satisfied and whose pooled support passes
the applicable threshold test.  On such a component this monotone iteration
from the empty set reaches the least fixpoint; fair update order gives the
same result.  The least-fixpoint construction is not applied to opposition or
exception cycles.  Those cycles use only the named policies of
Appendix~\ref{sec:cycles}, or are reported outside the sampler's fragment.
The sampler then applies the
opposition-resolution and default-interaction arithmetic, including
strict-priority override and its restriction, the rank restriction on direct
derivations, the credulous resolution of cycles through exception conditions at the queried
atom, and per-instance evaluation of open queries, before counting the four
query-atom outcomes.  It also implements the reciprocal priority-zero rule by
distinguishing cycle-internal exception references from external exception
evidence.  Contextual priority-zero cases whose inherited priority restriction
excludes a candidate contribution are currently reported outside the sampler's
fragment rather than evaluated numerically.  On its stated fragment, the
sampler directly samples the defined shared-threshold semantics; its only
numerical approximation is Monte Carlo sampling error.  Agreement with \GK{}
is expected up to sampling precision, and a persistent disagreement
indicates an implementation or semantic discrepancy.

\subsubsection{Independent shared-threshold checker}

The public shared-threshold module independently implements the atom-level
model and does not launch \GK{}.  It reads function-free JSON-LD-Logic input
written as directional clauses or as a restricted implication with atomic
antecedents and one consequent.  The last ordinary clause literal, or the
implication consequent, is the authored conclusion.  The module normalizes
this implication form, grounds the clauses over the named constants, samples
atom thresholds, and counts the four outcomes.  It does not perform general
clausification or process other formula connectives.  Monotone same-polarity
cycles use a least fixpoint.  Reciprocal priority-zero opposition and blocker
cycles through the queried atom use their specified policies; other cycles
through contested atoms or exception conditions are not scored.  Two fixed
seeds check reproducibility; they do not establish
statistical significance.  This is the independent reference implementation
used for the general shared-threshold comparisons.

\subsection{Test suites}

\paragraph{Analytic and default cases.}
The internal report regression suite contains small knowledge bases with expected
report fields derived by hand.  The default suite with closed queries and
zero-contribution branches has 45 non-open cases.  Compatibility mode is
checked separately as a repository
regression test.  Three instance-event comparison cases separate per-use,
per-instance, and per-clause counting.

The independent sampler regression passes all eight priority-zero
cases.  They include both query orientations and cases that distinguish
reciprocal cycle members from external exception evidence.  Separately, the
\GK{} polarity-pair runner verifies the two D4-0 query orientations reported
in Section~\ref{sec:defaultcalcs}.  The mixed case within the sampler's
implemented priority-zero fragment has the analytic and independently sampled
partition $.16/.56/.24/.04$.

The shared-threshold suite contains examples Y1--Y18, two variants, and two
trace checks.  Together they compare 54 scalar frequencies with their analytic
values at 400,000 worlds per case; each example contributes between one and
four checked scalars.

Separate cases test the blocked/unblocked calculation, numerical priority, and
the paired-exception case with an additional unpaired argument.  Two decision
experiments compare the policies using binary log-score, interpreted as Kelly
growth under the stated betting setup, and cost--loss.

A final two-world case measures unopposed support and signed confidence under
shared-threshold and independent couplings and compares the opposition rules
quoted in Section~\ref{sec:sharedthreshold}.  These cases do not invoke \GK{}.

The following table lists the calculation and limitation tested by each
reported group; raw totals alone are not used as coverage evidence.
\begin{center}
\scriptsize
\begin{tabular}{L{3.2cm}L{5.2cm}L{5.1cm}}
\toprule
Tests & Defined calculation or branch & Limitation or alternative tested \\
\midrule
IE1--IE3
& ground-instance identity and proof union,
Equations~\eqref{eq:ie} and~\eqref{eq:measuredoverlap}
& per-use and per-clause counterfactual baselines \\
Default examples in Table~\ref{tab:defaultexamples}
& evaluated exception, paired class, strict priority, and shared predecessor
& rejected scalar shortcuts shown in the last column \\
Priority-zero cases
& reciprocal cycle, opposite query, certain case, one-sided undercutting case,
mixed support pool, and contextual-priority inclusion or exclusion of a contribution
& internal reciprocal exception references versus external exception evidence;
contextual priority-zero case outside the sampler's fragment \\
Examples 1--11 and 17
& retained-proof union on supported one-sided programs
& incomplete retained-proof coverage is separated from semantic agreement \\
Examples 5--7, 18--19, and 22
& ordinary opposition and premise propagation
& activation-model difference and polarity reversal \\
Examples 12--16 and 20--23
& factoring, open premises, nonmonotonic cycles, and non-ground answers
& direct, flagged-fallback, and outside-fragment statuses \\
Limit regressions
& replay, traversal, joint-enumeration, and deadline exits
& cause-specific fallback flags; silent limits remain in
Section~\ref{sec:limitations} \\
\bottomrule
\end{tabular}
\end{center}

\paragraph{Proof-overlap experiments.}
The provenance-aware proof-overlap experiment constructs two proofs from
independently sampled evidence items, with a known number of shared items.  It
compares Equation~\eqref{eq:measuredoverlap} with the earlier interpolation
that uses only the number of shared evidence items (the
\emph{shared-item-count interpolation})~\cite{tammet2021}, maximum, and
noisy-or.  A betting
variant draws the overlap and item input confidences on every round.

\paragraph{Example suite.}
The examples cover same-polarity aggregation, rule chains, answer
substitutions, social-network propagation, contradictory evidence, Bayesian
network encodings, and penguin examples~\cite{tammet2021}.  We ran \GK{} and the
shared-threshold checker for 100,000 worlds with seeds 7 and 42.  The
stored clause-activation run used 3,000 worlds and seed 7, but its per-row
completion counts and effective denominators are unavailable.  Those values
are omitted from the suite table.  The complete \GK{} source inputs are
printed in Appendix~\ref{sec:completeinputs}.
For the 100,000-world shared-threshold runs, the binomial standard error of
each component is at most approximately $.0064$ when its probability is near
one half; we use this as a
conservative per-component Monte Carlo tolerance.  Section~\ref{sec:samplers}
states the completion data required for a new CA run.

\paragraph{Repository regression tests.}
\label{sec:examplesuite}

The public repository contains a set of examples ranging from ordinary facts
and rules to defaults, priorities, arithmetic, equality, planning, and logic
generated from English.  Not all of these examples are treated as semantic experiments: many exercise
input formats or first-order prover capabilities and contain no defaults or
input confidences.  On the uncertainty-bearing supported subset, the
threshold-world sampler agrees with \GK{} within Monte Carlo error, including
exception, priority, and mutual-blocking cases.  Cases requiring function
terms, arithmetic, equality, structured context terms, compact syntax handled
only after clausification, or taxonomy-valued priorities are outside that
sampler's fragment.  Their pass/fail status belongs in the repository's
regression records.  The semantic evidence retained in this paper consists of
the fully specified, discriminating examples of Sections~\ref{sec:example}
and~\ref{sec:results}.

\paragraph{Probabilistic and causal scope cases.}
A separate standard-library suite uses exact enumeration as its reference result for
small Bayesian networks and structural causal models, with two seeded
100,000-world Monte Carlo runs.  Checked-in ProbLog encodings and a pgmpy
variable-elimination verifier provide optional independent checks.  \GK{} is
run only on derived-output encodings of forward rules and on knowledge bases in which an
intervention has already been performed by replacing the affected structural
equation.  These cases test the limits of the input-confidence interpretation;
they are not used to claim native conditioning, intervention, or
counterfactual inference.

\subsection{System comparison protocol}
\label{sec:comparisonprotocol}

The comparison uses three independent classifications.  Translation status
describes the relationship between a source case and a target-system
encoding:
\begin{description}
\item[N --- native:] the system directly accepts the construct being tested;
\item[E --- exact fragment translation:] the translation preserves the stated
interpretation on the explicitly delimited fragment;
\item[A --- analogue:] a non-equivalent encoding reproduces a selected
behavior or marginal but does not preserve the tested construct or semantics;
\item[P --- derived-output encoding:] only an already computed output
distribution or externally transformed forward model is encoded; and
\item[U --- unsupported:] no faithful translation is claimed.
\end{description}
\GK{} calculation method and fragment status are recorded separately using
the notation of Section~\ref{sec:reportstatus}.  Where a comparison table
carries an interpretation column, it states in plain terms which outputs are
being compared and whether they arise from exact translations, analogues, or
different native semantics; it is not a mathematical relation between every
cell in the row.
Numerical agreement is interpreted as semantic agreement only for completed
E-labeled cases; A- and P-labeled rows compare selected behaviors or outputs,
and U means that no corresponding encoding was used.  An operational failure
is distinct: a suitable input exists, but parsing, grounding, search, or a
resource limit prevents a completed result.
For example, after fixing a finite grounding, an uncertain activation rule
$p::H(x)\leftarrow B(x)$ can be translated for a probabilistic-fact system by
introducing a fresh probabilistic activation atom for each ground instance
$\theta$ of the source rule occurrence $R$,
\[
  p::act_{R,\theta},\qquad
  H\theta\leftarrow B\theta\land act_{R,\theta}.
\]
This is class E only for the independent ground-instance activation
interpretation of Model~1; it is not a translation of shared-threshold
opposition or of \GK{}'s recursive exception checking.  Explicit classical negation is
mapped to strong negation where available.  A fresh atom such as
\texttt{neg\_h} is documented as a distinct atom, with any coherence
constraint stated explicitly.  Negation-as-failure guards are not described
as evidence for the explicit negation.  \GK{} priorities are class U unless a
separate equivalence argument is supplied.  Conditioning, learning,
interventions, and counterfactuals unsupported by \GK{} are likewise reported
as unsupported rather than simulated by adding independent facts.

The machine-readable comparison package identifies, for each case that was
run, the input, query, command,
translation status, output in the form returned directly by the system,
expected relation, and completion status.
The compared versions are \ProbLog{} 2.2.10, PASTA 1.0.1, plingo 1.1.0,
smProbLog 2.1.0.42, TweetyProject 1.31, clingo 5.6.2, DLV 2.1.1, I-DLV
1.1.6, and s(CASP) 1.1.4.  Running smProbLog required the
syntax and invocation adjustments
documented in the comparison package; the probabilistic facts, rules, queries, and
stable-model semantics were unchanged.  Standalone clingo 5.6.2, I-DLV 1.1.6,
and plingo 1.1.0 are
treated as distinct configurations.  We report semantic agreement only when
the translation is exact on the stated fragment and both computations
complete; bounded search, timeouts, parser rejection, and grounding
exhaustion are reported as resource limits, rejected syntax, or unsupported
language features.

\subsection{First-order workloads}
The birds family contains a basic default problem, a version with recursively
nested function terms, and versions with up to 100,000 irrelevant constants.
The comparison also uses Steamroller (PUZ031+1), Dreadbury (PUZ001+2),
Lukasiewicz (LCL047-1), and four large CSR problems from TPTP
\cite{sutcliffe2017}.  These cases test whether uncertainty reporting returns reports on inputs
handled by the first-order core.

\section{Evaluation results}
\label{sec:results}

\subsection{Event identity and proof overlap}

\subsubsection{Deterministic checks}

The shared-threshold suite passed all 45 non-open analytic cases.
Table~\ref{tab:instances} shows three complete examples printed in
Appendix~\ref{sec:completeinputs}.  The per-use and per-clause columns are
counterfactual counting conventions.
The compatibility-mode outputs match their stored reference files byte for
byte, and every case in the suite terminates with a well-formed report.
These checks test regression behavior; the semantic comparisons use the
analytic and sampling results below.

\begin{table}[ht]
\centering
\scriptsize
\caption{Three examples distinguishing per-use, per-ground-instance, and
per-clause activation identity.  CA is the
bounded-prover clause-activation output from a seed-7 sample of 3,000
worlds; completion counts are unavailable.  ST is the shared-threshold sampler
(100,000 worlds, seed 7).  ST does not support the conjunctive query form in
the first row.  RP/in is a direct retained-proof result within its stated
fragment; the last two columns are counterfactual counting baselines, not
system outputs.  The per-use baseline treats each occurrence of a rule in a
proof as an independent activation.  The per-clause baseline shares one
activation across all uses and ground instances of the source clause.}
\label{tab:instances}
\begin{tabular}{L{4.2cm}crrrrr}
\toprule
Case and identity & \shortstack{Method/\\status} & \GK{} & CA & \shortstack{ST result/\\status} &
\shortstack{per-use\\baseline} & \shortstack{per-clause\\baseline} \\
\midrule
IE1: two ground instances of one $.8$ rule & RP/in & .640 & .631 & unsupported query form & .640 & .800 \\
IE2: one $.8$ instance reused in two branches & RP/in & .800 & .790 & .799 & .640 & .800 \\
IE3: $.9$ and $.85$ facts joined by a $.7$ rule & RP/in & .3427 & .343 & .342 & .3427 & .4284 \\
\bottomrule
\end{tabular}
\end{table}

The suite also checks ordinary opposition, cancellation at a premise,
structured opposition, and multiple exception conditions directly against the closed
forms of Sections~\ref{sec:sharedthreshold} and~\ref{sec:defaults}.  The
following examples compare the defined calculations with simpler scalar
approximations.

These cases also exercise identification of the ground atom at which both
polarities receive support.  Premise-conflict
reports identify the contested atom rather than only the final query; default
reports retain the exception condition and both the supporting and defeating
proofs; and flagged incomplete cases report the corresponding cause.  These diagnostics
distinguish contradiction from missing evidence and a localized undercutting
attack from general negative support for the conclusion.

\subsubsection{Provenance-aware proof overlap}

In the experiment, two four-item proofs used input confidence $.85$ and
shared $k=0,\ldots,4$ evidence items.  Equation~\eqref{eq:measuredoverlap} matched the
simulated proof-union frequency for every $k$; the largest deviation was
$.0007$.  The shared-item-count interpolation missed by as much as $.0210$.
In an asymmetric case whose shared items had input confidences $.9$ and $.75$, the
simulation and provenance-aware overlap value were both $.6642$, while the
shared-item-count interpolation returned $.7400$.

The three-proof chain example shows why the full inclusion--exclusion
calculation is needed.  Direct
inclusion--exclusion returns $.8984$ (simulation $.8995$), while folding the
two-proof formula with the second pairwise overlap returns $.8318$ and
forgetting that overlap returns $.9169$.  In the betting variant with a fresh
visible overlap each round, each proof contains four evidence items and the
number shared is uniform on $\{0,\ldots,4\}$.  Item probabilities are drawn
uniformly from $[.6,.95]$, and item outcomes are independent Bernoulli draws.
The observed event is that every item in at least one proof succeeds.  The two
forecasts are Equation~\eqref{eq:measuredoverlap} and the shared-item-count interpolation
with parameter $1-k/(8-k)$.  For forecast $q$ and outcome $y\in\{0,1\}$, the
even-odds growth is
\[
  \log_2\!\left(\frac{q^y(1-q)^{1-y}}{.5}\right).
\]
Over 100,000 rounds, the provenance-aware forecast attained $.0480$ mean bits
per round and shared-item-count interpolation attained $.0466$.

The two experiments above are separate activation-event simulations that
record the identity of every proof item; they are independent of CA world
evaluation.\footnote{\path{pooling_correspondence/corr_gk_measured_overlap.py}
supplies the reported frequencies;
\path{pooling_correspondence/decide_kelly_gk_world.py} supplies the betting
result.}  The ST implementation records only threshold-world query outcomes,
so it has no separate result for the asymmetric and three-proof cases.

C3 supplies a further same-polarity check on two retained-proof cases with a
shared premise.  \GK{} returns $.846$ and $.959$, while independent
possible-world sampling gives $.8450$ and $.9596$.  The check compares event unions independently of the \GK{}
implementation.

\subsection{Opposition and defaults}

\subsubsection{Motivating examples}
The two reference models agree on the one-sided examples of
Section~\ref{sec:example} and differ on the contested-premise example.  For
the classical penguin input,
ground-instance activation semantics (Model~1) gives signed values $+1$ for
$tweety$ and $-1$ for $pingu$, while
shared-threshold sampling gives respectively $(1,0,0,0)$ and $(0,1,0,0)$,
equal to \GK{}.  For the two positive bird facts, both reference models have values
$.72$ for $flies(tweety)$ and $.54$ for $flies(robin)$.  For the uncertain
exception, Model~1 gives signed values $-.8$ for $tweety$ and $+1$ for
$robin$; shared-threshold sampling measured $(.102,.899,0,0)$ and
$(1,0,0,0)$,
against the \GK{} values $(.1,.9,0,0)$ and $(1,0,0,0)$.  Finally, in the balanced
contested-premise example Model~1 gives $.45$, because the negative premise
does not derive the negative query, whereas shared-threshold sampling and
\GK{} both give $(0,0,0,1)$.

The \texttt{gkmc.py} sample returned values numerically close to these
simple Model~1 calculations, including $-.804$ for the uncertain-exception
$tweety$ case.  Its completion counts were not retained, so those outputs
illustrate bounded-prover operational behavior rather than establish Model~1
probabilities.

\subsubsection{Default calculations}
\label{sec:defaultcalcs}

The \GK{} and 18-example specialized shared-threshold results agree on their
common default fragment.  A fixed regression suite covers the default family;
its specialized sampler reproduces every expected output within $0.01$ at
$10^5$ draws.

The clause-activation outputs are numerically consistent with the
corresponding simple Model~1 calculations, but their missing completion
counts prevent treating them as semantic estimates.  For example, the
uncertain-exception example of Section~\ref{sec:example} has analytic
Model~1 provability $.10$ and negation provability $.90$; the sampler
returned the same rounded values.  The query for the explicit negation has
analytic signed value $+.64$ and \GK{} tuple $(0.72,0.08,0,0.2)$; the sampler
returned $+.6409$.  The clause-activation sampler counts query-only,
negation-only, both, and neither worlds.  Those cells are activation-world
derivability counts, not shared-threshold conflict and ignorance regions;
equality of the shared-threshold four-component report is therefore checked
by the shared-threshold calculation.

The specialized shared-threshold sampler shares no code with the reasoner.
The clause-activation sampler does call the unweighted reasoner to decide
derivability, but contains none of \GK's uncertainty arithmetic.  Persistent
disagreement on the stated common fragments therefore indicates an
implementation or semantic discrepancy.

A completed dependency-aware evaluation of a contested premise illustrates
the calculation:
\[
 \begin{gathered}
 p(a),\qquad .6::w_1(a),\qquad .5::w_2(a),\\
 bird(x)\leftarrow p(x),\qquad
 \neg bird(x)\leftarrow w_1(x)\land w_2(x),\qquad
 canfly(x)\leftarrow bird(x).
 \end{gathered}
\]
For $Q=canfly(a)$, the negative derivation of $bird(a)$ has support
$.6\cdot.5=.3$, while the positive derivation is certain.  The report is
$.70/.00/.00/.30$.  The query atom has zero conflict.  The report separately
states conflict $.30$ at the predecessor $bird(a)$ and sets
\texttt{CONTESTED}.  This conflict prevents the contested part of the
support for $bird(a)$ from propagating.
The opposing route need not occur in the printed positive
proof because it is found by the report-time clause search.  Clause-activation
sampling instead gives signed value one, while shared-threshold sampling gives
$.70/0/0/.30$, matching \GK{}.

Three cases distinguish the default calculations.  First, the recursively
contested exception condition $.7::B,\ .3::\neg B$ with a certain body leaves
conditional probability $\beta_{B\mid A}=.4$ that the exception condition is
positively usable, so a $.9$
complete-blocking default contributes
$.54$.  Second, a
contrary-gated default with rule confidence $.6$ against ordinary support for the
opposite literal
$.9$ reports $.06/.90/0/.04$.  Third, two equally ranked contrary-gated defaults
with pooled support values $a$ and $b$ report
\[
  \bigl(a(1-b),\,b(1-a),\,0,\,(1-a)(1-b)+ab\bigr),
\]
including in ignorance the region $ab$ in which both independent threshold
tests pass.

All 54 scalar frequencies checked across the numbered examples, variants,
and trace cases were within Monte Carlo precision of their closed forms.  The
examples include cancellation at a
premise, a recursively evaluated exception condition, paired reference-class
exceptions with extra conditions, two
exceptions, two complete-blocking defaults, priority, shared exception-condition
support, and the case
where a paired reference-class exception also rebuts an independent derivation.
Table~\ref{tab:defaultexamples} summarizes one input for each distinct
calculation that is not already fully worked in Section~\ref{sec:example}.
The complete inputs are in Appendix~\ref{sec:defaultcaseinputs}.

\begin{table}[ht]
\centering
\scriptsize
\caption{Default cases that distinguish the calculations.  \GK{} and ST
values are ordered as positive support/negative support/conflict/ignorance;
CA reports $P(Q)-P(\neg Q)$.  The CA seed-7 sample used 3,000 worlds but
its completion counts are unavailable; ST used 400,000 worlds and seed 7.
ST is the shared-threshold sampler, and DA/in denotes dependency-aware
evaluation within its stated fragment.  The last column gives selected
support components only.}
\label{tab:defaultexamples}
\begin{tabular}{L{3.1cm}cL{1.9cm}L{1.1cm}L{2.3cm}L{4.0cm}}
\toprule
Case & \shortstack{Method/\\status} & \GK{} \sfour & CA & ST sample \sfour &
\shortstack[l]{Alternative calculation\\and selected components} \\
\midrule
E1: evaluated exception condition
& DA/in
& .54/0/0/.46
& .274
& .538/0/0/.462
& direct subtraction: positive $.5000$ \\
E2: paired reference class
& DA/in
& .48/.48/0/.04
& $-.115$
& .479/.481/0/.040
& ordinary opposition: positive/negative $0/0$ \\
E3: paired rule with extra condition
& DA/in
& .42/.24/0/.34
& .124
& .419/.240/0/.341
& assign unallocated blocked-branch probability to positive support: $.7200$ \\
E4: strict priority
& DA/in
& .60/.30/0/.10
& .332
& .599/.301/0/.100
& discard overlap: positive/negative $0/.3000$ \\
E5: shared exception predecessor
& DA/in
& 0/0/0/1
& 0
& 0/0/0/1
& body usable $\times$ exception not positively usable $\times$ rule
confidence: positive $.1728$ \\
\bottomrule
\end{tabular}
\end{table}

In E1, subtracting the exception marginals gives $.5000$.  In E2, treating
the paired branches as ordinary opposition gives $0/0$.  In E3, assigning the
unallocated \(1-.8\times.5=.6\) part of the blocked branch to positive support
gives $.7200$.  In E4, discarding both applications in their overlap gives
$0/.3000$.  In E5, the shortcut multiplies body usability $.48$, the
complement of exception-condition usability $1-.6$, and rule confidence $.9$,
giving $.1728$.

The next experiment tests whether rule confidence is preserved within the
unblocked cases.
With $r_{\mathrm{open}}=.9$ and recursively evaluated conditional probability
$\beta_{B\mid A}=.4$ that the exception condition is positively usable, the sampler in which blocked and unblocked
branches are disjoint gives
head support $.5410$ against the analytic $.5400$.  Its measured conditional
probability that the rule application contributes is
\[
\Pr(\delta\text{ contributes}\mid B\text{ is not positively usable})=.9006.
\]
Direct subtraction assumes maximal overlap between the rule-contribution
event and the event in which the exception condition is positively usable.  It gives $.4992$
against $.5000$, but its measured conditional probability that the rule
application contributes is $.8334$ rather
than the stated $.9$.  At $\beta_{B\mid A}=.9$, mixture gives $.0902$ and direct subtraction gives
zero.  A separate population model with actual exception classes reproduces
the mixture for $r_{\mathrm{blocked}}=0$ and $.3$.

\paragraph{Priority-zero reciprocal defaults.}
When both priorities are omitted in D4-0, they are parsed as zero.  The two
contrary-gated applications then form a reciprocal priority-zero pair.
Under the rule of Section~\ref{sec:defaults}, the two exception references
internal to that pair do not disable each other.  External exception evidence
continues to disable the application it reaches.  The active applications are
otherwise combined as ordinary unranked positive and negative support.

For premise input confidences $.8$ and $.6$, the Quaker-only, Republican-only,
both-premise, and neither-premise cases have probabilities $.32$, $.12$,
$.48$, and $.08$.  The resulting atom-level partition is therefore
\[
  R(pacifist(n))=(.32,.12,.48,.08),
\]
and polarity reversal gives
\[
  R(\neg pacifist(n))=(.12,.32,.48,.08).
\]
The corresponding signed confidences are $+.20$ and $-.20$.

Both runs match the analytic tuples above.  Their detail fields are
\begin{center}
\begin{tabular}{ll}
\texttt{calculation} & \texttt{canonical\_atom}\\
\texttt{coverage\_status} & \texttt{complete}\\
\texttt{polarity\_status} & \texttt{guaranteed}.
\end{tabular}
\end{center}
They are opposite-polarity presentations of one atom-level calculation.

Here \texttt{coverage\_status=complete} means that the implemented
priority-zero dependency-aware calculation completed without an operational
failure.  It does not place D4-0 inside the finite acyclic correspondence
fragment of Proposition~\ref{prop:dependencycorrespondence}.  D4-0 is
validated against the separately defined reciprocal priority-zero rule and
the public priority-zero sampler.

The public \path{montecarlo/test_threshold_rank0.py} sampler passes all eight
priority-zero cases, including both D4-0 query orientations.  It
implements the same priority-zero rule.  In the mixed case within the
sampler's implemented priority-zero fragment, it independently returns
$.16/.56/.24/.04$, matching the analytic reference value.  A contextual case
whose inherited priority restriction excludes one of the candidate
contributions is reported outside the sampler's fragment rather than assigned
a numerical value.  The mixed case adds an additional ordinary negative fact
with input confidence $.5$ to the reciprocal pair:
\[
\begin{gathered}
 .8::quaker(n),\qquad
 .6::republican(n),\qquad
 .5::\neg pacifist(n),\\
 pacifist(x)\leftarrow quaker(x)
   \ [\mathrm{unless}\ \neg pacifist(x);0],\\
 \neg pacifist(x)\leftarrow republican(x)
   \ [\mathrm{unless}\ pacifist(x);0].
\end{gathered}
\]
Its file and the other inputs and captured outputs of this paragraph are listed in
Appendix~\ref{sec:rank0artifacts}.

\subsubsection{Decision experiments}

Agreement with the named shared-threshold sampler checks the implemented
branches covered by the test suite.  The decision experiments address a
narrower question: whether two formulas match the synthetic model from which
outcomes are generated.  In the two-rule experiment, one rule contributes
$.7$ and is disabled with probability $p_B$, while an independent rule
contributes $.5$.  The true outcome probability is
$1-(1-.7(1-p_B))(1-.5)$.  Each policy reports a probability, and mean binary
log score is reported as bits per round.  Discounting only the affected rule
matches the generating probability and scores $.1188$; discounting the pooled
conclusion scores $-.1806$.

The cost--loss experiment uses a binary action.  Acting on a false conclusion
costs $c$, while failing to act on a true conclusion costs $1-c$.  For a
complete-blocking rule with confidence $r_{\mathrm{open}}$, acting when
$\beta_{B\mid A}<1-c/r_{\mathrm{open}}$ gives the same decisions as the
quantitative calculation at
$c=.20,.45,.70$.  Using the fixed cutoff $\beta_{B\mid A}<.5$ agrees only at
$c=r_{\mathrm{open}}/2=.45$.  These results verify the formulas under the stated synthetic
models.

\subsection{Coverage and fallbacks}

\begin{table}[!htb]
\centering
\small
\caption{Model difference and fallback difference in two examples from the
example suite.  DA/in is dependency-aware evaluation within its fragment;
FB/in is a retained-proof fallback within the retained-proof fragment.  The
last two columns contain signed scalar values.}
\label{tab:twodiffs}
\begin{tabular}{@{}L{2.35cm}L{3.15cm}cL{2.3cm}L{2.6cm}@{}}
\toprule
Example & Cause of difference & \shortstack{Method/\\status} & Comparison value & Shared-threshold value \\
\midrule
Example 18, $flies(a)$ & different reference semantics & DA/in & Model~1: $0.45$ & $0$ \\
Example 15, $smokes(4)$ & open-premise fallback & FB/in & retained-proof fallback: $0.44$ & $0.30$ \\
\bottomrule
\end{tabular}
\end{table}
\subsubsection{Model differences}
\label{sec:differences}

A model difference means that two calculations measure different quantities
and each value can be correct for its own semantics.  This subsection isolates
that case; the following subsections separate incomplete coverage from known
implementation limitations.

\paragraph{Contested premises: a model difference.}  The regression case has
$bird(tweety)$ with $0.5$ positive and $0.2$ negative support, and birds fly at
$0.9$.  \GK{} resolves the opposition at the premise before using it and
reports $(0.5-0.2)\cdot0.9=0.27$.  The unguarded ground-instance activation program does not
produce that reduction.  In a sampled world the clause ``tweety is not a
bird'' is an ordinary additional clause.  No rule of the program concludes
``tweety does not fly'', so this clause removes no proof of the query and
proves no negation, and the count stays at $0.5\cdot0.9\approx0.46$.  The
mechanism is the same as in the eight-world enumeration of
Section~\ref{sec:twomodels}.  Under the shared-threshold modeling choice
adopted here, support for $\neg bird(tweety)$ reduces every downstream use of
$bird(tweety)$, even when it does not derive the negation of the final
conclusion.  The threshold-world sampler, which realizes the
shared-threshold semantics used for \GK{}'s dependency-aware result, returns
\GK{}'s value.

\subsubsection{Coverage limits and fallbacks}

\paragraph{An open-premise fallback.}
Consider
\[
 p(a),\quad q(c),\quad bird(x)\leftarrow p(x),\quad
 \neg bird(a)\leftarrow q(y),\quad canfly(x)\leftarrow bird(x).
\]
Matching the ground head $\neg bird(a)$ leaves the body target $q(y)$ open.
The evaluator does not enumerate substitutions for the unbound variable,
although general resolution can use
$q(c)$.  It discards the incomplete dependency-aware result and reports the
retained-proof fallback $1.00/.00/.00/.00$.  The flags are
\texttt{CONTESTED}, \texttt{OPEN\_PREMISE}, \texttt{PROOF\_FALLBACK}, and
\texttt{SCRUTINY\_INCOMPLETE}.  The diagnostic \texttt{-dwopen} mode instead
tries substitutions drawn from the known constants, subject to a fixed
bound.  Other detected implementation limits produce the same flagged
retained-proof fallback and carry their own cause flags.

\paragraph{Recursive rules: a coverage difference.}  The tested example has
nine certain links and a transitivity rule at $0.9$.  \GK{} retains one
derivation, a single chain using eight distinct ground instances of that rule,
and reports $0.9^8=0.4305$; the example contains no opposing evidence, so no
contested atom triggers the dependency-aware evaluation and the report is the
direct retained-proof result.  Both reference models assign the query a
probability close to one; the ST sampler and bounded-prover CA sample return
values near that limit.  The query has many distinct derivations using ground
instances of the transitivity rule, so almost every sampled world contains at
least one derivation.  \GK{}'s $0.4305$ is Equation~\eqref{eq:ie} evaluated on
the single derivation retained by the bounded search, which
Section~\ref{sec:proofunion}
identified as a lower bound on the reference value whenever the retained proof
set is incomplete; this example has the largest observed difference in the
paper between that bound and the reference value.  The returned value depends
on which proofs the bounded search retains.  Recursive
programs of this shape are therefore a documented limitation of retained-proof
coverage.  They do not test agreement between the two reference semantics.

The example-suite rows add the remaining coverage classes: factoring lies
outside the directed fragment; contested cycles produce flagged fallbacks; and
non-ground answers remain retained-proof approximations.  Their individual
paths and causes are listed in Section~\ref{sec:examplesuiteresults}.

\subsubsection{Directed-reading limitation}

\paragraph{An unflagged direction limitation.}
Let
\[
 .9::penguin(a),\qquad .8::\neg bird(a),\qquad
 bird(x)\leftarrow penguin(x),\qquad
 swims(x)\leftarrow penguin(x),
\]
and ask $swims(a)$.  The directed reading contains no rule concluding
$\neg penguin$ or $\neg swims$.  Although $bird$ is contested, the retained
proof of $swims(a)$ does not pass through it, so dependency-aware evaluation is not triggered.
The result is $.90/.00/.00/.10$ with no flag.  A classical reading of the
certain confidence-one rule $bird(x)\leftarrow penguin(x)$ would permit a
modus-tollens derivation of $\neg penguin(a)$ from $\neg bird(a)$.
The defined directed semantics excludes this route and does not flag the
exclusion.  Both reference constructions follow the defined
direction and therefore do not add that classical route.

\subsection{Example suite results}
\label{sec:examplesuiteresults}

The example driver produced 43 answer rows.  Table~\ref{tab:studyclass}
reports the \GK{} result and the corresponding shared-threshold sample
rather than only an aggregate classification.  \GK{} and shared-threshold (ST) cells use the order
$s^+/s^-/c/g$.  The stored clause-activation values are omitted because their
completion counts are unavailable.  ST used 100,000 worlds with seed 7 and
the public threshold-world implementation
\path{montecarlo/threshold_worlds.py}.  The local status codes are:

\begin{center}
\scriptsize
\begin{tabular}{L{1.5cm}L{11.7cm}}
\toprule
Code & Meaning \\
\midrule
\texttt{RP/in} & direct result exact for the retained proof set \\
\texttt{DA/in} & completed dependency-aware result within the correspondence fragment \\
\texttt{FB/in} & fallback exact for the retained proof set \\
\texttt{BL/out} & completed local blocked/unblocked result outside atom-level correspondence \\
\texttt{RP/out} & direct retained-proof result outside the ST fragment \\
\texttt{FB/out} & retained-proof fallback outside the completed shared-threshold fragment \\
\texttt{AP/out} & retained-proof approximation after replay failure, a non-ground identifier, or a union limit \\
\bottomrule
\end{tabular}
\end{center}

The code-specific suffix meanings are defined in
Section~\ref{sec:reportstatus}; they do not name one common fragment.  The three
non-ground candidates are approximations; contested cycles are fallbacks
outside the completed shared-threshold fragment.  ST supports neither case.
Appendix~\ref{sec:completeinputs} gives every example in full without embedded
output logs.
The interpretation column classifies only the \GK{}--ST relation.
``Agreement'' means component-wise agreement within the conservative
per-component Monte Carlo tolerance $.0064$ stated in
Section~\ref{sec:samplers}; fallback, cycle, factoring, or non-ground labels state why no
within-fragment equality claim is made.

{\scriptsize
\begin{longtable}{L{2.8cm}L{3.6cm}cL{3.1cm}L{2.7cm}}
\caption{Complete set of results for the example suite.  Numerical ST
results use the order \(s^+/s^-/c/g\).  The suite tests specified semantic and
implementation cases.}
\label{tab:studyclass}\\
\toprule
Example and answer & \GK{} \sfour & \shortstack{Method/\\status} & ST result & \GK{}--ST interpretation \\
\multicolumn{5}{@{}l@{}}{Codes: DA dependency-aware; RP retained-proof; FB fallback; AP approximation; in/out: within/outside its fragment.}\\
\midrule
\endfirsthead
\multicolumn{5}{c}{\tablename\ \thetable\ continued}\\
\toprule
Example and answer & \GK{} \sfour & \shortstack{Method/\\status} & ST result & \GK{}--ST interpretation \\
\multicolumn{5}{@{}l@{}}{Codes: DA dependency-aware; RP retained-proof; FB fallback; AP approximation; in/out: within/outside its fragment.}\\
\midrule
\endhead
\bottomrule
\endfoot
1: $bird(a)$ & .750/0/0/.250 & RP/in & .751/0/0/.249 & agreement \\
2: $bird(a)$ & .800/0/0/.200 & RP/in & .801/0/0/.199 & agreement \\
3: $twobirds$ & .300/0/0/.700 & RP/in & .303/0/0/.697 & agreement \\
4: $(a,a)$ & .500/0/0/.500 & RP/in & .500/0/0/.500 & agreement \\
4: $(a,b)$ & .300/0/0/.700 & RP/in & .303/0/0/.697 & agreement \\
4: $(b,a)$ & .300/0/0/.700 & RP/in & .303/0/0/.697 & agreement \\
4: $(b,b)$ & .600/0/0/.400 & RP/in & .601/0/0/.399 & agreement \\
5: $bird(a)$ & 0/0/.500/.500 & DA/in & 0/0/.500/.500 & agreement \\
6: $bird(a)$ & .300/0/.500/.200 & DA/in & .300/0/.500/.199 & agreement \\
7: $\neg bird(a)$ & 0/.300/.500/.200 & DA/in & 0/.301/.500/.199 & agreement \\
8: $smokes(carl)$ & .096/0/0/.904 & RP/in & .096/0/0/.904 & agreement \\
9: $smokes(carl)$ & .1376/0/0/.8624 & RP/in & .138/0/0/.862 & agreement \\
10: $smokes(ann)$ & .800/0/0/.200 & RP/in & .800/0/0/.200 & agreement \\
10: $smokes(bob)$ & .688/0/0/.312 & RP/in & .689/0/0/.311 & agreement \\
10: $smokes(carl)$ & .1376/0/0/.8624 & RP/in & .138/0/0/.862 & agreement \\
11: $smokes(sam)$ & .3764/0/0/.6236 & RP/in & .379/0/0/.621 & agreement \\
12: $smokes(sam)$ & .3084/0/0/.6916 & RP/out & .302/0/0/.698 & outside fragment: factoring \\
13: $burglary(t1)$ & $.0822/0/.9156/.0022$ & FB/out & unsupported: cycle & nonmonotonic cycle \\
14: $earthquake(t1)$ & $.0388/0/.9448/.0164$ & FB/out & unsupported: cycle & nonmonotonic cycle \\
15: $smokes(1)$ & $.440/0/0/.560$ & FB/in & .439/0/0/.561 & fallback value agrees \\
15: $smokes(2)$ & 1/0/0/0 & RP/in & 1/0/0/0 & agreement \\
15: $smokes(3)$ & .300/0/0/.700 & RP/in & .298/0/0/.702 & agreement \\
15: $smokes(4)$ & $.440/0/0/.560$ & FB/in & .298/0/0/.702 & fallback value differs \\
16: $asthma(1)$ & $.176/0/0/.824$ & FB/in & .176/0/0/.824 & fallback value agrees \\
16: $asthma(2)$ & .400/0/0/.600 & RP/in & .400/0/0/.600 & agreement \\
16: $asthma(3)$ & .120/0/0/.880 & RP/in & .119/0/0/.881 & agreement \\
16: $asthma(4)$ & $.176/0/0/.824$ & FB/in & .119/0/0/.881 & fallback value differs \\
17: $abeautifulmind$ & .9997/0/0/.0003 & RP/in & 1/0/0/0 & agreement \\
17: $casinoroyale$ & 1/0/0/0 & RP/in & 1/0/0/0 & agreement \\
17: $sleepyhollow$ & .9082/0/0/.0918 & RP/in & .908/0/0/.092 & agreement \\
17: $robinhood$ & .9764/0/0/.0236 & RP/in & .977/0/0/.023 & agreement \\
18: $flies(a)$ & 0/0/0/1 & DA/in & 0/0/0/1 & agreement \\
19: $bird(a)$ & 0/0/.500/.500 & DA/in & 0/0/.500/.500 & agreement \\
20: non-ground answer candidate & 0/.4345/.090/.4755 & AP/out & unsupported: non-ground query & non-ground answer \\
20: $flies(a)$ & 0/0/0/1 & FB/out & unsupported: cycle & nonmonotonic cycle \\
21: non-ground answer candidate & 0/.455/.090/.455 & AP/out & unsupported: non-ground query & non-ground answer \\
21: $bird(a)$ & 0/0/.500/.500 & FB/out & unsupported: cycle & nonmonotonic cycle \\
22: $flies(tweety)$ & .8991/.0001/.0009/.0999 & DA/in & .900/.000/.001/.099 & agreement \\
22: $flies(pennie)$ & 0/.100/.900/0 & DA/in & 0/.099/.901/0 & agreement \\
23: non-ground answer candidate & 0/.910/.090/0 & AP/out & unsupported: non-ground query & non-ground answer \\
23: $flies(messy)$ & 0/0/0/1 & FB/out & unsupported: cycle & nonmonotonic cycle \\
23: $flies(pennie)$ & 0/.100/.900/0 & FB/out & unsupported: cycle & nonmonotonic cycle \\
23: $flies(tweety)$ & .800/0/.100/.100 & FB/out & unsupported: cycle & nonmonotonic cycle \\
\end{longtable}
}

The two fallback differences are visible in rows 15(4) and 16(4):
\GK{} preserves retained-proof support $.44$ and $.176$, while ST gives
approximately $.30$ and $.12$.  Both \GK{} rows carry the fallback cause;
Section~\ref{sec:twomodels} works through the 15(4) mechanism.
The remaining numerical difference is Example~12: the direct retained-proof
method returns $.3084$ and ST approximately $.302$.  Its non-Horn encoding permits classical factoring in
the resolution proof, and the directed reading used by ST defines no
counterpart of that step.  Factoring-derived answers are therefore excluded
from the stated fragment of Section~\ref{sec:samplers}.  This row does not
contradict the correspondence proposition.
The evaluator does not yet detect the exclusion, so the row carries no flag;
a flagged retained-proof fallback for such answers, or a directed-support
counterpart of factoring, is future work
(Section~\ref{sec:limitations}).  Examples 13, 14, 20, 21, and 23 have no ST
number because their contested cycles are outside the sampler's fragment.

\subsection{Literature default benchmarks}

Table~\ref{tab:lifschitz} reports the independently specified A1--A6
benchmarks from Lifschitz's \emph{Benchmark Problems for Formal Nonmonotonic
Reasoning}.  The reference conclusion is the one stated in the published
benchmark, not an
output chosen to match \GK{}.  Six of the seven queried conclusions agree.
A5 is qualified because the matching answer is a retained-proof fallback for
an unsupported formula.  A6b is a known divergence: its intended
conclusion requires a minimal-exception choice from a disjunctive premise,
which \GK{}'s local blocker procedure does not implement.
The case labels and intended conclusions follow the cited source benchmark.
Here \texttt{DA/in} means a completed
shared-threshold result within the correspondence fragment, while
\texttt{FB/out} means a retained-proof fallback outside that fragment.

\begin{table}[ht]
\centering
\small
\caption{Literature-originated default benchmarks.  ``DA/in'' concerns
\GK{}'s dependency-aware calculation.  No result is attributed to another system
because no faithful translation of these cases was run in another system.}
\label{tab:lifschitz}
\begin{tabular}{L{1.1cm}L{4.5cm}L{4.0cm}L{3.6cm}}
\toprule
Case & Intended conclusion & Comparison with intended conclusion & Method/status or limitation \\
\midrule
A1 & B is on the table & match & DA/in \\
A2 & B is on the table despite irrelevant information & match & DA/in \\
A3 & B is on the table and A is red & match & DA/in \\
A4 & B is on the table with A's default disabled & match & DA/in \\
A5 & every heavy block other than A is on the table
& intended conclusion returned by flagged fallback & FB/out \\
A6a & C is on the table
& match & DA/in \\
A6b & exactly one of A and B is not on the table
& divergence & unsupported minimal-exception policy \\
\bottomrule
\end{tabular}
\end{table}

\section{Comparison with implemented systems}
\label{sec:systemcomparison}

Table~\ref{tab:systemmaster} reports each system's output in its native
type; probabilities, intervals, model sets, and argument statuses are not
converted to a common score.  A scalar is the result for the displayed query; a \GK{} scalar is its
signed confidence.  The entries \(0\) (ign.) and \(0\) (bal.) distinguish
pure ignorance from balanced positive and negative support.  ``2 models''
means that two stable models were returned.  For TweetyProject, ``undec.''
means that it constructs opposed DeLP arguments and warrants neither
conclusion.  Other
textual answers retain their native meaning.  Native inputs have no prefix,
and different output kinds are not numerically comparable.
The entries \emph{error} and \emph{does not ground} record the system's own
failure output for the stated input; the captured messages are in the
comparison package.
Complete case definitions are in Appendix~\ref{sec:crossinputs}.

\begin{center}
\scriptsize
\begin{tabular}{L{2.2cm}L{8.0cm}}
\toprule
Notation & Meaning \\
\midrule
scalar & result for the displayed query \\
$x/y$ & displayed query / explicit opposite query \\
$[l,u]$ & lower / upper probability \\
E & exact translation on the stated fragment \\
A & non-equivalent analogue \\
U & no faithful encoding claimed \\
\(0\) (ign.) & signed confidence zero with pure ignorance \\
\(0\) (bal.) & signed confidence zero with nonzero positive and negative
support; the scalar omits conflict and ignorance \\
\(C\;(s^+/s^-)\) & \GK{} signed confidence with its positive and negative
support components \\
error & the system's own error or refusal on the stated input \\
does not ground & grounding did not terminate within the recorded
ten-second limit \\
inc $m$ & smProbLog probability mass on worlds without stable models \\
\bottomrule
\end{tabular}
\end{center}

{\scriptsize
\setlength{\tabcolsep}{1.25pt}
\begin{longtable}{@{}L{.55cm}L{2.6cm}@{\hspace{2pt}}C{1.15cm}C{1.35cm}C{1.55cm}C{1.25cm}C{1.25cm}C{1.0cm}C{.95cm}C{.95cm}C{1.05cm}@{}}
\caption{Implemented-system comparison.  Native inputs have no prefix;
E = exact translation, A = non-equivalent analogue, and U = no faithful
encoding claimed.}
\label{tab:systemmaster}\\
\toprule
ID & Example & \GK{} & \ProbLog{} & PASTA & plingo &
\shortstack{sm\\ProbLog} &
Tweety & clingo & DLV & \shortstack{s\\(CASP)} \\
\midrule
\endfirsthead
\multicolumn{11}{c}{\tablename\ \thetable\ continued}\\
\toprule
ID & Example & \GK{} & \ProbLog{} & PASTA & plingo &
\shortstack{sm\\ProbLog} &
Tweety & clingo & DLV & \shortstack{s\\(CASP)} \\
\midrule
\endhead
\bottomrule
\endfoot
C1 & independent causes
& $.80$ & E:$.80$ & E:$[.80,.80]$ & E:$.80$ & E:$.80$
& U & U & U & U \\
C2 & conjunctive causes
& $.30$ & E:$.30$ & E:$[.30,.30]$ & E:$.30$ & E:$.30$
& U & U & U & U \\
D1 & evidence for negated conclusion
& \shortstack{$-.80$\\$(.10/.90)$} & \shortstack[l]{A:$.10$/\\$.90$} &
\shortstack[l]{A:$[.10,.10]$/\\$[.90,.90]$} &
\shortstack[l]{A:$.10$/\\$.90$} & \shortstack[l]{A:$.10$/\\$.90$}
& U & \shortstack[l]{A:no/\\yes} & \shortstack[l]{A:no/\\yes} &
\shortstack[l]{A:no/\\yes} \\
D2 & exception condition with opposing support
& $.54$ & A:$.27$ & A:$[.27,.27]$ &
A:$.27$ & A:$.27$
& U & A:no & A:no & A:no \\
D3 & contested premise
& 0 (ign.) & A:$.225$ &
\shortstack[l]{A:\\$[.225,.225]$} &
A:$.225$ & A:$.225$
& U & A:no & A:no & A:no \\
D4 & Nixon defaults
& 0 (ign.) & U & \shortstack[l]{$[0,1]$/\\$[0,1]$} & $.50/.50$ & $.50/.50$
& \shortstack[l]{undec./\\undec.} & \shortstack{2\\models} &
\shortstack{2\\models} & \shortstack{yes/\\yes} \\
D4-P & probabilistic Nixon premises
& \shortstack[l]{A:$+.20$/\\$-.20$} & U &
\shortstack[l]{$[.32,.80]$/\\$[.12,.60]$} &
\shortstack[l]{$.54054$/\\$.40540$} &
\shortstack[l]{$.56$/\\$.36$}
& U & U & U & U \\
D5 & strict priority
& $.30$ & \shortstack[l]{A:$.63$/\\$.30$} &
\shortstack[l]{A:$[.63,.63]$/\\$[.30,.30]$} &
\shortstack[l]{A:$.63$/\\$.30$} & \shortstack[l]{A:$.63$/\\$.30$}
& \shortstack[l]{A:no/\\yes} & \shortstack[l]{A:no/\\yes} &
\shortstack[l]{A:no/\\yes} & \shortstack[l]{A:no/\\yes} \\
D6a & exception undercuts application
& $1$ & A:$1$ & A:$[1,1]$ & A:$1$ & A:$1$
& U & A:yes & A:yes & A:yes \\
D6b & opposite conclusion rebuts
& $1$ & A:$1$ & A:$[1,1]$ & A:$1$ & A:$1$
& U & A:yes & A:yes & A:yes \\
D7 & paired reference-class exception
& \shortstack{$0$ (bal.)\\$(.48/.48)$} & \shortstack[l]{A:$.48$/\\$.48$} &
\shortstack[l]{A:$[.48,.48]$/\\$[.48,.48]$} &
\shortstack[l]{A:$.48$/\\$.48$} & \shortstack[l]{A:$.48$/\\$.48$}
& U & U & U & U \\
EQ1 & equality with opposed evidence
& \shortstack{$.34$\\$(.90/.56)$} & U & A:error &
A:$.79839$ & \shortstack[l]{A:$.82$;\\inc $.504$}
& U & \shortstack[l]{A: no\\models} & \shortstack[l]{A: no\\models} &
\shortstack[l]{A: no\\models} \\
X1 & contradiction beside unrelated query
& $1$ & U & A:error & A:$1$ & \shortstack[l]{A:$1$;\\inc $.42$}
& U & \shortstack[l]{A: no\\models} & \shortstack[l]{A: no\\models} &
\shortstack[l]{A: no\\models} \\
F1 & recursion over function terms
& $.576$ & E:$.576$ & \shortstack[l]{A: does\\not\\ground} &
\shortstack[l]{A: does\\not\\ground} & \shortstack[l]{E: does\\not\\ground}
& U & \shortstack[l]{A: does\\not\\ground} &
\shortstack[l]{A: does\\not\\ground} &
A:yes \\
DJ1 & uncertain classical disjunction
& $.72$ & U & E:$[.72,.80]$ & E:error & A:$.76$
& U & A:yes & A:yes & A:yes \\
\end{longtable}
}

The public repository's \path{comparisons/} package contains the input,
recorded command, and captured output for each cell backed by a run.  Its
\path{README.md} states the package contents, system versions, and rerun
procedure.  \path{manifest.json} and \path{results/table_results.tsv} cover
every table cell in machine-readable form.  Executed cells name their input,
query, and run identifier; each U cell instead records why no faithful
encoding is claimed.  The \GK{} cells reproduce with the public
binary; the external cells require the corresponding system at the stated
version.

\begin{footnotesize}
\paragraph{Encoding notes.}
\textbf{D1.} Compiles the simple exception to negation as failure; the
non-probabilistic encodings return only derivability and make the exception
fact certain.
\par\noindent
\textbf{D2.} Contains facts with confidences \(.7::B\) and \(.3::\neg B\), and
a default with rule confidence \(.9::H\,[\mathrm{unless}\ B]\).  The external
inputs replace that default by an independent $.9$ activation and
replace its exception check by a negation-as-failure guard.  They therefore
do not implement \GK{}'s recursive exception semantics.  The \texttt{neg\_b}
fact mirrors the source input but is disconnected in this analogue; the guard
tests only whether \texttt{b} is absent.
\par\noindent
\textbf{D3.} Represents evidence for the explicit negation by a separate atom and adds a
negation-as-failure guard; the non-probabilistic encodings make both premise
literals certain.
\par\noindent
\textbf{D4.} The two s(CASP) queries succeed in separate partial stable models.
Tweety returns \emph{undecided} for each query because neither opposed DeLP
argument is warranted.
\par\noindent
\textbf{D4-P.} Assigns confidence $.8$ to \(quaker(n)\) and $.6$ to
\(republican(n)\).  The stable-model systems apply their native
negation-as-failure cycle.  \GK{} has no such operator here, so its
equal-priority contrary-gated input is marked A.  It returns signed
confidence \(+.20\) for \(pacifist(n)\); the explicit opposite query is
reported as \(-.20\).
\par\noindent
\textbf{D5.} Gives the positive default confidence $.9$ at priority 2 and the
negative default confidence $.3$ at priority 3.  The external encoding guards
the positive rule by the absence of the $.3$ negative activation.  The Tweety
analogue uses specificity; the non-probabilistic encodings activate both
sides.
\par\noindent
\textbf{D6a--D6b.} \GK{} distinguishes the two rule forms: D6a contains a
potential undercutting exception, whereas D6b uses the explicit opposite
conclusion as its exception condition.  Because neither condition is supported
in these inputs, both queries receive confidence one.  The external
translations do not preserve the rule-form distinction.
\par\noindent
\textbf{D7.} The external encoding explicitly adds rules whose probabilities
reproduce the two output marginals; it does not implement \GK{}'s paired
reference-class operation.  The analogue reproduces the marginals but not the
operation.
\par\noindent
\textbf{EQ1.} Contains \(.9::p(a)\), \(.7::\neg p(b)\), and the uncertain
equality \(.8::(a=b)\).  \GK{} evaluates the query by paramodulation and
nets \(.9\) against \(.7\times.8\).  The external encodings state the
equality as an ordinary probabilistic fact with substitution rules for both
polarities of \(p\); their joint world derives \(p(b)\) and \(\neg p(b)\)
and has no model.  \ProbLog{} cannot state \(\neg p(b)\); its recorded
fresh-atom attempt returns \(.9\) and \(.56\) as unrelated marginals, so
the cell is U.  The s(CASP) input directs the substitution rules toward the
query; with the symmetric rules its goal-directed execution does not
terminate.
\par\noindent
\textbf{X1.} Asserts \(.7::p(a)\) and \(.6::\neg p(a)\) beside the certain
unrelated fact \(q(b)\).  The ASP-based encodings state the pair through
strong negation; smProbLog uses a fresh atom with an incompatibility loop
and prints a negative false mass beside the inconsistent mass \(.42\).
\ProbLog{} programs are consistent by construction: a second probabilistic
fact for the same atom combines by noisy-or, evidence conditioning replaces
the pair by a posterior, and contradictory hard evidence aborts every
query, so the cell is U.
\par\noindent
\textbf{F1.} Applies the rule \(.8::bird(f(X))\leftarrow bird(X)\) twice to
a \(.9\) fact.  \ProbLog{} keeps the program verbatim, as does smProbLog;
the PASTA and plingo files omit the rule weight, which would require a
probabilistic gate for every instance of the infinite Herbrand universe.
The grounding-based runs were stopped after ten seconds.  The s(CASP)
encoding is crisp and answers by top-down execution.
\par\noindent
\textbf{DJ1.} Assigns confidence \(.8\) to the classical disjunction
\(p(a)\lor q(a)\) and \(.9\) to \(\neg p(a)\).  PASTA states the
disjunctive head directly; smProbLog and s(CASP) state it by the standard
even-loop idiom.  The plingo problog front end fails on the disjunctive
head with a parse-transformation error.  An annotated disjunction
distributes \(.8\) among the disjuncts and is a different statement, so the
\ProbLog{} cell is U.
\par\medskip\noindent
The unguarded D3 value $.45$ was computed analytically; no system was run
for it.
\end{footnotesize}

Rows N1 and N2 are non-ground source queries.  Systems requiring a ground
query were run once for each listed binding; the rules remain non-ground.  The
non-probabilistic encodings return binding sets rather than probabilities.

\begin{table}[ht]
\centering
\scriptsize
\caption{Open-query results.  N denotes a native input, E an exact translation
on the stated fragment, and A a non-equivalent binding-set analogue.  Values follow the
binding order in the second column.  The E translations agree on the listed
per-binding probabilities; the A translations reproduce only the binding
set.}
\label{tab:openqueries}
\begin{tabular}{L{.7cm}L{4.4cm}L{4.7cm}L{4.7cm}}
\toprule
ID & Query and binding order & Probabilistic results & Non-probabilistic results \\
\midrule
N1 & \(twobirds(X,Y)\):
\((a,a),(a,b),(b,a),(b,b)\)
& \GK{} N: \((.50,.30,.30,.60)\);
\ProbLog{}, plingo, and smProbLog E: the same values; PASTA E:
\shortstack[l]{\(([.50,.50],[.30,.30],\)\\
\(\phantom{(}[.30,.30],[.60,.60])\)}
& clingo, DLV, s(CASP) A: all four bindings \\
N2 & \(smokes(X)\): \(ann,bob,carl\)
& \GK{} N: \((.80,.688,.1376)\);
\ProbLog{}, plingo, and smProbLog E: the same values; PASTA E:
\shortstack[l]{\(([.80,.80],[.688,.688],\)\\
\(\phantom{(}[.1376,.1376])\)}
& clingo, DLV, s(CASP) A: all three bindings \\
\bottomrule
\end{tabular}
\end{table}

\paragraph{Exact probabilistic fragment.}
C1 and C2 test ground disjunction and conjunction.  N1 and N2 add open queries,
multiple substitutions, and recursion.  \GK{}, \ProbLog{}, PASTA, plingo, and
smProbLog agree on every listed per-binding probability.  The non-probabilistic systems
derive the same bindings in N1 and N2 but return no probabilities, so those
cells are binding-set analogues rather than numerical confirmations.

\paragraph{Probabilities combined with a default.}
In D2 both the facts and the default rule carry input confidences, so the
row tests the combination directly.  \GK{} evaluates the default with confidence $.9$ and a
recursively checked
exception and returns $.54$.  \ProbLog{}, PASTA, plingo, and smProbLog can run
the stated activation-and-negation-as-failure translation, but that translation
returns $.27$ and is marked A.  The non-probabilistic encodings make both facts
and the activation certain and do not derive the query.  D2 therefore shows which part of the \GK{} default operation the
translation fails to preserve, rather than a numerical disagreement under
one shared semantics.

\paragraph{Opposition and ambiguity.}
D1 and D3 show the same issue from the evidence side.  The external analogues
use a guard to exclude worlds in which the separate exception atom is active,
while \GK{} recursively
evaluates opposition and reports either ignorance or signed confidence
favoring the explicit opposite.  In D4 the Nixon conflict is native in all
of the selected nonmonotonic systems, but their output types differ:
ignorance in \GK{}, probability intervals, system-specific probabilities,
undecided arguments, or separate stable models.  These outputs express different semantics.

\paragraph{Probabilistic nonmonotonic cycle.}
D4-P adds probabilities to the Nixon premises, not to its default rules.
PASTA, plingo, and smProbLog handle the resulting probabilistic stable-model
input natively and return different quantities under their respective
semantics.  \GK{} has no native stable-model negation-as-failure cycle in this
case; its \(+.20/-.20\) entry is explicitly an equal-priority
contrary-gated analogue.  No faithful D4-P encoding was constructed for the
remaining systems, so those cells are marked U.

\paragraph{GK-specific default structure.}
D5--D7 isolate strict priority, the distinction between undercutting
exceptions and rebutting evidence, and the paired reference-class construction.  External
inputs can reproduce selected finite outcomes, but they do not preserve those
constructs.  Their A and U entries are therefore translation limits, not
failed runs.  No row among C1--D7 has an operational failure.  In rows EQ1,
X1, F1, and DJ1 the compared result of several cells is the system's own
failure; the comparison package contains each failure output.

\paragraph{Equality with opposed evidence.}
EQ1 attaches a confidence to an equality and lets it carry a negative
literal onto the queried atom.  \GK{} answers by paramodulation with signed
confidence $.34$ and reports the opposing mass $.56$ as conflict.  None of
the compared systems has equality with confidence; their encodings state
the equality as a probabilistic fact with hand-written substitution rules
for the one predicate of the row.  Under that encoding the joint world is
inconsistent: PASTA reports an error, plingo renormalizes to $.79839$, and
smProbLog moves $.504$ to its inconsistency column.  The
non-probabilistic encodings return no models.  With the opposition
removed, a single substitution rule reproduces the \GK{} value exactly, so
the opposed form is the discriminating part of the row.

\paragraph{Contradiction beside an unrelated query.}
X1 asserts contradictory evidence about one atom and queries another.
The query-clause ancestry requirement of Section~\ref{sec:twomodels} keeps
the contradiction local, so \GK{} answers the unrelated query with
confidence one.  The non-probabilistic systems reject the whole program:
clingo and DLV return no models and the s(CASP) query fails.  The
probabilistic stable-model systems differ among themselves: PASTA declines
to answer because a world has no answer set, plingo returns one after
renormalizing away the inconsistent mass $.42$, and smProbLog reports that
mass explicitly.  In \ProbLog{} the pair is not statable, as the EQ1 and X1
encoding notes explain.  X1 thus tests the localization claim of
Section~\ref{sec:twomodels} on the compared implementations.

\paragraph{Recursion over function terms.}
F1 iterates a weighted rule through a function term.  \GK{} and \ProbLog{}
both answer $.576$; both ground on demand, driven by the query.
PASTA, plingo, smProbLog, clingo, and DLV ground the full program first;
none of them terminates on the infinite Herbrand universe, and each run was
stopped after ten seconds.  s(CASP) executes top down and confirms the crisp
conclusion without a probability.  F1 thus distinguishes query-directed
grounding from ground-and-solve; the underlying probability semantics plays
no part in the outcome.

\paragraph{Uncertain disjunction.}
DJ1 attaches confidence $.8$ to a classical disjunction and opposes one
disjunct with confidence $.9$.  \GK{} resolves the surviving disjunct at
$.72$.  The PASTA credal interval $[.72,.80]$ has the \GK{} value as its
lower bound; the upper bound comes from the worlds in which the undefended
disjunction leaves both disjuncts admissible.  smProbLog assigns those
worlds equally to their stable models and returns the point value $.76$.
The plingo front end fails on the disjunctive head, and \ProbLog{} has no
classical disjunction.  The non-probabilistic systems derive the surviving
disjunct without a probability.

Figure~\ref{fig:gkproblogpolarity} gives a separate view of the C1 agreement
and the D3-type difference.  The guarded \ProbLog{} query counts positive-only
worlds, whereas \GK{} retains opposing overlap as conflict.

\begin{figure}[htbp]
\centering
\begin{tikzpicture}
\begin{axis}[
  name=samepolarity,
  at={(0,0)}, anchor=south west,
  width=7.25cm, height=5.5cm,
  xmin=0, xmax=1, ymin=0, ymax=1,
  title={same polarity: $+p,+p$},
  xlabel={input confidence $p$},
  ylabel={reported value},
  xtick={0,.25,.5,.75,1},
  ytick={0,.25,.5,.75,1},
  tick label style={font=\scriptsize},
  label style={font=\footnotesize},
  title style={font=\footnotesize\bfseries},
  axis line style={draw=gkgray},
  tick style={draw=gkgray},
  grid=major,
  grid style={draw=gkgray!16, line width=.3pt},
  clip=false
]
  \addplot[draw=gkteal, line width=1.1pt, domain=0:1, samples=80]
    {2*x-x^2};
  \node[font=\scriptsize, text=gkteal, anchor=north west]
    at (axis description cs:.05,.95) {\GK{} $=$ \ProbLog{}: $2p-p^2$};
  \addplot[draw=gkgray, densely dotted, line width=.6pt]
    coordinates {(.5,0) (.5,1)};
  \addplot[only marks, mark=*, mark size=1.7pt, draw=gkteal, fill=gkteal]
    coordinates {(.5,.75)}
    node[font=\scriptsize, text=gkgray, above left] {$0.75$};
\end{axis}

\begin{axis}[
  at={(7.75cm,0)}, anchor=south west,
  width=7.25cm, height=5.5cm,
  xmin=0, xmax=1, ymin=0, ymax=1,
  title={opposite polarity: $+p,-p$},
  xlabel={input confidence $p$},
  xtick={0,.25,.5,.75,1},
  ytick={0,.25,.5,.75,1},
  yticklabels={},
  tick label style={font=\scriptsize},
  label style={font=\footnotesize},
  title style={font=\footnotesize\bfseries},
  axis line style={draw=gkgray},
  tick style={draw=gkgray},
  grid=major,
  grid style={draw=gkgray!16, line width=.3pt},
  clip=false
]
  \node[font=\scriptsize, text=gkpurple, anchor=north east]
    at (axis cs:.93,.89) {\GK{} report has $c=p$};

  \addplot[draw=gkblue, dashed, line width=1.0pt,
    domain=0:1, samples=80] {x*(1-x)};
  \node[font=\scriptsize, text=gkblue, anchor=south]
    at (axis cs:.56,.29) {\ProbLog{} query $p(1-p)$};

  \addplot[draw=gkteal, line width=1.1pt]
    coordinates {(0,0) (1,0)};
  \node[font=\scriptsize, text=gkteal, anchor=south]
    at (axis cs:.68,.025) {\GK{} $s^+=C=0$};

  \addplot[draw=gkgray, densely dotted, line width=.6pt]
    coordinates {(.5,0) (.5,1)};
  \addplot[only marks, mark=*, mark size=1.7pt, draw=gkblue, fill=gkblue]
    coordinates {(.5,.25)}
    node[font=\scriptsize, text=gkgray, below right] {$0.25$};
\end{axis}
\end{tikzpicture}
\caption{Same-polarity agreement and opposite-polarity divergence.  Left: two
independent positive causes with input confidence $p$ give the same noisy-or value in
\GK{} and \ProbLog{}.  Right: for equal positive and explicitly negated pools,
the checked guarded \ProbLog{} encoding returns $p(1-p)$, while \GK{} returns
$(s^+,s^-,c,g)=(0,0,p,1-p)$ and signed confidence $C=0$.  The markers at
$p=.5$ illustrate the plotted values $.75$ and $.25$; \GK{} conflict is
stated in the figure text rather than plotted as a second verdict.}
\label{fig:gkproblogpolarity}
\end{figure}
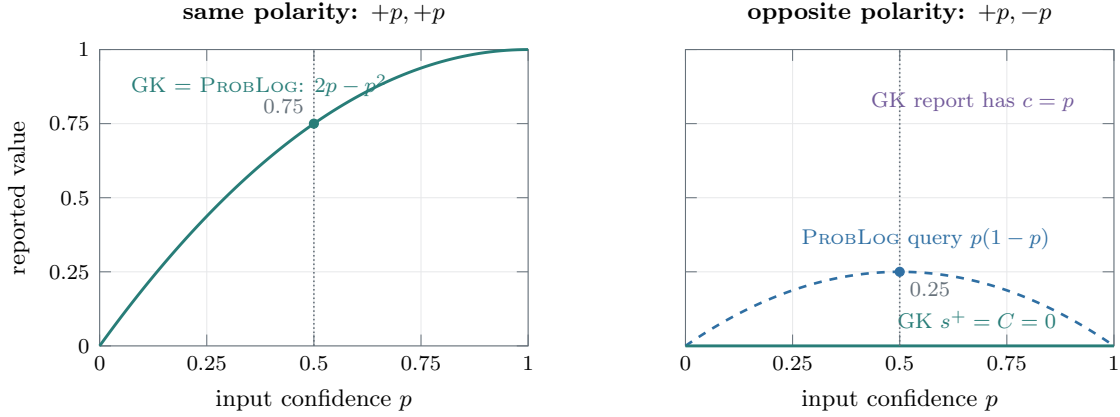

The signed-confidence calculation of \cite{tammet2021} remains available in
compatibility mode.  On the balanced premise followed
by a $.9$ rule it returns $.45$, because positive support is propagated
before the balanced premise is examined.  Configuration-wise evaluation returns
total ignorance for the conclusion and identifies the contested premise.  On
one-sided cases the two modes normally coincide; a repository regression test
fixes the expected compatibility-mode output.

\subsection{First-order scope and grounding}

The experiment measures how the tested implementations handle the query
$flies(b1)$ in two settings: a large finite closure irrelevant to the answer
and an infinitely groundable recursive rule.  Table~\ref{tab:fol} reports the
single-machine runs on one fixed host.  The host was Linux on an Intel Core
i7-10875H with eight cores, 16 hardware threads, and 30~GiB of memory.  Every
numeric cell reports wall seconds; $>$ marks the stated external time limit.
The complete commands and limits are in
\path{Examples/asp_comparison/results.md}.  These single-machine runs
illustrate grounding and search on the stated inputs; they do not rank overall
performance.  clingo grounds the finite closure,
I-DLV query mode applies Magic Sets,
s(CASP) evaluates top down, and \GK{} uses query-directed resolution.
The basic and function-term inputs are
\path{Examples/exceptions/gbirds.js} and
\path{Examples/exceptions/gbirds_funsymbs.js}.  The finite-size variants add
$N$ irrelevant \texttt{bird(bK)} facts to the basic input; they are generated
instances, not separate semantic problems.

{\scriptsize
\begin{table}[ht]
\centering
\setlength{\tabcolsep}{3pt}
\caption{Query-focused first-order workload.  Cells reporting a time,
grounding, or stack limit are incomplete runs rather than logical results.
``Cautious answer'' means that the query is true in every stable model.  DLV
results are reported in the text below.}
\label{tab:fol}
\begin{tabular}{L{3.0cm}L{2.8cm}L{2.9cm}L{2.8cm}L{3.0cm}}
\toprule
Input & \multicolumn{4}{c}{Result; wall time (s)} \\
\cmidrule(l){2-5}
& \GK{} & clingo 5.6.2 & I-DLV 1.1.6 query & s(CASP) 1.1.4 \\
\midrule
Basic birds
& answer; $0.10$ & cautious answer; $0.006$
& answer; time not recorded & answer; $0.07$ \\
1,000 constants
& answer; $0.40$
& cautious answer; $10.34$
& answer; $<0.01$
& time limit; $>30$ \\
2,000 constants
& answer; $0.39$
& cautious answer; $99.95$
& answer; $<0.01$
& stack limit; $27.55$ \\
100,000 constants
& answer; $5.83$
& time limit; $>30$
& answer; $0.24$
& stack limit; $19.70$ \\
Recursive function terms
& answer; $2.1$
& grounding limit; $>120$
& not run
& stack limit; $22.73$ \\
\bottomrule
\end{tabular}
\end{table}
}

The 1,000--100,000-constant inputs were generated from one normalized
two-chain specification.  The transitive relation has quadratically many
consequences but is irrelevant to the query.  clingo's completed runs reported
effectively zero solving time after grounding.  I-DLV shows that the comparison
cannot be reduced to query-directed systems versus ASP: its Magic Sets mode
also avoids the irrelevant closure and answered the
100,000-constant query in $.24$ seconds.  s(CASP) avoids grounding but followed
the left-recursive relation in this encoding and reached a time or stack
limit.

DLV 2.1.1 answered the basic, 1,000-, and
2,000-constant cases in under $.01$, $7.29$, and $68.72$ seconds, respectively,
and exceeded a 30-second limit on recursive function terms.  s(CASP) answered
the basic case in $.07$ seconds, reached its 1~GB stack limit on recursive
function terms after $22.73$ seconds, and exceeded 30 seconds on the
transitive-ancestor case.

In these runs, uncertainty reporting did not prevent answers on the tested
first-order inputs.  The certain Steamroller and
Dreadbury problems are found with
signed confidence one.  LCL047-1 with one uncertain axiom returns the per-instance value
$0.8^5=0.32768$, because the five proof uses are distinct ground instances.
The four CSR confidence problems return the same answers and values in default
and compatibility modes.  The uncertain Steamroller and Dreadbury cases use a
flagged retained-proof fallback because their disjunctive-conclusion proofs
are outside the dependency evaluator's supported history forms.  Their signed
values, $.0023$ and $.0057$, are below the default output threshold
\texttt{-confidence .1}, so the result label is \texttt{evidence below
limit}.  This threshold changes the label, not the selected calculation.

The basic birds input is deterministic, so both samplers reproduce the same
accepted and rejected answers as \GK{}.  Sampling the 1,000- through
100,000-constant timing variants would add no semantic evidence and would
mostly repeat expensive grounding.  The recursive function-term birds input,
Steamroller, Dreadbury, LCL047-1, and the CSR inputs are outside both finite
samplers because they contain function terms, equality, or non-ground
first-order structure.  These inputs test prover capability; they are not
used for the three-way numerical comparison.

A separate equality-free translation of a non-Horn set-identity theorem
tests the target languages.  \GK{} returned a 74-clause proof.  The tested
\ProbLog{} file reduced the classical disjunctive query to failure and is
therefore class U because the translation is unsupported; PASTA, clingo, and I-DLV reached
grounding resource bounds on the unbounded term closure; DLV rejected the
recursive term rule under its termination condition; TweetyProject's simple
FOL reasoner rejected the functor-bearing signature; and s(CASP) returned no
answer within the run bound.  Equality-bearing NLP and Dreadbury clauses were
not translated to ASP because target-language equality does not supply
paramodulation in non-Horn clauses.  These failures result from parsing,
grounding, or bounded search.  The public comparison
record named above contains the birds inputs, commands, bounds, and outputs.

\section{Relation to other semantics}
\label{sec:furtherrelated}

This section compares the two \GK{} reference models with related
probabilistic and nonmonotonic semantics.  The systems implemented and
evaluated here are \ProbLog{}, PASTA, plingo, smProbLog, TweetyProject, clingo,
DLV, I-DLV, and s(CASP), using the versions and evidence status of
Section~\ref{sec:comparisonprotocol}.  P-DeLP, LPPOD, CHRiSM, Markov logic,
probabilistic soft logic, subjective logic, and the other argument-strength
frameworks below are related systems or formalisms not experimentally
evaluated here unless a specific run is reported.

\subsection{Probabilistic logic programming}
\ProbLog{} defines a distribution over logic programs from independent
probabilistic choices and computes query success probabilities
\cite{deraedt2007}.  Its inference compiles relevant ground programs to
weighted Boolean formulas or related circuits~\cite{fierens2015}.  Algebraic
model counting generalizes the calculation to semirings~\cite{kimmig2017}.
The ground-instance activation semantics of \GK{} uses the same independent
ground-instance
interpretation on its retained proof set, while the surrounding prover accepts
general first-order clauses, equality, and function terms without a global
grounding phase.  The shared-threshold semantics adds a separate treatment of explicit
opposition and recursively evaluated exception conditions.
Appendix~\ref{sec:causalboundary} distinguishes these forward
calculations from conditioning, explaining away, intervention, and
counterfactual inference.

Stable-model probabilistic systems cover programs where distribution
semantics does not select one two-valued model.  smProbLog distributes the
mass of a total choice over its stable models and reports inconsistent
worlds~\cite{totis2023}.  Credal probabilistic answer-set semantics returns
lower and upper probabilities when several stable models are possible
\cite{cozman2020}; plingo adds several probabilistic modes to clingo
\cite{hahn2025}.  These systems combine probability and nonmonotonicity on
groundable logic programs.  \GK's optional conflict-sensitivity interval is
similar in purpose to a lower/upper sensitivity report, although its base
semantics and construction differ.

\subsection{Statistical relational models}
Markov logic attaches weights to first-order formulas and defines a log-linear
distribution over ground worlds~\cite{richardson2006}; Alchemy implements
learning and approximate or lifted inference for that model.  Probabilistic
soft logic replaces Boolean formulas by soft truth values and uses hinge-loss
Markov random fields~\cite{bach2017}.  These systems answer model-based
probability or maximum a posteriori (MAP) questions.  \GK{} treats input confidences
as externally supplied parameters and neither learns nor calibrates them.
Probabilistic theorem proving provides a lifted
alternative for probabilities of weighted first-order formulas
\cite{gogate2011}; it does not supply \GK's recursive default-justification
procedure.  Numerical agreement is therefore expected only on encodings with
aligned reference semantics.

\subsection{Defaults and answer sets}
Reiter's default logic introduced defaults with justifications and a proof
theory connected to resolution~\cite{reiter1980}.  Prioritized default logics
use specificity or explicit orders~\cite{brewka1994}.  Systems such as DLV
and clingo solve answer-set programs on finite ground programs
\cite{alviano2017,gebser2019}.  s(CASP) evaluates predicate answer-set programs
top down without grounding~\cite{arias2018}.  \GK{} uses a different
architecture.  Candidate proofs are generated by a general first-order
resolution prover, after which default justifications are checked under
decreasing time limits.  These bounds can leave a non-derivability check
incomplete, while the query-directed search can still handle function terms
and large irrelevant domains.

PASTA, plingo, and smProbLog combine numerical choices with stable-model
reasoning in different ways~\cite{azzolini2023,hahn2025,totis2023}.  The
relevant difference is that their native values concern
probabilistic choices and stable models, whereas \GK{} evaluates the region in
which an exception condition is positively usable for one rule application.
Section~\ref{sec:systemcomparison}
reports these native outputs as returned.

Possibilistic argumentation, LPPODs, CHRiSM, probabilistic structured
argumentation, and gradual bipolar argumentation provide further semantic
comparisons.  P-DeLP propagates possibilistic necessity degrees through
arguments and dialectical defeat, with later work adding argument
accrual~\cite{alsinet2008,gomez2013}; LPPODs attach necessity values to rules
and ordered alternatives~\cite{confalonieri2012}; and CHRiSM samples
acceptance of facts and rules~\cite{sneyers2013}.  We compare the numerical
quantities they define and their treatment of undercutting and rebutting, but do not treat a
formalism as an experimental comparator unless its implementation and
translated input were run for the comparison package.

The mechanisms differ in the quantity being measured.
$\beta_{B\mid A}$ is the conditional probability that the exception condition
is positively usable when the rule body is usable.  It determines the
proportions of the blocked and unblocked cases.  Equation~\eqref{eq:blockmix}
combines the corresponding conditional support contributions
$r_{\mathrm{blocked}}$ and $r_{\mathrm{open}}$.
This exception constitutes an undercutting attack.  Rebuttal
is calculated separately at the conclusion atom, retained conflict is
reported, and shared predecessors are enumerated before marginalization.

The implementation differs from the compared systems in the combination of
these capabilities.  ProbLog and
related systems provide probabilistic inference and conditioning on their
supported logic-program fragments but use a
logic-program and relevant-grounding pipeline; Markov logic admits general
weighted formulas but constructs a finite-domain graphical model; probabilistic
ASP supplies quantitative stable-model reasoning on groundable programs; and
s(CASP) supplies non-ground goal-directed defaults without numeric uncertainty.
\GK{} instead uses a non-Horn resolution prover,
including equality and recursively nested function terms, while attaching
input confidences and recursively checked prioritized exceptions to clauses.

\subsection{Conflict and incomplete information}
Belnap--Dunn logics use four logical states corresponding to supported true,
supported false, both, and neither~\cite{belnap1977,fitting1991}.  Recent work
combines this structure with Dempster--Shafer belief functions
\cite{bilkova2023}.  Subjective logic represents belief, disbelief, and
uncertainty with an additional base rate~\cite{josang2016}.  \GK's four fields
are probabilities of shared-threshold regions; they are not the truth values of
a four-valued logic or a subjective-logic opinion.  The similar report shape
reflects the shared need to distinguish contradiction from missing evidence.
The shared-threshold construction has a formal similarity to possibility
distributions: the events
$\{U_A\le a\}$ for varying $a$ nest exactly like the $\alpha$-cuts of a
possibility distribution~\cite{dubois1988}.  The shared-threshold
construction differs from a fuzzy interpretation of the support values in
pooling each side probabilistically before the single comparison and in
drawing the thresholds independently across atoms.

Dempster--Shafer theory represents masses on sets of hypotheses and provides a
rule for combining evidence~\cite{dempster1967,shafer1976}.  Dempster's rule
normalizes away conflicting mass.  \GK{} keeps the conflict region in the
report and reports the Fr\'echet lower bound on unopposed support.  This is appropriate
for the chosen shared-threshold construction; it is not proposed as a general
replacement for evidence-theory combination.
For overlapping bodies of evidence, Den\oe ux's cautious rule is idempotent
and does not assume distinct sources~\cite{denoeux2008}.  \GK{} addresses the
corresponding dependence problem at a different level when proof composition
is visible: shared ground events are represented once in every intersection
term of Equation~\eqref{eq:ie}.

Quantitative bipolar argumentation assigns weights to arguments and combines
support and attack relations~\cite{potyka2020}.  Recent bilateral gradual
semantics provide convergent values even on cyclic weighted bipolar graphs
\cite{wang2026}.  Both approaches combine support and attack at nodes
receiving opposing contributions.  \GK{} additionally records the ground rule
instances and shared activation events of the supporting proofs.  Its defined
cycle policies cover only the cases listed in Appendix~\ref{sec:cycles}.
Prakken gives a structured probabilistic account in which defeasible rule
strengths are conditional probabilities and distinguishes internal from
dialectical argument strength~\cite{prakken2018}.  It addresses uncertain
rules and undercutting attacks, but its
object is an ASPIC+ argument evaluated against a probability distribution;
the \GK{} implementation includes the blocked and unblocked default
calculation, paired reference-class rules, and explicit conflict reporting.
Probability-based measures of argument strength also evaluate necessity and
sufficiency of premises and competing claims under a distribution over
models~\cite{hunter2022}.  That account evaluates argument strength against a
model distribution.  \GK{} instead takes input confidences as parameters and
reports support regions derived from directed rule applications.

\subsection{Pooling rules}
\label{sec:poolingrules}
Linear, logarithmic, multiplicative, maximum, and noisy-or pooling rules have
different characterization properties and generating assumptions
\cite{genest1986,dietrich2016}.  MYCIN certainty factors use another rule for
combining supporting and opposing evidence only in the revised
opposite-sign case~\cite{buchanan1984,heckerman1986}; the original model used
the same pooled-belief minus pooled-disbelief form as \GK's signed confidence
\cite{shortliffe1975}.  The two support values do not determine a unique combination rule; each rule
adds a dependence or normalization assumption.  \GK{} fixes noisy-or for
same-polarity aggregation and shared-threshold opposition resolution for
opposing pools.
Equations
\eqref{eq:signed}--\eqref{eq:frechet} state the dependence property used to
select the latter.

When pooled belief and disbelief are both present, original MYCIN used the signed difference of its pooled
belief and disbelief measures~\cite{shortliffe1975}; \GK{} retains the
positive and negative parts separately.  The later MYCIN/EMYCIN rule is
$(a-b)/(1-b)$ for $a>b$~\cite{buchanan1984,heckerman1986}.  Other alternatives
include the multiplicative normalized difference
$(a-b)/(a(1-b)+(1-a)b)$, normalized Dempster--Shafer belief
$a(1-b)/(1-ab)$~\cite{dempster1967,shafer1976}, and averaging
$(a+1-b)/2$.  These formulas correspond to different dependence or
normalization assumptions.

At $(a,b)=(.7,.4)$, sampling gives unopposed positive support $.2996$ under
the shared-threshold coupling and $.4202$ under independent draws.  The
MYCIN/EMYCIN ratio, multiplicative normalized difference,
Dempster--Shafer belief, and average are $.500$, $.556$, $.583$, and $.650$.
The measured signed difference is approximately $.3$ under both couplings, as
Equation~\eqref{eq:signed} predicts.  The shared-threshold rule therefore yields the conservative lower bound; the
competing formulas are not interchangeable estimates of one quantity.

The provenance-aware calculation in Equation~\eqref{eq:measuredoverlap} is
ordinary union of proof-availability events after shared activation events
have been identified.  A scalar positive/negative support pair alone cannot reproduce the
asymmetric value $.6642$ or the three-proof chain value; the complete
activation-event sets determine both.

\section{Limitations}
\label{sec:limitations}

\paragraph{Search and proof coverage.}
First-order entailment and the non-derivability checks needed by defaults are
undecidable.  Main proof-search limits can omit proofs for the query, proofs
for its explicit negation, or proofs of exception conditions, so additional
search can move a report in either direction.  The threshold-world sampler
follows its entire finite grounded directed program, whereas retained-proof
aggregation is limited to retained proofs; recursive-rule examples therefore
remain a documented retained-proof coverage limitation.  They do not justify
treating each uncertain source statement as one activation event in place of
its ground instances.

\paragraph{Reconstruction and implementation limits.}
Same-polarity retained-proof values are exact only for the retained proof
set when every reconstructed activation-event identifier is ground and replay and
proof-union evaluation remain within their bounds.  When an identifier remains
non-ground, treating occurrences as proof-specific across proofs is an approximation,
not a guaranteed lower bound; the maximum of the individual proof lower bounds
is the guaranteed lower bound retained by the analysis.
A completed dependency-aware value is exact only for the directed derivations
reached by the report-time clause traversal within its stated fragment and
bounds.  Dependency-aware reports marked \texttt{canonical\_atom} and
\texttt{guaranteed} satisfy polarity symmetry.
Retained-proof fallbacks and approximations remain proof-pool decompositions
and are not interpreted as completed atom-level partitions.

Proof replay does not yet reconstruct every equality and rewriting history.
Inclusion--exclusion and joint enumeration have fixed caps.  Their fallbacks
are deterministic and signaled where the public report has a corresponding
flag, but the proof-union approximation applied when more than 20 reduced
activation-event sets are present currently appears as a warning rather than
a report field.  The following
dependency-evaluation limitations can also be unflagged:
\begin{itemize}
\item a non-ground answer can prevent traversal from starting;
\item an index or head-selection miss can prevent the dependency-evaluation trigger;
\item dependency-set overflow can select the retained-proof calculation;
\item the coarse dependency index can merge distinct ground predecessors with
the same predicate-symbol and arity key; and
\item a proof that uses classical factoring is outside the stated fragment
without being detected (Example~12 in Section~\ref{sec:results}).
\end{itemize}

\paragraph{Semantic scope.}
The directed semantics of Section~\ref{sec:premisesearch} excludes
contraposition.  This semantic restriction is unflagged.
The paper specifies selection between the retained-proof and dependency-aware
calculations and proves correspondence between the dependency evaluator and
shared-threshold semantics on the stated fragment.  It does not
define a single joint probability measure over clause activations and atom
thresholds.  A unified measure-theoretic account, or a proof of equivalence to
an established probabilistic logic on a specified syntactic fragment, remains to
be developed.

The paired reference-class complement of Appendix~\ref{sec:pairedexceptions}
assumes that an exception rule partitions its
reference class.  This interpretation is suitable only when the input
confidence was
elicited as a class frequency.  Ordinary source reliability does not license
the complement.  Input authors must distinguish the two uses.

Cycle behavior is defined (Section~\ref{sec:cycles}), but two aspects
remain open.  The credulous, query-relative resolution of even
exception-condition loops
is a design choice; a skeptical alternative would report mutual blocking for
both queries and could be adopted later.  Recursion through a contested atom
still has no completed dependency-aware value: the defined behavior is the
flagged retained-proof fallback obtained by resolving the opposing proof-pool
marginals, not a completed evaluation.  Open premises in generated clause
forms produce retained-proof fallbacks, not completed dependency-aware
evaluations.  Their point values are accompanied by fallback flags and must
not be interpreted as completed shared-threshold values.

\paragraph{Empirical scope.}
The evaluation uses small comparison cases designed to distinguish
the semantics and the commonsense examples introduced in \cite{tammet2021}.
It verifies formulas and implementation branches, but it
does not establish calibration on a large collection of independently measured
real-world frequencies.  Learning or calibrating input confidences from
data is outside the system.

The external-system comparison likewise uses a small set of comparison cases
and hand-checked translations.  It establishes the behavior of
the stated system versions on those cases; it is not a scalability benchmark,
an exhaustive comparison of system capabilities, or evidence that one
semantics subsumes another.  The N/E/A/P/U labels of
Section~\ref{sec:comparisonprotocol} keep those distinctions explicit.

\section{Conclusion}

\GK{} now retains the identities of uncertain premises while combining
proofs and resolves opposing support before it is propagated.  Uncertain
exceptions are evaluated for the rule applications they can block.  Reports
separate positive support, negative support, conflict, and ignorance and
identify detected incomplete calculations.

The retained-proof calculation reproduces proof-union probabilities on the
tested complete proof sets.  The dependency-aware calculation agrees with the
shared-threshold reference model on its stated finite fragment.  Comparisons
with other reasoners show agreement on the one-sided and deterministic
fragments with the same independent-choice and proof-union semantics, and
different outputs when the systems use different semantics for opposition,
defaults, or multiple models.  Explicit \GK{} priorities remain an
encoding-dependent translation limit; no equivalent numerical run is claimed.

The implementation remains bounded.  Proof search can omit explanations, some
report-time limits are unflagged, and several cycle and open-variable cases
return retained-proof fallbacks.  Remaining work includes broader replay
coverage, proof-guided completion of open premises, provenance-aware
aggregation in the threshold-world sampler, a stated criterion for choosing
between credulous and skeptical cycle policies, and empirical calibration of
input confidences.

\section*{Acknowledgments}

We thank Priit J\"arv and Dirk Draheim for ideas and contributions that
influenced the development of \GK{} and several of the concepts underlying
the approach presented in this paper.

\section*{Declaration on generative AI}

During the preparation of this work, the author used generative-AI
assistants---Claude Opus 4.8, Claude Fable 5 (Anthropic) and
GPT-5.6 Sol (OpenAI)---to draft experiment code and figure code,
edit text, and review drafts.  The author reviewed and edited the
resulting material as needed and takes full responsibility for the
content of this paper.

\appendix
\section{\GK{} syntax and report fields}
\label{sec:gksyntax}

This appendix collects the concrete notation and report vocabulary.  The main
text uses mathematical rule notation; concrete inputs use the syntax below.

\paragraph{Concrete \GK{} syntax.}
The mathematical notation used above is independent of the input language.
The following Prolog-style fragment shows the main public input forms:
\begin{verbatim}
0.8::bird(tweety).
0.9::flies(X) :- bird(X).
0.3::-bird(robin).

flies(X) :- bird(X), unless(injured(X), 0).
flies(X) :- bird(X), unless(-flies(X), 2).
flies(X) :- bird(X), unless(-flies(X), 0).

query(flies(tweety)).
query(-flies(tweety)).
query(flies(X)).
\end{verbatim}
The first two lines are a fact and rule with input confidences.  The third uses
explicit classical negation.  The next three rules show an arbitrary exception
condition, a contrary-gated default at explicit priority 2, and a
priority-zero default, respectively.  The final lines are ground, explicitly
negated, and open queries.
\texttt{unless(E,r)} names the exception condition \texttt{E} and the priority
\texttt{r}; zero is incomparable with positive ranks.  The prefix
\texttt{-} is explicit negation, not negation as failure.  \GK{} also accepts
JSON and TPTP input; all forms are clausified before proof search.

\paragraph{All-negative orientation fallback.}
In inputs using the \texttt{isa} convention, an \texttt{isa} literal denotes
a type-membership premise and is treated as a guard.  For an all-negative
clause without blocker metadata, the evaluator first selects the sole
non-\texttt{isa} eligible literal when every other eligible literal is an
\texttt{isa} guard.  Otherwise it uses the source-head marker recorded from an
authored rule before clause sorting.  If neither selection rule applies, it selects the
last eligible literal in the stored clause order after clausification.  The
last fallback is order-dependent.  Blocker and answer literals are not
eligible.

\textbf{Exception limit.}  The command-line option
\texttt{-blockerconfidence <value>} sets $\lambda$ in the exception-limit test
defined in Section~3.3.  The default is $0.5$.

For a query literal $L$, the four fields are positive support $s^+(L)$,
negative support $s^-(L)$, conflict $c(L)$, and ignorance $g(L)$.  Signed
confidence is $C(L)=s^+(L)-s^-(L)$.

With optional report details enabled, the current implementation reports the
following fields; the default output remains unchanged.
\begin{center}
\scriptsize
\begin{tabular}{L{2.6cm}L{4.2cm}L{6.4cm}}
\toprule
Field & Meaning & Values \\
\midrule
\texttt{calculation}
& calculation selected
& \texttt{canonical\_atom}, \texttt{blocked\_flat},
  \texttt{proof\_fallback}, \texttt{flat} \\
\texttt{coverage\_status}
& coverage assessment
& \texttt{complete}, \texttt{incomplete},
  \texttt{unsupported\_fragment} \\
\texttt{polarity\_status}
& polarity-orientation status
& \texttt{guaranteed}, \texttt{not\_guaranteed} \\
\bottomrule
\end{tabular}
\end{center}

Section~\ref{sec:reportstatus} supplies the paper-level interpretation.
\texttt{canonical\_atom} with the required coverage denotes a completed
shared-threshold partition.  \texttt{flat} denotes a direct retained-proof
result, and \texttt{proof\_fallback} a flagged retained-proof fallback; both
are proof-pool decompositions rather than atom-level partitions.
\texttt{blocked\_flat} denotes the local blocked/unblocked default
calculation.  A result outside its stated fragment has no atom-level partition
interpretation.

Common status causes include non-ground answers or premises, unsupported
factoring, nonmonotonic cycles, proof-union or traversal caps, replay failure,
enumeration limits, and deadlines.  Flags report detected causes and
fallbacks; Section~6.6 states the remaining unflagged limitations.

In dependency-aware reports, the sign of
$C(L)=s^+(L)-s^-(L)$ determines accepted-versus-rejected list placement and
the displayed assessment magnitude is $|C(L)|$.  Compatibility mode instead
uses the recursive exception-limit test, so candidate placement can differ
from the direction of the reported tuple.
For a ground answer obtained through an open query, the same four fields and
status vocabulary apply.  A completed within-fragment evaluation gives the same
atom-level tuple as the corresponding ground query.  A dependency-aware report with
\texttt{calculation=canonical\_atom} and
\texttt{polarity\_status=guaranteed} satisfies the component-swap property for
opposite-polarity ground queries.
Direct retained-proof results, flagged retained-proof fallbacks, and
outside-fragment numerical results do not acquire an atom-level partition
interpretation merely from their polarity presentation.  Recognized
incomplete cases carry the relevant cap or deadline cause and flags including
\begin{center}
\texttt{PROOF\_FALLBACK}, \quad
\texttt{SCRUTINY\_INCOMPLETE}, \quad
\texttt{OPEN\_PREMISE}.
\end{center}
The legacy output flag \texttt{SCRUTINY\_INCOMPLETE} denotes incomplete
dependency-aware evaluation.

\section{Further default cases}
\subsection{Paired reference-class exceptions}
\label{sec:pairedexceptions}

Ordinary default blocking does not require this extension: when an exception
condition is
usable, the blocked branch contributes zero.  \GK{} additionally recognizes the
restricted paired construction below, in which an explicit exception rule
partitions the blocked reference class.  It is kept in an appendix because
this class-frequency reading is optional and more specialized than the core
input-confidence and default semantics.  This section supplies the exact meaning of
the optional blocked-branch contribution $r_{\mathrm{blocked}}=p_{\rm exc}$ mentioned in
Section~\ref{sec:defaults} and states when the implementation may apply it.

For the complete ground query instance, let the input contain
\[
bird(a),\qquad 0.6::penguin(a),
\]
\[
0.9::flies(x)\leftarrow bird(x)
\quad[\mathrm{unless}\ penguin(x)],\qquad
0.8::\neg flies(x)\leftarrow penguin(x),
\]
and query $flies(a)$.
The main and exception branches are exclusive for one paired rule instance.
Under the
reference-class interpretation used by \GK, the exception-rule confidence
partitions
the blocked class: the $0.8$ part supports $\neg flies$ and its $0.2$ residue
supports $flies$.  The ordinary body is certain, so
$\beta_{B\mid A}=0.6$, and the report is
\[
  s^+=0.4\cdot0.9+0.6\cdot0.2=0.48,
  \qquad s^-=0.6\cdot0.8=0.48,
\]
with conflict zero and ignorance $0.4\cdot0.1=0.04$.
This reading is licensed when $0.8$ is elicited as the observed frequency of
the negative outcome among members of the exception class and the class is
assumed to have exactly the two displayed outcomes.  A source reliability
score or an arbitrary rule confidence does not license assigning its
complement to the original head.
This is the paired row of Table~\ref{tab:defaultexamples}: the
bounded-prover clause-activation sample returned signed value $-.115$, consistent
with treating the two branches as ordinary derivability events, while
shared-threshold sampling measured
$(.479,.481,0,.040)$ and \GK{} gives $(.48,.48,0,.04)$.

The complement is licensed only by this explicit paired reference-class
construction.  An input confidence $q$ on a rule supporting one conclusion
does not imply a separate input confidence $1-q$ on a rule supporting the
opposite conclusion.  With extra
exception conditions $E_i$ and input confidences $q_i$, define the following
conditional quantities when the conditions have no shared predecessors:
let $u_i$ be the conditional probability that $E_i$ is positively usable
within the blocked class.  The positive, negative, and ignorant conditional
components of the blocked branch are
\[
 b^{+}=\prod_i(1-q_i u_i)-\prod_i(1-u_i),\qquad
 b^{-}=1-\prod_i(1-q_i u_i),\qquad
 b^{g}=\prod_i(1-u_i).
\]
They sum to one.  Their unconditional blocked-branch probabilities are
\[
 s^+_{\rm blocked}=\beta_{B\mid A}b^{+},\qquad
 s^-_{\rm blocked}=\beta_{B\mid A}b^{-},\qquad
 g_{\rm blocked}=\beta_{B\mid A}b^{g}.
\]
The scalar $r_{\mathrm{blocked}}$ in Equation~\eqref{eq:blockmix} refers only
to $b^{+}$, the
blocked branch's contribution to the original head; a complete paired branch
can also contribute to the opposite head and to ignorance.

For one extra condition this reduces to
\[
 s^+_{\rm blocked}=\beta_{B\mid A}u_E(1-q),\qquad
 s^-_{\rm blocked}=\beta_{B\mid A}q u_E,\qquad
 g_{\rm blocked}=\beta_{B\mid A}(1-u_E),
\]
where $u_E$ is the corresponding conditional probability that the extra
condition is positively usable.
Failure of the extra condition therefore contributes only to ignorance; it
does not become support for the main head.

Pairing is local and syntactic.  The exception head must complement the
contrary-gated rule's head under the same ground substitution, and the same
exception literal must occur in the exception rule's body.  A rule that does not
meet these conditions contributes ordinary support for the opposite
conclusion and therefore constitutes a rebutting attack.  Even when paired with
one rule instance, the exception remains a rebutting attack on other, unpaired
derivations of the same head.  Without these restrictions, replacing an input confidence
$q$ by its complement $1-q$ would introduce a confidence value that the
input does not state.

\label{sec:defaultextensions}

\subsection{Priority restrictions}
\label{sec:rankclasses}

Priority also restricts which arguments a blocker check may use: inside a
check at priority $\pi$, support derived through a default of strictly lower
priority contributes nothing.  Only the higher-priority support remains.  The
two cases use
\[
\begin{aligned}
p(a)&\ [\mathrm{unless}\ \neg p(a);\pi_p],\\
.9::\neg p(a)&\ [\mathrm{unless}\ p(a);\pi_n].
\end{aligned}
\]
With $(\pi_p,\pi_n)=(3,1)$, the candidate's only exception proof passes through a
lower-priority default and is excluded, giving $(1,0,0,0)$.  With
$(\pi_p,\pi_n)=(1,3)$, the exception is admissible and the higher-ranked negative
side gives $(0.1,0.9,0,0)$.  The restriction depends only on the input
(priorities and derivation structure), so it is part of the semantics rather
than a search heuristic, and the threshold-world sampler of
Section~\ref{sec:samplers} implements it.
For these two examples, Model~1 gives signed values $+1$ and $-.8$, while
shared-threshold sampling gives $(1,0,0,0)$ and
$(.1,.9,0,0)$, equal to \GK{}.

A priority restriction inherited from an enclosing exception check remains in
force throughout recursive evaluation.  The local joint calculation omits
every contribution excluded by that context.  The current
shared-threshold sampler reports the corresponding contextual priority-zero case
outside its supported fragment.

Priorities can be integers or classes ordered by a taxonomy; a more specific
class defeats a more general one~\cite{tammet2022}.

\subsection{Default cycles}
\label{sec:cycles}

Defaults can block each other through their conclusions.  The defined forms
are summarized before the details.

\begin{center}
\scriptsize
\begin{tabular}{L{4.5cm}L{8.8cm}}
\toprule
Dependency form & Treatment \\
\midrule
Acyclic exception dependencies & core shared-threshold evaluation \\
Two defaults over different conclusion atoms, each with the other default's
conclusion as its exception condition
& query-relative credulous cycle policy \\
Opposite-polarity defaults with explicit positive ranks
& local mutual blocking or strict-priority override \\
Reciprocal opposite-polarity defaults at priority zero
& remove the two internal exception edges and combine active applications as
ordinary unranked support \\
Self-blocking default $p$ unless $p$
& contributes nothing; pure ignorance without other support \\
Recursion through a contested atom
& no completed dependency-aware value; flagged retained-proof fallback \\
\bottomrule
\end{tabular}
\end{center}

This table is exhaustive for the cyclic forms assigned a result here; it is
not a general semantics for cyclic default programs.  The least-fixpoint
operator used for monotone same-polarity cycles is defined in
Section~\ref{sec:samplers}.  Every cycle with an opposition or exception edge
must match one of the remaining rows or is reported outside the fragment.

First, an even loop --- two defaults each naming the other's conclusion as its
exception --- resolves credulously and per query: an exception argument whose
own validity depends on defeating the queried candidate is not used to defeat
that candidate.  With
$p(a)\leftarrow t(a)$ unless $q(a)$, and $q(a)\leftarrow t(a)$ unless
$p(a)$, the query $p(a)$ answers with signed confidence one and the query
$q(a)$ also answers with signed confidence one on the same knowledge base.
The two answers are
jointly inconsistent by design, exactly as credulous extensions of a default
theory are.  In the same loop with rule confidences $0.9$ and $0.8$, the
queried side answers with its own support: $(0.9,0,0,0.1)$.

The local opposing-default case of Section~\ref{sec:localrules} is a pair
of opposite-polarity defaults with the same explicit positive priority; the
query-relative loop above is not of that form.  In the local case the two
threshold conditions mutually block, and their simultaneous region
contributes to ignorance.  When both priorities are omitted they are zero,
and the reciprocal priority-zero rule removes the two internal exception
references before combining the active applications as ordinary unranked
support; the simultaneous region then contributes to conflict.

A pair
$p\ [\mathrm{unless}\ q]$ and
$q\ [\mathrm{unless}\ p]$
has two distinct conclusion atoms and remains governed by the stated
credulous cycle policy.

Second, a self-blocking default ($p(a)$ unless $p(a)$) contributes nothing;
the report is pure ignorance.  Self-blocking is determined from the
application's authored head and is therefore the same in both query
orientations.

Third, recursion through a \emph{contested} atom has no defined
dependency-aware value.  The dependency-aware evaluation returns no completed
result, so the report uses
ordinary opposition between the proof pools and sets the incompleteness flags
\begin{center}
\texttt{PROOF\_FALLBACK}\\
\texttt{SCRUTINY\_INCOMPLETE}.
\end{center}
No dependency-aware value is
reported because that calculation did not complete; the displayed number is
the flagged retained-proof fallback.

The two reference constructions give the following results on these cases.
On the certain even loop, Model~1 gives signed value $+1$, while ST gives
$(1,0,0,0)$ for either closed query, matching the two query-relative \GK{}
answers.  On the same loop with rule confidences $.9$ and $.8$, Model~1
has signed value $+.9$ for the first query and ST gives $(.9,0,0,.1)$;
reversing the query gives the corresponding $.8$ result.  On the self-blocking
default Model~1 gives zero and ST gives $(0,0,0,1)$, matching \GK's pure
ignorance.  On the contested recursive example Model~1 has signed value
$.7-.4=.3$, while ST reports the nonmonotonic cycle outside its supported
fragment; \GK{} returns the same signed fallback together with
conflict and incompleteness flags.

The credulous interpretation of even loops is a design choice; a skeptical
alternative in which both queries report mutual blocking would also be
coherent.  The policy is defined per query and is reproduced by the
threshold-world sampler on its stated fragment; the reasoner's operational
result remains subject to the coverage and time budget of its blocker
searches.

\subsection{Conflict sensitivity}

The report can also list contested predecessor atoms and compute an optional
conflict-sensitivity interval.
For $k$ contested predecessor atoms, the system recomputes the answer for the $2^k$
global assignments that send each listed atom's entire conflict region to
either its positively or its negatively usable state.  One assignment is used
consistently on every downstream occurrence of that atom, so shared
dependencies are preserved during recomputation.  The minimum and maximum
signed confidences over these assignments form a sensitivity interval.  It is
not a probability interval and assumes no distribution over the assignments.

\section{Probabilistic and causal scope}
\label{sec:causalboundary}

\subsection{Conditioning and explaining away}

For the two-cause alarm network,
$P(burglary)=.1$, $P(earthquake)=.2$, and the two alarm rules have confidences
$.9$ and $.6$.  Exact enumeration, \ProbLog{}, and \GK{} all give the aligned
forward value
\[
  P(alarm)=1-(1-.9\cdot.1)(1-.6\cdot.2)=.1992.
\]
Clause-activation and shared-threshold sampling both estimate this same
one-sided value.  Exact conditioning, however, gives a burglary posterior of
$.457831$; observing an earthquake as well reduces it to $.150943$.
Asserting $alarm$ as a new \GK{} fact does not condition the existing clause
choices and cannot express the explaining-away query.

\subsection{Causal intervention}

The firing-squad structural model of \cite{bochman2015} illustrates the
intervention limitation.  Let the court order and structural equations be
\[
C\sim\mathrm{Bernoulli}(.4),\qquad A=C,\qquad B=C,\qquad dies=A\lor B.
\]
The forward model gives $P(dies)=.4$.  After externally
replacing the equation for shooter $A$ by $A=1$ and fixing the context $C=0$,
the transformed model has $B=0$ and $dies=1$.  The corresponding \GK{}
knowledge base entails $dies$ with
signed confidence one and explicitly refutes that shooter $B$ fires.  The run
checks only the externally transformed forward model; \GK{} has no native
\texttt{do} operator.  Back-door,
front-door, and same-context counterfactual values were rechecked by exact
enumeration and Monte Carlo, but are not reported as \GK{} results because the
current query language cannot express them.
On the forward encoding, CA and ST both estimate the same $.4$ one-sided
value.  On the externally transformed deterministic input both return one.
These sampler results validate only the two supplied forward programs;
neither sampler supports intervention or counterfactual queries.

\FloatBarrier

\section{Experiment inputs}
\label{sec:completeinputs}

This appendix gives the source-level logical content of the examples in
Tables~\ref{tab:instances} and~\ref{tab:studyclass}.  Some stored example
files also contain prover output and encodings for other systems;
those logs are not part of the input and are not reproduced here.  No example
was excluded for size: the largest printed example, Example~17, has 23
facts/rules plus its query.  In the compact notation below, omitted
input confidences are one, variables begin with an upper-case letter, commas in a
rule body mean conjunction, a prefix \texttt{-} denotes explicit classical negation,
and \texttt{Q:} introduces the query.  Duplicate uncertain statements, such
as the two $.5::bird(a)$ statements in Example~1, are retained as
distinct input-clause occurrences with distinct identifiers and therefore as
distinct activation events.
Confidence-bearing equivalences below are syntactic shorthand.  Their exact uncertainty
events are the clausified clauses with split confidences described in
Section~3.1, so classically equivalent source formulas need not define the
same activation model.  The inputs printed here are the reproduction record
for the examples.

\subsection{Default calculation inputs}
\label{sec:defaultcaseinputs}

Here \(L\) and \(M\) are proposition names used only in their own examples;
priority follows the semicolon in E4.
\[
\begin{array}{ll}
\mathrm{E1}:&
 .7::B,\quad .3::\neg B,\quad .9::L\ [\mathrm{unless}\ B].\\[1mm]
\mathrm{E2}:&
 .6::P,\quad .9::L\ [\mathrm{unless}\ P],\quad
 .8::\neg L\leftarrow P.\\[1mm]
\mathrm{E3}:&
 .6::P,\quad .5::E,\quad .9::L\ [\mathrm{unless}\ P],\quad
 .8::\neg L\leftarrow P\land E.\\[1mm]
\mathrm{E4}:&
 .9::L\ [\mathrm{unless}\ \neg L;2],\quad
 .3::\neg L\ [\mathrm{unless}\ L;3].\\[1mm]
\mathrm{E5}:&
 .6::S,\quad .8::D\leftarrow S,\quad
 .9::M\leftarrow D\ [\mathrm{unless}\ C],\quad C\leftarrow S.
\end{array}
\]

\subsection{Event-identity test cases}

\begin{small}
\begin{verbatim}
IE1:
  p(a).  p(b).
  .8::r(X) <- p(X).
  Q: r(a) and r(b).

IE2:
  p(a).
  .8::q(X) <- p(X).
  s(X) <- q(X).  t(X) <- q(X).
  g(X) <- s(X), t(X).
  Q: g(a).

IE3:
  .9::p(a).  .85::p(b).
  .8::r(X) <- p(X).
  .7::r(c) <- r(a), r(b).
  Q: r(c).
\end{verbatim}
\end{small}

\subsection{The example inputs}

Examples 13--14, 15--16, and 20--21 share a knowledge base within each pair
and differ only in the query shown.  Printing each shared knowledge base once
is therefore a complete specification of both examples.

\begin{small}
\begin{verbatim}
Example 1:
  .5::bird(a).  .5::bird(a).
  Q: bird(a).

Example 2:
  .5::bird(a).  .6::bird(a).
  Q: bird(a).

Example 3:
  .5::bird(a).  .6::bird(b).
  twobirds(dummy) <- bird(a), bird(b).
  Q: twobirds(dummy).

Example 4:
  .5::bird(a).  .6::bird(b).
  twobirds(X,Y) <- bird(X), bird(Y).
  Q: twobirds(X,Y).

Example 5:
  .5::bird(a).  .5::-bird(a).
  Q: bird(a).
\end{verbatim}
\end{small}
\newpage
\begin{small}
\begin{verbatim}

Example 6:
  .5::bird(a).  .6::bird(a).  .5::-bird(a).
  Q: bird(X).

Example 7:
  .5::bird(a).  .6::bird(a).  .5::-bird(a).
  Q: -bird(X).

Example 8:
  .8::stress(ann).
  .6::influences(ann,bob).  .2::influences(bob,carl).
  smokes(X) <- stress(X).
  smokes(X) <- smokes(Y), influences(Y,X).
  Q: smokes(carl).

Example 9:
  .8::stress(ann).  .4::stress(bob).
  .6::influences(ann,bob).  .2::influences(bob,carl).
  smokes(X) <- stress(X).
  smokes(X) <- smokes(Y), influences(Y,X).
  Q: smokes(carl).

Example 10:
  .8::stress(ann).  .4::stress(bob).
  .6::influences(ann,bob).  .2::influences(bob,carl).
  smokes(X) <- stress(X).
  smokes(X) <- smokes(Y), influences(Y,X).
  Q: smokes(X).

Example 11:
  .3::smokes(X).  .1::friends(X,Y).
  .9::friends(Y,X) <- friends(X,Y).
  .6::susceptible(X).
  .2::smokes(X) <- smokes(Y), friends(X,Y), susceptible(X).
  friends(chris,sam).  smokes(chris).
  Q: smokes(sam).

Example 12:
  .3::smokes(X).  .1::friends(X,Y).
  .9::friends(Y,X) <- friends(X,Y).
  .6::nonconformist(X).
  .2::smokes(X) <- -smokes(Y), friends(X,Y), nonconformist(X).
  friends(chris,sam).  smokes(chris).
  Q: smokes(sam).
\end{verbatim}
\end{small}
\newpage
\begin{small}
\begin{verbatim}

Examples 13 and 14:
  person(john).  person(mary).
  .7::burglary(t1).  .2::earthquake(t1).
  .9::(burglary(t1) and earthquake(t1) <-> alarm(t1)).
  .8::(burglary(t1) and -earthquake(t1) <-> alarm(t1)).
  .1::(-burglary(t1) and earthquake(t1) <-> alarm(t1)).
  .8::(alarm(t1) and person(X) <-> calls(X)).
  .1::(-alarm(t1) and person(X) <-> calls(X)).
  calls(john).  calls(mary).
  Q13: burglary(t1).  Q14: earthquake(t1).

Examples 15 and 16:
  .3::stress(X) <- person(X).
  .2::influences(X,Y) <- person(X), person(Y).
  smokes(X) <- stress(X).
  smokes(Y) <- friend(X,Y), influences(Y,X).
  .4::asthma(X) <- smokes(X).
  person(1). person(2). person(3). person(4).
  friend(1,2). friend(2,1). friend(2,4).
  friend(3,2). friend(4,2).
  smokes(2).  -influences(4,2).
  Q15: smokes(X).  Q16: asthma(X).

Example 17:
  movie(X) <- moviedirectedbydirector(X,Y).
  movie(X) <- moviestaractor(X,Y).
  movie(X) <- movie_base(X).
  moviedirectedbydirector(Y,X) <- directordirectedmovie(X,Y).
  moviestaractor(Y,X) <- actorstarredinmovie(X,Y).
  directordirectedmovie(X,Y) <- directordirectedmovie_base(X,Y).
  actorstarredinmovie(X,Y) <- actorstarredinmovie_base(X,Y).
  book(X) <- bookwriter(X,Y).
  book(X) <- book_base(X).
  bookwriter(X,Y) <- bookwriter_base(X,Y).
  .9375::directordirectedmovie_base(ronhoward,abeautifulmind).
  .995924::movie_base(abeautifulmind).
  .999999::book_base(abeautifulmind).
  .927773::movie_base(casinoroyale).
  .9375::directordirectedmovie_base(martincampbell,casinoroyale).
  .999512::actorstarredinmovie_base(danielcraig,casinoroyale).
  .999996::bookwriter_base(casinoroyale,ianfleming).
  .999998::book_base(casinoroyale).
  .9375::bookwriter_base(sleepyhollow,washingtonirving).
  .96875::actorstarredinmovie_base(christopherwalken,sleepyhollow).
  .976353::movie_base(theadventuresofrobinhood).
  bookwriter_base(theadventuresofrobinhood,howardpyle).
  moviebook(X) <- movie(X), book(X).
  Q: moviebook(X).
\end{verbatim}
\end{small}
\newpage
\begin{small}
\begin{verbatim}

Example 18:
  .5::bird(a).  .5::-bird(a).
  .9::flies(X) <- bird(X).
  Q: flies(X).

Example 19:
  .5::bird(a).  .5::-bird(a).
  .9::flies(X) <- bird(X).
  Q: bird(X).

Examples 20 and 21:
  .5::bird(a).  .5::-bird(a).
  .9::flies(X) <- bird(X).
  .1::-flies(X) <- bird(X).
  .1::flies(X) <- -bird(X).
  .9::-flies(X) <- -bird(X).
  Q20: flies(X).  Q21: bird(X).

Example 22:
  bird(tweety).  penguin(pennie).
  bird(X) <- penguin(X).
  .001::penguin(X) <- bird(X).
  .9::flies(X) <- bird(X).
  -flies(X) <- penguin(X).
  Q: flies(X).

Example 23:
  .5::bird(messy).  .5::-bird(messy).
  bird(tweety).  penguin(pennie).
  bird(X) <- penguin(X).
  .001::penguin(X) <- bird(X).
  .999::-penguin(X) <- bird(X).
  .9::flies(X) <- bird(X).
  .1::-flies(X) <- bird(X).
  .1::flies(X) <- -bird(X).
  .9::-flies(X) <- -bird(X).
  -flies(X) <- penguin(X).
  Q: flies(X).
\end{verbatim}
\end{small}

\subsection{Priority-zero checks}
\label{sec:rank0artifacts}

The public file \path{montecarlo/test_threshold_rank0.py} contains the positive,
negative, certain, one-sided, mixed, and contextual-priority cases with their
expected tuples.  The input in Section~\ref{sec:defaultcalcs} gives the
positive and negative D4-0 queries; the sampler changes only the query
polarity.  \path{Examples/exceptions/nixon.gkp} is the public deterministic
Nixon example.

\section{System comparison inputs and provenance}
\label{sec:crossinputs}

The abstract inputs below define the IDs in
Table~\ref{tab:systemmaster}; concrete
syntax differs by system and is preserved separately so that a surface
translation is not mistaken for a shared semantics.  Omitted input confidences are
one, and $Q$ is the query.
\[
\begin{array}{ll}
\mathrm{C1}:&
.5::a,\ .6::b,\quad Q\leftarrow a,\quad Q\leftarrow b;\\[1mm]
\mathrm{C2}:&
.5::a,\ .6::b,\quad Q\leftarrow a\land b;\\[1mm]
\mathrm{D1}:&
bird(a),\ .9::\neg flies(a),\quad
flies(a)\ [\mathrm{unless}\ \neg flies(a)];\\[1mm]
\mathrm{D2}:&
.7::B,\ .3::\neg B,\quad .9::H\ [\mathrm{unless}\ B];\\[1mm]
\mathrm{D3}:&
.5::B,\ .5::\neg B,\quad .9::H\leftarrow B;\\[1mm]
\mathrm{D4}:&
quaker(n),\ republican(n),\\
& pacifist(n)\ [\mathrm{unless}\ \neg pacifist(n);2],\quad
\neg pacifist(n)\ [\mathrm{unless}\ pacifist(n);2];\\[1mm]
\mathrm{D4\mbox{-}P}:&
.8::quaker(n),\ .6::republican(n),\\
& pacifist(n)\leftarrow quaker(n)\land
\mathop{\mathrm{not}}\neg pacifist(n),\\
& \neg pacifist(n)\leftarrow republican(n)\land
\mathop{\mathrm{not}}pacifist(n);\\[1mm]
\mathrm{D5}:&
.9::L\ [\mathrm{unless}\ \neg L;2],\quad
.3::\neg L\ [\mathrm{unless}\ L;3];\\[1mm]
\mathrm{D6a}:&
flies(a)\ [\mathrm{unless}\ injured(a)];\\
\mathrm{D6b}:&
flies(a)\ [\mathrm{unless}\ \neg flies(a)];\\[1mm]
\mathrm{D7}:&
.6::P,\quad .9::L\ [\mathrm{unless}\ P],\quad
.8::\neg L\leftarrow P;\\[1mm]
\mathrm{EQ1}:&
.9::p(a),\ .7::\neg p(b),\ .8::(a=b),\quad Q=p(a);\\[1mm]
\mathrm{X1}:&
.7::p(a),\ .6::\neg p(a),\ q(b),\quad Q=q(b);\\[1mm]
\mathrm{F1}:&
.9::bird(a),\quad .8::bird(f(X))\leftarrow bird(X),\quad
Q=bird(f(f(a)));\\[1mm]
\mathrm{DJ1}:&
.8::(p(a)\lor q(a)),\ .9::\neg p(a),\quad Q=q(a);\\[1mm]
\mathrm{N1}:&
.5::bird(a),\ .6::bird(b),\quad
twobirds(X,Y)\leftarrow bird(X)\land bird(Y),\\
&Q=twobirds(X,Y);\\[1mm]
\mathrm{N2}:&
.8::stress(ann),\ .4::stress(bob),\\
&.6::influences(ann,bob),\ .2::influences(bob,carl),\\
&smokes(X)\leftarrow stress(X),\\
&smokes(X)\leftarrow influences(Y,X)\land smokes(Y),\quad
Q=smokes(X).
\end{array}
\]
D4-P's displayed rules are the stable-model input.  Its D4-EQ \GK{} analogue
replaces each \(\mathop{\mathrm{not}}\) guard by a contrary-gated \GK{} rule with
a blocker literal encoding its exception condition at priority 2.  D4-0 is the otherwise identical priority-zero
regression with both priorities omitted.  It is not used as a translation of
the stable-model program.  Both query orientations are tested, and the
current implementation returns the component-swapped priority-zero partitions
stated in Section~\ref{sec:defaultcalcs}.  Their inputs and captured outputs are
listed in Appendix~\ref{sec:completeinputs}.

The small encoding differences used in the comparisons are shown explicitly
below.  D1-PL was run in the additional comparison suite.  D3-raw remains
an analytic baseline.  \texttt{not} denotes the target system's negation as
failure.
\begin{small}
\begin{verbatim}
D1-PL (A, run):
  bird(a).  0.9::neg_flies(a).
  flies(a) :- bird(a), not neg_flies(a).
  query(flies(a)).  query(neg_flies(a)).

D2-ST (GK, N):
  0.7::b.  0.3::-b.
  0.9::h :- unless(b).
  query(h).

D2-A (PASTA/plingo/smProbLog, A):
  0.7::b.  0.3::neg_b.  0.9::gate.
  h :- gate, not b.
  query(h).

D3-raw (E) and D3-guard (A):
  0.5::b.  0.5::neg_b.
  0.9::h_raw :- b.
  0.9::h_guard :- b, not neg_b.
  query(h_raw).  query(h_guard).

D4-P (N in stable-model systems):
  0.8::quaker.  0.6::republican.
  pacifist    :- quaker,    not nonpacifist.
  nonpacifist :- republican, not pacifist.
  query(pacifist).  query(nonpacifist).

D4-EQ-positive-query / D4-EQ-negative-query (A relative to D4-P):
  0.8::quaker(n).  0.6::republican(n).
  pacifist(X)  :- quaker(X),    unless(-pacifist(X),2).
  -pacifist(X) :- republican(X), unless(pacifist(X),2).
  query(pacifist(n)).            % positive-query file
  query(-pacifist(n)).           % negative-query file
\end{verbatim}
\end{small}

C3 is represented by the two complete shared-premise files
\path{Examples/confidences/overlap1.js} and
\path{Examples/confidences/overlap3.js}; their proof DAGs are too large to print usefully
here.

The EQ1, X1, F1, and DJ1 encodings are the per-row files of the
\path{comparisons/} package: \path{inputs/gk/e1.gkp}, \path{inputs/pasta/pasta_x1.lp},
\path{inputs/problog/problog_f1.pl}, \path{inputs/asp/asp_dj1.lp}, and the analogously named files
for the other systems; the EQ1 files keep the shorter \path{e1} stem.  The
substitution rules of EQ1, the strong-negation
and loop encodings of X1, and the even-loop disjunction of DJ1 are stated
in the package's case descriptions (\path{CASES.md}); the recorded failure
outputs and exit codes are in its results records.

Table~\ref{tab:crossmanifest} gives selected mappings from the comparison
runs to input specifications or captured files.  The complete
machine-readable manifest has one record for every table cell, including its
query or queries.  The
manifest of the comparison package records the corresponding commands;
\path{probabilistic_defaults_2026/comparison_manifest.json} gives the
initial cross-system mapping, and \path{comparisons/README.md} documents
the additional runs.  Entries beginning
\path{pasta/}, \path{plingo/}, \path{smproblog/}, or \path{gk_inputs/}
are relative to \path{probabilistic_defaults_2026/}.
The Class column uses the N/E/A/P/U labels of
Section~\ref{sec:comparisonprotocol}; U occurs only in the complete
machine-readable manifest, and P does not occur in these mappings.  These classes describe the relation between inputs, not
whether a command completed.

{\scriptsize
\setlength{\tabcolsep}{3pt}
\begin{longtable}{L{1.2cm}L{3.0cm}L{5.25cm}L{2.55cm}L{1.4cm}}
\caption{Selected input and requested-output mappings for the semantic comparison runs.}
\label{tab:crossmanifest}\\
\toprule
ID & System & Input specification or file & Query or requested output & Class \\
\midrule
\endfirsthead
\multicolumn{5}{c}{\tablename\ \thetable\ continued}\\
\toprule
ID & System & Input specification or file & Query or requested output & Class \\
\midrule
\endhead
\bottomrule
\endfoot
C1 & \ProbLog{} 2.2.10 & abstract C1 input above & \texttt{q} & E \\
   & \GK{} & \path{Examples/confidences/cumulate.js} & \texttt{bird(a)} & N \\
   & PASTA 1.0.1 & \path{Examples/system_comparison/pasta_independent_support.lp} & \texttt{q} & E \\
   & plingo 1.1.0 & abstract C1 input above & \texttt{q} & E \\
   & smProbLog 2.1.0.42 & abstract C1 input above & \texttt{q} & E \\
C2 & \ProbLog{} 2.2.10 & abstract C2 input above & \texttt{q} & E \\
   & \GK{} & \path{Examples/confidences/rulemult.js} & \texttt{twobirds(dummy)} & N \\
   & PASTA 1.0.1 & \path{Examples/system_comparison/pasta_conjunctive_support.lp} & \texttt{q} & E \\
   & plingo 1.1.0 & abstract C2 input above & \texttt{q} & E \\
   & smProbLog 2.1.0.42 & abstract C2 input above & \texttt{q} & E \\
D2-ST & \GK{} & abstract D2 input above & \texttt{h} & N \\
D2-A & PASTA 1.0.1 & activation analogue stated in Section~9.2 & \texttt{h} & A \\
   & plingo 1.1.0 & same activation analogue & \texttt{h} & A \\
   & smProbLog 2.1.0.42 & same activation analogue & \texttt{h} & A \\
D3-guard & \GK{} & Example 18 in Appendix~\ref{sec:completeinputs} & \texttt{flies(a)} & N \\
D3-guard & \ProbLog{} 2.2.10 & guarded D3 analogue stated in Section~9.2 & \texttt{flies(a)} & A \\
D4 & \GK{} & \path{Examples/exceptions/nixon.js} & \texttt{dislikeswar(n)} & N \\
D4 & PASTA 1.0.1 & abstract D4 input above & \(pacifist(n)\) & N \\
   & TweetyProject 1.31 & abstract D4 input above & positive argument & N \\
   & I-DLV 1.1.6 & abstract D4 input above & ground program for clasp 3.3.5 & N \\
D4-P & PASTA 1.0.1 & D4-P input in Section~9.2 & \(pacifist(n)\) & N \\
   & plingo 1.1.0 & same D4-P input & \(pacifist(n)\) & N \\
   & smProbLog 2.1.0.42 & same D4-P input & \(pacifist(n)\) & N \\
D4-EQ & \GK{} & D4-EQ input in Section~9.2 & positive & A \\
   & \GK{} & same D4-EQ input & negative & A \\
D1 & \ProbLog{} 2.2.10
   & \path{comparisons/inputs/problog/problog_cases.pl} & \(flies(a)\) & A \\
D4 & clingo 5.6.2
   & \path{comparisons/inputs/asp/asp_cases.lp} & \(pacifist(n)\) & N \\
D4 & DLV 2.1.1
   & \path{comparisons/inputs/asp/asp_cases.lp} & model enumeration & N \\
D4 & s(CASP) 1.1.4
   & \path{comparisons/inputs/scasp/scasp_cases.pl} & \(pacifist(n)\) & N \\
D5 & TweetyProject 1.31
   & \path{comparisons/inputs/tweety/d5_specificity.delp} & positive argument & A \\
EQ1 & \GK{} & \path{comparisons/inputs/gk/e1.gkp} & \(p(a)\) & N \\
   & PASTA 1.0.1 & \path{comparisons/inputs/pasta/pasta_e1.lp} & \(p(a)\) & A \\
X1 & \GK{} & \path{comparisons/inputs/gk/x1.gkp} & \(q(b)\) & N \\
   & smProbLog 2.1.0.42 & \path{comparisons/inputs/smproblog/smproblog_x1.pl} & \(q(b)\) & A \\
F1 & \GK{} & \path{comparisons/inputs/gk/f1.gkp} & \(bird(f(f(a)))\) & N \\
   & \ProbLog{} 2.2.10 & \path{comparisons/inputs/problog/problog_f1.pl} & \(bird(f(f(a)))\) & E \\
DJ1 & \GK{} & \path{comparisons/inputs/gk/dj1.gkp} & \(q(a)\) & N \\
   & PASTA 1.0.1 & \path{comparisons/inputs/pasta/pasta_dj1.lp} & \(q(a)\) & E \\
N1 & \GK{} & \path{comparisons/inputs/gk/n1_study4.gkp} &
   \texttt{twobirds(X,Y)} & N \\
   & \ProbLog{} 2.2.10 & \path{comparisons/inputs/problog/problog_nonground.pl} &
   open query & E \\
   & PASTA 1.0.1 & \path{comparisons/inputs/pasta/pasta_nonground.lp} &
   \texttt{twobirds(a,a)} & E \\
   & plingo 1.1.0 & \path{comparisons/inputs/plingo/plingo_nonground.lp} &
   \texttt{twobirds(a,a)} & E \\
   & smProbLog 2.1.0.42 & \path{comparisons/inputs/smproblog/smproblog_nonground.pl} &
   open query & E \\
   & s(CASP) 1.1.4 & \path{comparisons/inputs/scasp/scasp_nonground.pl} &
   \texttt{n1\_twobirds(X,Y)} & A \\
N2 & \GK{} & \path{comparisons/inputs/gk/n2_study10.gkp} &
   \texttt{smokes(X)} & N \\
   & \ProbLog{} 2.2.10 & \path{comparisons/inputs/problog/problog_nonground.pl} &
   open query & E \\
   & PASTA 1.0.1 & \path{comparisons/inputs/pasta/pasta_nonground.lp} &
   \texttt{smokes(ann)} & E \\
   & plingo 1.1.0 & \path{comparisons/inputs/plingo/plingo_nonground.lp} &
   \texttt{smokes(ann)} & E \\
   & smProbLog 2.1.0.42 & \path{comparisons/inputs/smproblog/smproblog_nonground.pl} &
   open query & E \\
   & s(CASP) 1.1.4 & \path{comparisons/inputs/scasp/scasp_nonground.pl} &
   \texttt{n2\_smokes(X)} & A \\
\end{longtable}
}

The public \GK{} Nixon file uses \texttt{dislikeswar} for the same positive
Nixon conclusion represented as \texttt{pacifist} in the abstract comparison
input.  For D4, I-DLV 1.1.6 produced the ground program and clasp 3.3.5
enumerated two stable models.

The public \path{comparisons/} package contains the concrete system-specific
inputs, commands, and captured outputs for the cells backed by runs in
Table~\ref{tab:systemmaster}.  Its machine-readable manifest records every
cell, including the reason for each U classification.  The abstract inputs
above give system-independent descriptions of the same cases.  The clingo, DLV, I-DLV,
s(CASP), and \GK{} birds inputs are identified in Section~9.3.  The
N/E/A/P/U labels are those of Section~\ref{sec:comparisonprotocol}.

\end{document}